\documentclass[twoside,11pt]{article}

\usepackage{apacite}
\usepackage{jmlr2e} 
\usepackage{graphicx}
\usepackage{arydshln}
\usepackage[table,xcdraw]{xcolor}

\usepackage{multirow, float,bm,amsmath}
\usepackage{amssymb,amsmath,multirow, amsfonts}
\usepackage{latexsym,mathtext, array, enumerate, verbatim, hyperref, listings}
\usepackage{color}
\usepackage{dsfont,booktabs}
\usepackage{natbib}
\usepackage[margin=1in]{geometry}
\usepackage{fancyhdr}
\usepackage{pdflscape} 
\usepackage[font=small,skip=0pt]{caption}
\usepackage{subcaption} 
\usepackage{mathtools} 
\usepackage{enumitem}  
\usepackage{float}
\usepackage{lastpage}
\usepackage{tabu}

\newcommand{\bb}{{\boldsymbol b}}
\newcommand{\bc}{{\boldsymbol c}}

\newcommand{\be}{{\boldsymbol e}}

\newcommand{\bp}{{\boldsymbol p}}
\newcommand{\bq}{{\boldsymbol q}}

\newcommand{\bs}{{\boldsymbol s}}
\newcommand{\bt}{{\boldsymbol t}}
\newcommand{\bu}{{\boldsymbol u}}
\newcommand{\bv}{{\boldsymbol v}}

\newcommand{\bx}{{\boldsymbol x}}

\newcommand{\bz}{{\boldsymbol z}}

\newcommand{\bbw}{{\boldsymbol W}}

\DeclareMathSymbol{*}{\mathbin}{symbols}{"03}

\newtheorem{assumption}{Assumption}[section]

\jmlrheading{21}{2021}{1-\pageref{LastPage}}{}{}{}{Dimitris Bertsimas et al}
\ShortHeadings{A Robust Optimization Approach to Deep Learning}{Bertsimas et al}
\firstpageno{1}

\begin{document}
\title{\textcolor{black}{Robust Upper Bounds for Adversarial Training}}

\author{\name Dimitris Bertsimas \email dbertsim@mit.edu \\
       \addr Sloan School of Management and Operations Research Center\\
       Massachusetts Institute of Technology\\
       Cambridge, MA 02139, USA
       \AND
       \name Xavier Boix \email xboix@mit.edu \\
       \addr Department of Brain and Cognitive Sciences\\
       Massachusetts Institute of Technology\\
       Cambridge, MA 02139, USA
       \AND
       \name Kimberly Villalobos Carballo \email kimvc@mit.edu \\
       \addr Operations Research Center\\
       Massachusetts Institute of Technology\\
       Cambridge, MA 02139, USA
       \AND
       \name Dick den Hertog \email D.denHertog@uva.nl \\
       \addr Amsterdam Business School\\
       University of Amsterdam\\
       Plantage Muidergracht 12, 1018 TV  Amsterdam, The Netherlands}

\editor{}

\maketitle

\begin{abstract}
Many state-of-the-art adversarial training methods \textcolor{black}{for deep learning} leverage upper bounds of the adversarial loss to provide security guarantees {against adversarial attacks}. Yet, these methods {rely on \textcolor{black}{convex relaxations to propagate} lower and upper bounds for intermediate layers, which affect the tightness of the bound at the output layer}. We introduce a new approach to adversarial training by minimizing an upper bound of the adversarial loss {that is based on a holistic expansion of the network instead of separate bounds for each layer}. This bound is facilitated by state-of-the-art tools from \textcolor{black}{R}obust \textcolor{black}{O}ptimization; it has closed-form \textcolor{black}{ and can be effectively trained using backpropagation}. We derive two new methods with \textcolor{black}{the proposed} approach.  The first method (\textit{Approximated Robust Upper Bound} or aRUB) uses the first order approximation of the network as well as basic tools from \textcolor{black}{L}inear \textcolor{black}{R}obust \textcolor{black}{O}ptimization to obtain an \textcolor{black}{empirical} upper bound of the adversarial loss that can be easily implemented. The second method (\textit{Robust Upper Bound} or RUB),  computes a \textcolor{black}{provable} upper bound of the adversarial loss. Across a variety of tabular and vision data sets we demonstrate the effectiveness of \textcolor{black}{our} approach ---RUB is substantially more robust than state-of-the-art methods for larger perturbations, while aRUB matches the performance of state-of-the-art methods for small perturbations. 
All the code to reproduce the results can be found at \url{https://github.com/kimvc7/Robustness}. 

\end{abstract}


\section{Introduction}
Robustness of neural networks for classification problems has received increasing attention in the past few years, since it was exposed that these models could be easily fooled by introducing some small perturbation in the input data. These perturbed inputs, which are commonly referred to as \textit{adversarial examples}, are visually indistinguishable from the natural input, and neural networks simply trained to maximize accuracy often assign them to an incorrect class~\citep{szegedy2013intriguing}.  This problem becomes particularly relevant when considering applications related to self-driving cars or medicine, in which adversarial examples represent an important security threat~\citep{kurakin2016adversarial}. 

The Machine Learning community has recently developed multiple heuristics to make neural networks robust. The most popular ones are perhaps those based on \textcolor{black}{training with adversarial examples}, a method first proposed by \cite{goodfellow2015explaining} and which consists in training the neural network using adversarial inputs instead of \textcolor{black}{or in addition to} the standard data. The defense by \cite{madry2019deep}, which finds the adversarial examples with bounded norm using iterative projected gradient descent (PGD) with random starting points, has proved to be one of the most effective methods \citep{ tjeng2019evaluating}, although it comes with a high computational cost. \textcolor{black}{Another more efficient defense was proposed by \cite{wong2020fast}, which uses instead fast gradient sign methods (FGSM) to find the attacks.} Other heuristic defenses rely on preprocessing or projecting the input space \citep{lamb, kabilan, samangouei}, on randomizing the neurons \citep{prakash, xie2017mitigating} or on adding a regularization term to the objective function \citep{ross2017improving, DBLP:journals/corr/HeinA17, yan2018deep}. There is a plethora of heuristics for adversarial robustness by now. Yet, these defenses are only effective to adversarial attacks  of small magnitude and are vulnerable to attacks of larger magnitude or to new attacks~\citep{athalye2018obfuscated}.

\textcolor{black}{Given the lack of an exact and tractable reformulation of the adversarial loss,} a recent strand of research has been to leverage upper bounds to improve adversarial robustness. These upper bounds provide security guarantees against adversarial attacks, even new ones, by finding a mathematical proof that a network is not susceptible to any attack, e.g.~\citep{dathathri2020enabling,raghunathan2018semidefinite,katz2017reluplex,tjeng2019evaluating,bunel2017unified,anderson2020strong,singh2018fast,zhang2018efficient,weng2018towards,gehr2018ai2,dvijotham2018dual,lecuyer2019certified,cohen2019certified}. \textcolor{black}{Replacing the standard loss with these upper bounds during training is a common technique for obtaining adversarial defenses. \cite{wong2018provable} for instance, find an upper bound for the adversarial loss by applying linear relaxations in the network and computing a convex polytope that contains all possible values for the last layer given adversarial examples with bounded norm.} An upper bound on the adversarial loss is also computed in \cite{raghunathan2018certified} by solving instead a semidefinite program. Other more scalable and effective methods based on minimizing an upper bound of the adversarial loss  have been introduced~\citep{gowal2019scalable,balunovic2019adversarial,mirman2018differentiable,dvijotham2018training,wong2018scaling,zhang2019towards}.

While these methods \textcolor{black}{can} provide security guarantees \textcolor{black}{against adversarial examples, most of them rely on \textcolor{black}{convex relaxations} to recursively compute upper and lower bounds for each layer, which introduces gaps that propagate and can affect the final bound for the last layer. For instance, the approach proposed in \cite{gowal2019scalable} computes bounds for each layer by assuming that the worst-case bounds for all previous layers can be achieved simultaneously. This often yields a loose upper bound of the adversarial loss whose minimization can be sensitive to hyperparameters \citep{zhang2019towards}. Another example is the aforementioned defense from \cite{wong2018provable}, where bounds at each layer are computed by solving a linear program that uses the bounds from previous layers for the linear ReLU relaxations. \textcolor{black}{Unlike these approaches, the method proposed in \cite{raghunathan2018certified} does not require computation of intermediate bounds}, however, their proposed upper bound }only works for neural networks with two layers. 

\textcolor{black}{A promising yet under-explored approach is the application of state-of-the-art Robust Optimization (RO) tools~\citep{bertsimasbook2022}. RO has proven to be effective in handling uncertainty in parameters that may result from rounding or implementation errors. Recently, it has also been applied to provide robustness against input perturbations in some machine learning models, such as Support Vector Machines and Optimal Classification Trees~\citep{bertsimas2019robust}, and it could be similarly leveraged for deep learning. While previous works on robustness of neural networks generally formulate the problem in the context of RO, they do not utilize the more advanced tools available in this field. Instead, they mostly depend on linear or convex relaxations and heuristic methods to simplify the original non-convex problem. In this paper, we use state-of-the-art RO tools to derive a new closed-form solution of an upper bound of the adversarial loss.  Our approach is based on a holistic expansion of the network; it does not rely on convex relaxations or separate computation of bounds for each layer of the network, and it can still be effectively trained with backpropagation.}



We develop two new methods for training deep learning models that are robust against input perturbations.  
 The first method (\textit{Approximated Robust Upper Bound} or aRUB),  minimizes an \textcolor{black}{empirical} upper bound of the \textcolor{black}{adversarial loss} for $L_p$ norm bounded uncertainty sets, for general $p$. It is simple to implement and performs similar to \textcolor{black}{state-of-the-art defenses} on small uncertainty sets. The second method (\textit{Robust Upper Bound} or RUB), minimizes a \textcolor{black}{provable} upper bound of the \textcolor{black}{adversarial loss} specifically for $L_1$ norm bounded uncertainty sets. This method shows the best performance for larger uncertainty sets, and more importantly, it provides security guarantees against $L_1$ norm bounded adversarial attacks. 
More concretely, we introduce the following robustness methods:

\begin{itemize}
    \item \textit{Approximated Robust Upper Bound} or aRUB: We develop a simple method to approximate an upper bound of the adversarial loss by adding a regularization term for each \textcolor{black}{target} class separately. As an alternative to standard adversarial training (which relies on linear approximations to find good adversarial attacks), we use the first order approximation of the network to estimate the worst case scenario for each individual class. We then apply standard results from \textcolor{black}{L}inear \textcolor{black}{R}obust \textcolor{black}{O}ptimization to obtain a new objective that behaves like an upper bound of the \textcolor{black}{adversarial loss} and which can be tractably minimized for robust training. This method can be easily implemented and performs very well when the uncertainty set radius $\rho$ is small. 

    \item \textit{Robust Upper Bound} or RUB: We extended state-of-the-art tools from \textcolor{black}{RO} to functions that like neural networks are neither convex nor concave. By splitting each layer of the network as the sum of a convex function and a concave function, we are able to obtain an upper bound of the \textcolor{black}{adversarial loss} for the case in which the uncertainty set is the $L_1$ sphere. Since the dual function of the $L_1$ norm is the $L_{\infty}$ norm, we convert the maximum over the uncertainty set into a maximum over a finite set. In the end, instead of minimizing the worst case loss over an infinite uncertainty set, the new objective minimizes the worst case loss over a discrete set whose cardinality is twice the dimension of the input data. While this represents a significant increase in memory for high dimensional inputs, we show that this approach remains tractable for multiple applications. 
    The main advantage of this method is that it provides security guarantees against adversarial examples bounded in the $L_1$ norm. Additionally, we also show experimentally that this method generally achieves the highest adversarial accuracies for larger uncertainty sets.
    
\end{itemize}

Also, we show that these methods consistently achieve higher standard accuracy (i.e., non adversarial accuracy), than the nominal neural networks trained without robustness. While this result is not true for a general choice of uncertainty set (see for example \cite{ilyas2019adversarial}), we observe that when the uncertainty set has the appropriate size it can significantly improve the classification performance of the network, which is consistent with the results obtained for other classification models like Support Vector Machines, Logistic Regression and Classification Trees \citep{bertsimas2019robust}.

The paper is organized as follows: Section \ref{sec:previous} revisits previous works on \textcolor{black}{RO}, Section \ref{sec:optimization-problem} defines the robust problem, Section \ref{sec:aRUB} presents \textcolor{black}{the} first method (Approximate Robust Upper Bound), and Section \ref{sec:RUB} contains \textcolor{black}{the} second method (Robust Upper Bound). Lastly, Section \ref{sec:experiments} contains the results for the computational experiments.

\section{Previous Works on Robust Optimization}
\label{sec:previous}

Over the last two decades, \textcolor{black}{RO} has become a successful approach to solve optimization problems under uncertainty. For an overview of the primary research in this field we refer the reader to \cite{bertsimas2011theory}. Areas like mathematical programming and engineering have long applied these tools to develop models that are robust against uncertainty in the parameters, which may arise from rounding or implementation errors. For many applications, the robust problem can be reformulated as a tractable optimization problem, which is referred to as the \textit{robust counterpart}. For instance, for several types of uncertainty sets, the robust counterpart of a linear programming problem can be written as a linear or conic programming problem \citep{ben2009robust}, which can be solved with many of the current optimization software. While there is not a systematic way to find robust counterparts for a general nonlinear uncertain problem, multiple techniques have been developed to obtain tractable formulations in some specific nonlinear cases.

As shown in \cite{ben2009robust}, the exact robust counterpart is known for Conic Quadratic problems and  Semidefinite problems in which the uncertainty is finite, an interval or an unstructured norm-bounded set. More generally, it is shown in \cite{ben2015deriving} that for problems in which the objective function or the constraints are concave in the uncertainty parameters, Fenchel duality can be used to exactly derive the corresponding robust counterpart. While the result does not necessarily have a closed-form, the authors show that it yields a tractable formulation for the most common uncertainty sets (e.g. polyhedral and ellipsoidal uncertainty sets).

The problem becomes significantly more complex when the functions in the objective or in the constraints are instead convex in the uncertainty \citep{chassein2019complexity}. Since obtaining \textcolor{black}{provable} robust counterparts in these cases is generally infeasible, safe approximations are considered instead \citep{bertsimasrobust}. For instance, \cite{zhen2017robust} develop safe approximations for the specific cases of second order cone and semidefinite programming constraints with polyhedral uncertainty. These techniques are generalized in \cite{roos2020tractable}, where the authors convert the robust counterpart to an adjustable \textcolor{black}{RO} problem that produces a safe approximation for any problem that is convex in the optimization variables as well as in the the uncertain parameters. 

Even though the approaches mentioned above consider uncertainty in the parameters of the model as opposed to uncertainty in the input data, the same techniques can be utilized for obtaining robust counterparts in the latter case. In fact, the robust optimization \textcolor{black}{RO} methodologies have recently been applied to develop machine learning models that are robust against perturbations in the input data. In \cite{bertsimas2019robust}, for example, the authors consider uncertainty in the data features as well as in the data labels to obtain tractable robust counterparts for some of the major classification methods: support vector
machines, logistic regression, and decision trees. However, due to the high complexity of neural networks as well as the large dimensions of the problems in which they are often utilized, robust counterparts or safe approximations for this type of models have not yet been developed. There are two major challenges with applying \textcolor{black}{RO} tools for training robust neural networks: 

\begin{enumerate}[label=(\roman*)]
    \item \textit{Neural networks are neither convex nor concave functions:} As mentioned earlier, robust counterparts are difficult to find for a general problem. Although plenty of work has been done to find tractable reformulations as well as safe approximations, all of them rely on the underlying function being a convex or a concave function of the uncertainty parameters. Unfortunately, neural networks don't satisfy either condition, which makes it really difficult to apply any of the approaches discussed above. \\
    \item \textit{The robust counterpart needs to preserve the scalability of the model:} Neural networks are most successful in problems involving vision data sets, which often imply large input dimensions and enormous amount of data. For the most part, they can still be successfully trained thanks to the fact that back propagation algorithms can be applied to solve the corresponding unconstrained optimization problem. However, the \textcolor{black}{RO} techniques for both convex and concave cases often require the addition of new constraints and variables for each data sample, increasing significantly the number of parameters of the network and making it very difficult to use standard machine learning software for training.
\end{enumerate} 

A straightforward way to overcome both of these difficulties would be to replace the loss function of the network with its first order approximation. However, this loss function is usually highly nonlinear and therefore the linear approximation is very inaccurate. Our method aRUB explores a slight modification of this approach that significantly improves adversarial accuracy by considering only the linear approximation of the network's output layer. 

Alternatively, a more rigorous approach to overcome problem (i) would be to piece-wise analyze the convexity of the network and apply the \textcolor{black}{RO} techniques in each piece separately, but this approach would introduce additional variables that are in conflict with requirement (ii). For \textcolor{black}{the proposed} RUB method we then develop a general framework to split the network by convexity type, and we show that in the specific case in which the uncertainty set is the $L_1$ norm bounded sphere, we can solve for the extra variables and obtain an unconstrained problem that can be tractably solved using standard gradient descent techniques.

\section{The Robust Optimization problem}\label{sec:optimization-problem}
We consider a classification problem over data points $\bx\in \mathbb{R}^M$ labeled with one of $K$ different classes \textcolor{black}{in [K], where we use the notation [n] to denote the set $\{1, \hdots, n\}$.} Given weight matrices $\bbw^\ell\in \mathbb{R}^{r_{\ell-1}\times r_{\ell}}$ and bias vectors $\bb^\ell\in \mathbb{R}^{r_\ell}$ for $\ell\in [L]$,  such that $r_0= M, r_L = K$, the corresponding feed forward neural network with $L$ layers and ReLU activation function is defined by the equations
\begin{align*}
    \bz^1(\theta, \bx) &= \bbw^1\bx+\bb^1,\\
    \bz^\ell(\theta, \bx) &= \bbw^\ell [\bz^{\ell-1}(\theta, \bx)]^+ + \bb^{\ell}, \quad \forall \: 2\leq \ell \leq L,
\end{align*}
where $\theta$ denotes the set of parameters $(\bbw^\ell, \bb^\ell)$ for all $\ell \in [L]$ and $[\bx]^+$ is the result of applying the ReLU function to each coordinate of $\bx$. For fixed parameters $\theta$, the network assigns a sample $\bx$ to the class $\hat{y}= \arg\max_k \: \bz^L_k(\theta, \bx)$. And given a data set $\{({\bx}_n, {{y}}_n)\}_{n=1}^N$, where $y_n\textcolor{black}{\in [K]}$ is the target class of $\bx_n$, the optimal parameters $\theta$ are usually found by minimizing the empirical loss
\begin{align}\label{eq:nominal}
\min_{\theta}\enspace \frac{1}{N}{\sum_{n=1}^{N} \mathcal{L}({y_n}, \bz^L(\theta, {\bx}_n))},
\end{align}
with respect to a specific loss function \textcolor{black}{$\mathcal{L}: [K]\times \mathbb{R}^K\rightarrow \mathbb{R}_{\geq 0}$.}

In the \textcolor{black}{RO} framework, however, we want to find the parameters $\theta$ by minimizing the worst case loss achieved over an uncertainty set of the input. More specifically, instead of \textcolor{black}{ optimizing the nominal loss in Eq. \eqref{eq:nominal}, we want to optimize the adversarial loss:}
\begin{align}
\min_{\theta}\enspace \frac{1}{N}{\sum_{n=1}^{N} \max_{\boldsymbol{\delta} \in \mathcal{U}} \: \mathcal{L}\left({y_n}, \bz^L(\theta, \:{\bx}_n + \boldsymbol{\delta})\right)},
\label{minmax-problem}
\end{align}
for some uncertainty set $\mathcal{U}\subset \mathbb{R}^M$. 

Unfortunately, a closed-form expression for the inner maximization problem above is unknown and solving the min-max problem is notoriously difficult. \textcolor{black}{RO} provides multiple tools for solving such problems when the loss function is either convex or concave in the input variables. For example, in the case of concave loss functions, a common approach would be to take the dual of the maximization problem so that the problem can be formulated as a single minimization problem \citep{bertsimasbook2022}. If the loss function is instead convex, Fenchel's duality as well as conjugate functions can be used to find upper bounds and lower bounds of the maximization problem. However, there is no general framework developed for loss functions that do not fall into those categories, like in the case of neural networks.  \\

In this paper, we will focus on the specific case in which the uncertainty set is the ball of radius $\rho$ in the $\text{L}_p$ space; i.e. $\mathcal{U} = \{{\boldsymbol{\delta}} : \|{\boldsymbol{\delta}}\|_p \leq \rho\}$, \textcolor{black}{ and we make the following assumptions about the loss function:} 
\textcolor{black}{
 \begin{assumption}\label{assumption:trans-inv}
 The loss function $\mathcal{L}$ is translationally invariant; i.e. for all $y\in[K], \boldsymbol{z}\in \mathbb{R}^K$, it satisfies
 \begin{align}\label{trans-invariant}
     \mathcal{L}(y, \bz)  =  \mathcal{L}(y ,\bz - c\boldsymbol{e}) \quad \forall \: c\in \mathbb{R},
 \end{align}
 where $\boldsymbol{e}\in\mathbb{R}^K$ denotes the vector with value $1$ in all the coordinates.
 \end{assumption}}
 \textcolor{black}{\begin{assumption} \label{assumption:monotone}
 The loss function $\mathcal{L}$ is monotonic; i.e. \textcolor{black}{for all $ \: y\in[K], \: \: \bz,\bz^\prime \in \mathbb{R}^K$} it satisfies 
 \begin{align}\label{monotonic}
    \textcolor{black}{\bigg(\forall  k\in[K],  \bz_k - \bz_y\leq \bz^\prime_k - \bz^\prime_y\bigg)\: \: \implies  \bigg(\mathcal{L}(y, \bz)\leq \mathcal{L}({y, \bz^\prime})\bigg).}
 \end{align}
 \end{assumption}}
\textcolor{black}{
Although some loss functions utilized in deep learning like the squared error or the absolute error do not satisfy these assumptions, the most popular losses for classification problems (softmax with cross-entropy loss, multiclass hinge loss, and hardmax with zero-one loss) do satisfy both of them. Intuitively, Assumption \ref{assumption:trans-inv} implies that the loss function $\mathcal{L}$ takes into account the differences between the coordinates of $\bz$ but not the exact value at each coordinate; while Assumption \ref{assumption:monotone} means that a larger difference between the coordinates of the incorrect classes and the correct class results in a larger loss. These assumptions allow us to obtain the following result:}
\textcolor{black}{
\begin{lemma}\label{lemma:gen_upper_bound}
If the loss function $\mathcal{L}$ satisfies Assumptions \ref{assumption:trans-inv} and \ref{assumption:monotone}, then for all $\bx\in \mathbb{R}^M, y\in[K]$ the adversarial loss can be upper bounded as
\begin{align}
    &\min_\theta \:\max_{\boldsymbol{\delta}\in \mathcal{U}} \: \mathcal{L}(y, \bz^L(\theta,\bx + \boldsymbol{\delta}))\\
    \leq &\min_\theta \: \mathcal{L}\left(y , \bigg(\max_{\boldsymbol{\delta}\in \mathcal{U}} \bz_1^L(\theta, \bx + \boldsymbol{\delta})- \bz_y^L(\theta, \bx + \boldsymbol{\delta}),\dots, \max_{\boldsymbol{\delta}\in \mathcal{U}} \bz_K^L(\theta, \bx + \boldsymbol{\delta}) - \bz_y^L(\theta, \bx + \boldsymbol{\delta})\bigg) \right).
\end{align}
\end{lemma}}
\textcolor{black}{\begin{proof}
See Appendix \ref{A1}.\hfill
\end{proof} The robustness methods we propose in the next sections generate upper bounds for the adversarial differences $\max_{\boldsymbol{\delta}\in \mathcal{U}}\: \bz_k^L(\theta, \bx + \boldsymbol{\delta})- \bz_y^L(\theta, \bx + \boldsymbol{\delta})$ for all $k\in[K]$, and then apply the previous lemma to upper bound the \textcolor{black}{adversarial loss}.}

\section{Approximate Robust Upper Bound for small \texorpdfstring{$\rho$}{Lg}}\label{sec:aRUB}
Perhaps the most intuitive approach to tackle problem \eqref{minmax-problem} is to consider the first order approximation  of the loss function 
\begin{align*}{\mathcal{L}}\left(y, \bz^L(\theta, {\bx} + \boldsymbol{\delta})\right) \approx \mathcal{L}\left(y, \bz^L(\theta, \:{\bx})\right) + \boldsymbol{\delta}^\top \nabla_{\bx}\mathcal{L}\left(y, \bz^L(\theta, \:{\bx})\right), \end{align*}
since the right hand side is a linear function of  $\boldsymbol{\delta}$ and the maximization problem can be more easily solved for linear functions of the uncertainty. For example, it is not hard to see that the first order approximation reaches its maximum value in $\mathcal{U} = \{\boldsymbol{\delta} : \|\boldsymbol{\delta}\|_\infty\leq \rho\}$ exactly at $\boldsymbol{\delta}^\star = \rho \: sign\left(\nabla_{\bx}\mathcal{L}\big(y, \bz^L(\theta \:{\bx})\big)\right)$. This approach is referred to as \textit{fast gradient sign method} and it was first explored in \cite{goodfellow2015explaining}, where the networks are trained with adversarial examples generated as $\boldsymbol{x} + \boldsymbol{\delta}^\star$. A similar approach was proposed in \cite{huang2016learning}, where the authors considered the cross entropy loss and use the linear approximation of the softmax layer instead of the approximation of the entire loss function. In these methods, linear approximations are used to find near optimal perturbations that can produce strong adversarial examples for training, but not to approximate the adversarial loss. 

An alternative to these methods would then be to train the network with the natural data and replace the loss function with its linear approximation, transforming the problem into
\begin{align}
    &\min_{\theta}\enspace \frac{1}{N}{\sum_{n=1}^{N}  \:\max_{\boldsymbol{\delta} \in \mathcal{U}} \: \mathcal{L}\left(y_n, \bz^L(\theta, {\bx}_n)\right) + \boldsymbol{\delta}^\top \nabla_{\bx}\mathcal{L}\left(y_n, \bz^L(\theta, {\bx}_n)\right)} \nonumber\\
    =& \min_{\theta}\enspace \frac{1}{N}{\sum_{n=1}^{N}  \: \mathcal{L}\left(y_n, \bz^L( \theta,{\bx}_n)\right) + \rho \| \nabla_{\bx}\mathcal{L}\left(y_n, \bz^L(\theta, {\bx}_n)\right)\|_q},
\label{simple-approx-problem}
\end{align}
where $\|\:\|_q$ is the dual norm of $\|\:\|_p$, satisfying $\frac{1}{p} + \frac{1}{q} =1$. However, since the loss function is highly nonlinear, this approach \textcolor{black}{(which we call Baseline-$L_{p}$)} generally performs worse than training with adversarial examples (see the Baseline method in Section \ref{sec:experiments}). \\

\textcolor{black}{A more promising approach can be derived by noting} that each component of the network $\bz^L(\theta, \boldsymbol{x})$ is in fact a continuous piecewise linear function \textcolor{black}{(see the network definitions in Section \ref{sec:optimization-problem})}, \textcolor{black}{which suggests} that the first order approximation of $\bz^L$ is more precise than that of $\mathcal{L}(y, \bz^L)$ \textcolor{black}{for small neighborhoods}. In fact, we expect the outputs $\bz^L(\theta, \boldsymbol{x})$ and $\bz^L(\theta, \boldsymbol{x} + \boldsymbol{\delta})$ to be in the same linear piece when $\boldsymbol{x} + \boldsymbol{\delta}$ is close to $\boldsymbol{x}$. In other words, the linear approximation 
\begin{align}
    \bz^L(\theta, \boldsymbol{x} + \boldsymbol{\delta}) \approx \bz^L(\theta, \boldsymbol{x}) +  \boldsymbol{\delta}^\top \nabla_{\bx} \bz^L(\theta, \boldsymbol{x}) 
\end{align}
 is exact for small enough $\boldsymbol{\delta}$. \textcolor{black}{We can then approximately solve the adversarial problem for each class $k$ as 
 \begin{align}
     \max_{\boldsymbol{\delta}\in \mathcal{U}} \: \bz_k^L(\theta, \boldsymbol{x} + \boldsymbol{\delta}) - \bz_y^L(\theta, \boldsymbol{x} + \boldsymbol{\delta}) &\approx \max_{\boldsymbol{\delta}\in \mathcal{U}} \:(\be_k - \be_y)^\top\bz^L(\theta, \boldsymbol{x}) + \boldsymbol{\delta}^\top \nabla_x (\be_k - \be_y)^\top\bz^L(\theta, \boldsymbol{x})\nonumber\\
     & = (\be_k - \be_y)^\top\bz^L(\theta, \boldsymbol{x}) + \rho\|(\be_k - \be_y)^\top\nabla_x\bz^L(\theta, \boldsymbol{x})\|_q, \label{class-adversary}
 \end{align} where  $k, y\in[K]$ and  $\boldsymbol{e}_k$ (respectively $\boldsymbol{e}_y$)  is the one-hot vector with a $1$ in the $k^{th}$ 
 (respectively in the $y^{th}$) coordinate and $0$, everywhere else. Applying the result from Lemma \ref{lemma:gen_upper_bound} and defining $\boldsymbol{c}_k \coloneqq \boldsymbol{e}_k - \boldsymbol{e}_y$ we obtain the approximate robust upper bound:}
\begin{align}
    &\min_\theta \:\max_{\boldsymbol{\delta}\in \mathcal{U}} \: \mathcal{L}(y, \bz^L(\theta,\bx + \boldsymbol{\delta}))\nonumber\\
      \lessapprox &\min_\theta \: \mathcal{L}\left(y , \big( \boldsymbol{c}_1^\top\bz^L(\theta, \bx) + \rho \|\boldsymbol{c}_1^\top\nabla_x \bz^L(\theta, \bx)\|_q, \dots, \boldsymbol{c}_K^\top\bz^L(\theta, \bx) + \rho \|\boldsymbol{c}_K^\top \nabla_x \bz^L(\theta, \bx)\|_q \big)\right).
     \label{eq:approx-upper-bound}\\
     &\hspace{12.7cm}\nonumber
\end{align}
\noindent Therefore, we propose to train the network by minimizing Eq. \eqref{eq:approx-upper-bound} instead of the standard average loss, \textcolor{black}{and we refer to this defense as aRUB-$L_{p}$}. For the particular case of the cross entropy loss with softmax activation function in the output layer, the exact optimization problem to be solved would be the following:
\begin{align}
    \min_{\theta}\: \frac{1}{N}\sum_{n = 1}^N\log\left(\sum_{k} e^{ (\boldsymbol{e}_k - \boldsymbol{e}_{y_n})^\top\bz^L(\theta, \bx ) + \rho \|\nabla_{\bx} (\boldsymbol{e}_k - \boldsymbol{e}_{y_n})^\top\bz^L(\theta, \bx)\|_q}\right).\label{method1}
\end{align}
\textcolor{black}{This expression may not always be an upper bound of the adversarial loss (Eq. \eqref{minmax-problem}); however, we observe across a variety of experiments that usually it is indeed an upper bound (see Table \ref{table:upper_bound}). This suggests that the upper bound provided by Lemma 1 compensates for the errors introduced by the first order approximation of $\bz^L$}. \textcolor{black}{  Additionally, in Figure \ref{fig:loss_plot} we empirically show that Eq. \eqref{eq:approx-upper-bound} is much closer than Eq. \eqref{simple-approx-problem} to the adversarial loss in Eq. \eqref{minmax-problem} for small values of $\rho$.}\\

\begin{figure}[ht]
     \centering
     \begin{subfigure}[b]{0.45\linewidth}
         \centering
         \includegraphics[width=\linewidth]{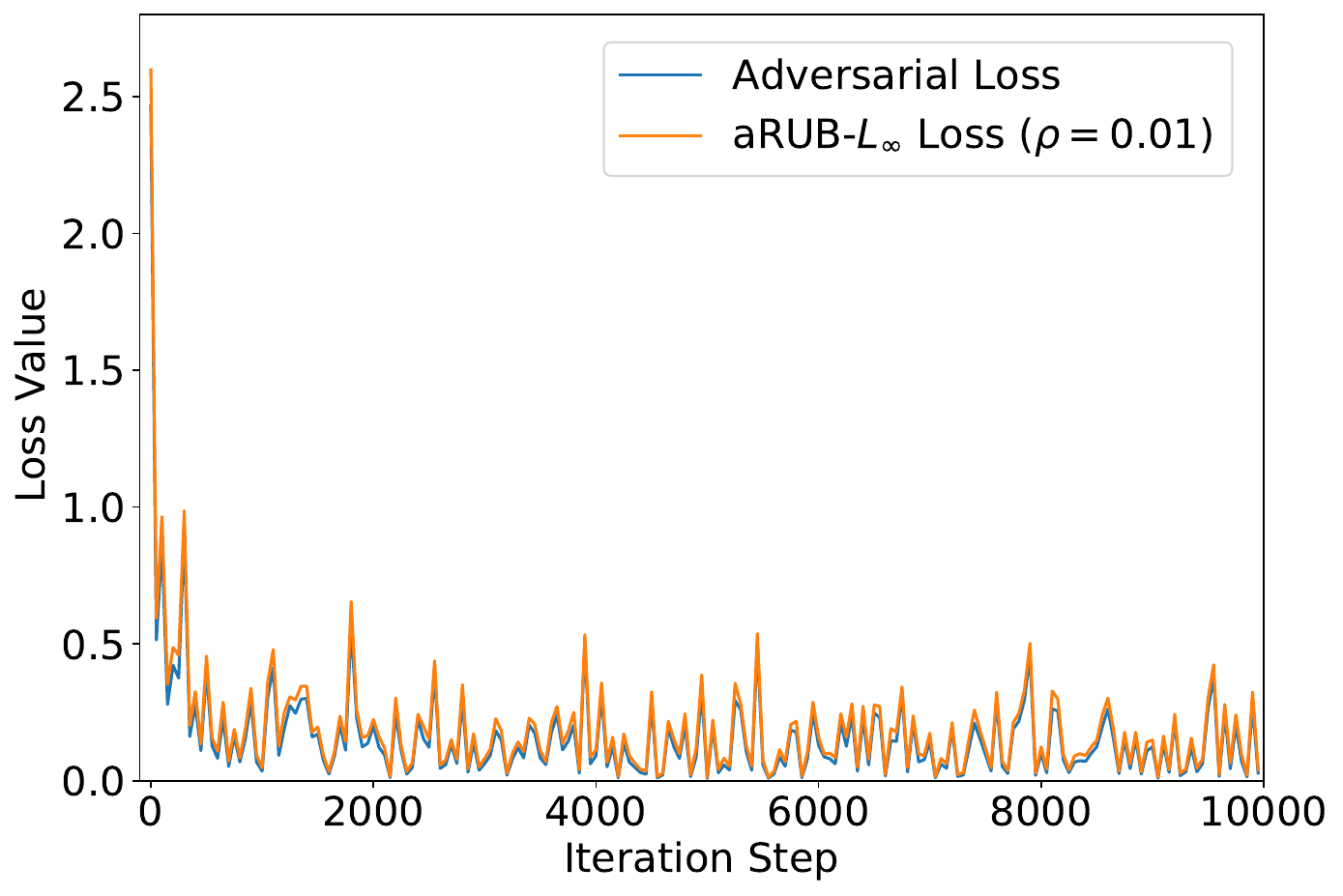}
         \caption{\textcolor{black}{aRUB-$L_{\infty}$, $\rho=0.01$}}
     \end{subfigure}
     \begin{subfigure}[b]{0.45\linewidth}
         \centering
         \includegraphics[width=\textwidth]{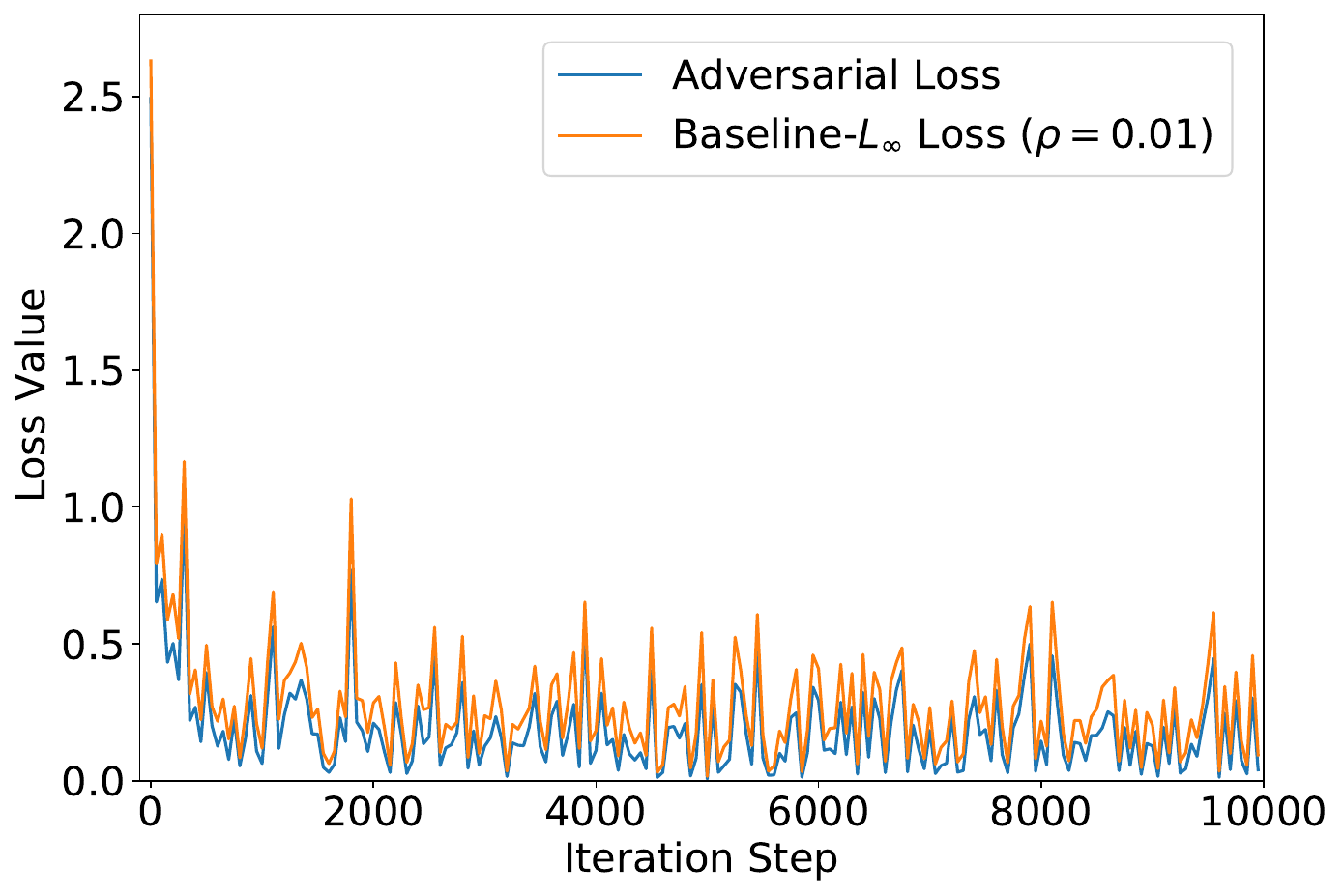}
         \caption{\textcolor{black}{Baseline-$L_{\infty}$, $\rho=0.01$}}
     \end{subfigure}\\
     \hfill\\
     \begin{subfigure}[b]{0.45\linewidth}
         \centering
         \includegraphics[width=\linewidth]{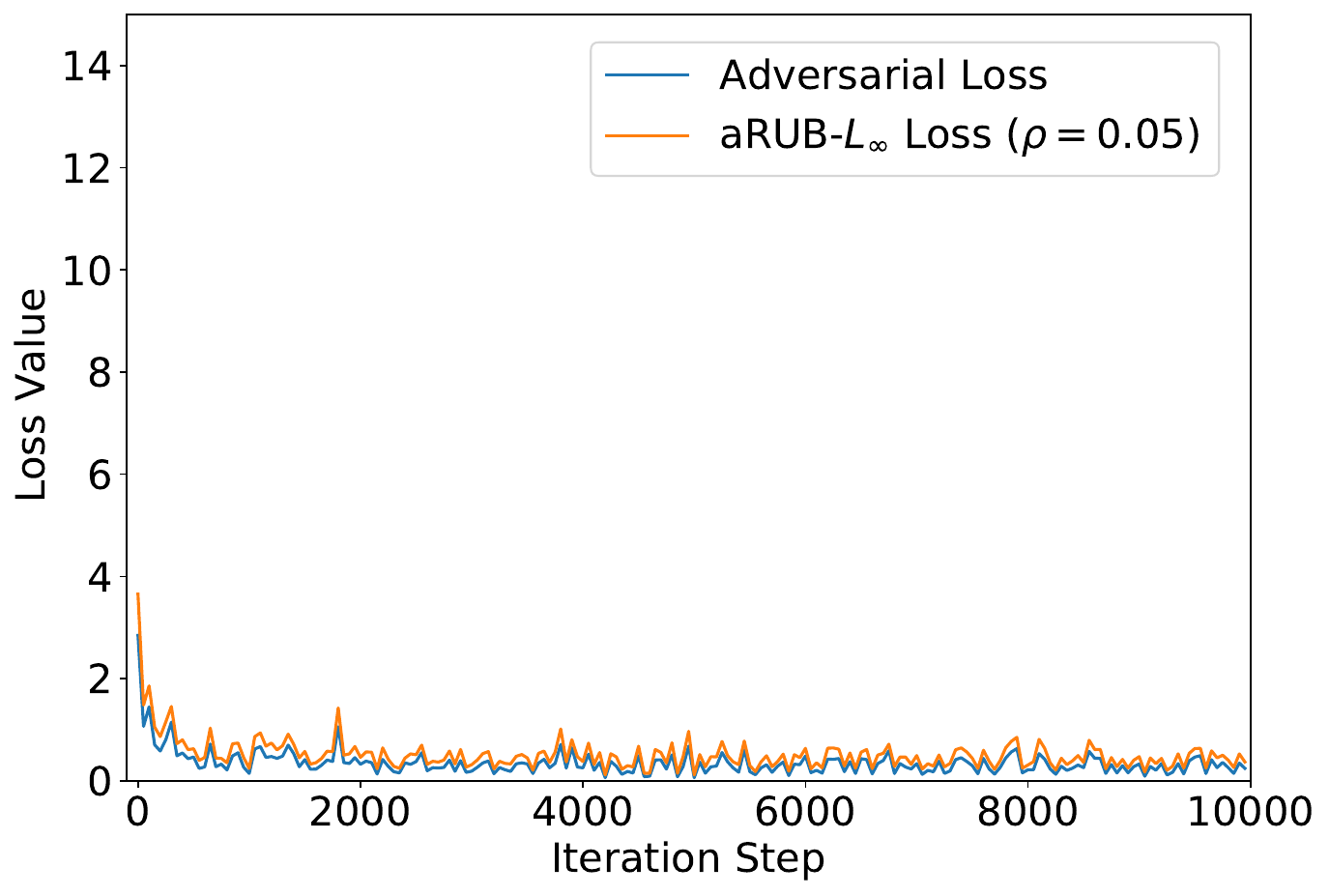}
         \caption{\textcolor{black}{aRUB-$L_{\infty}$, $\rho=0.05$}}
     \end{subfigure}
     \begin{subfigure}[b]{0.45\linewidth}
         \centering
         \includegraphics[width=\textwidth]{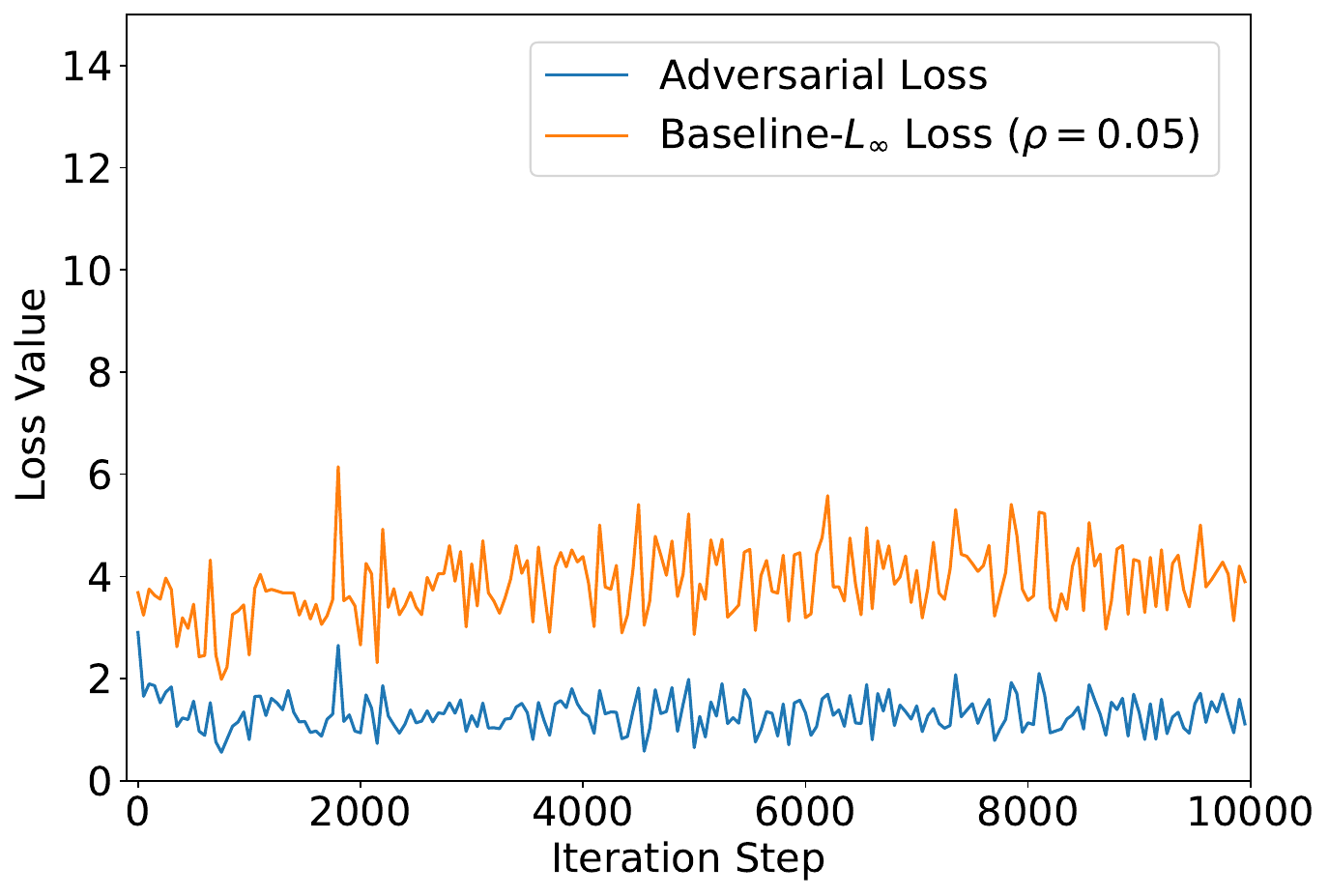}
         \caption{\textcolor{black}{Baseline-$L_{\infty}$, $\rho=0.05$}}
     \end{subfigure}\\
     \hfill \\
     \begin{subfigure}[b]{0.45\linewidth}
         \centering
         \includegraphics[width=\linewidth]{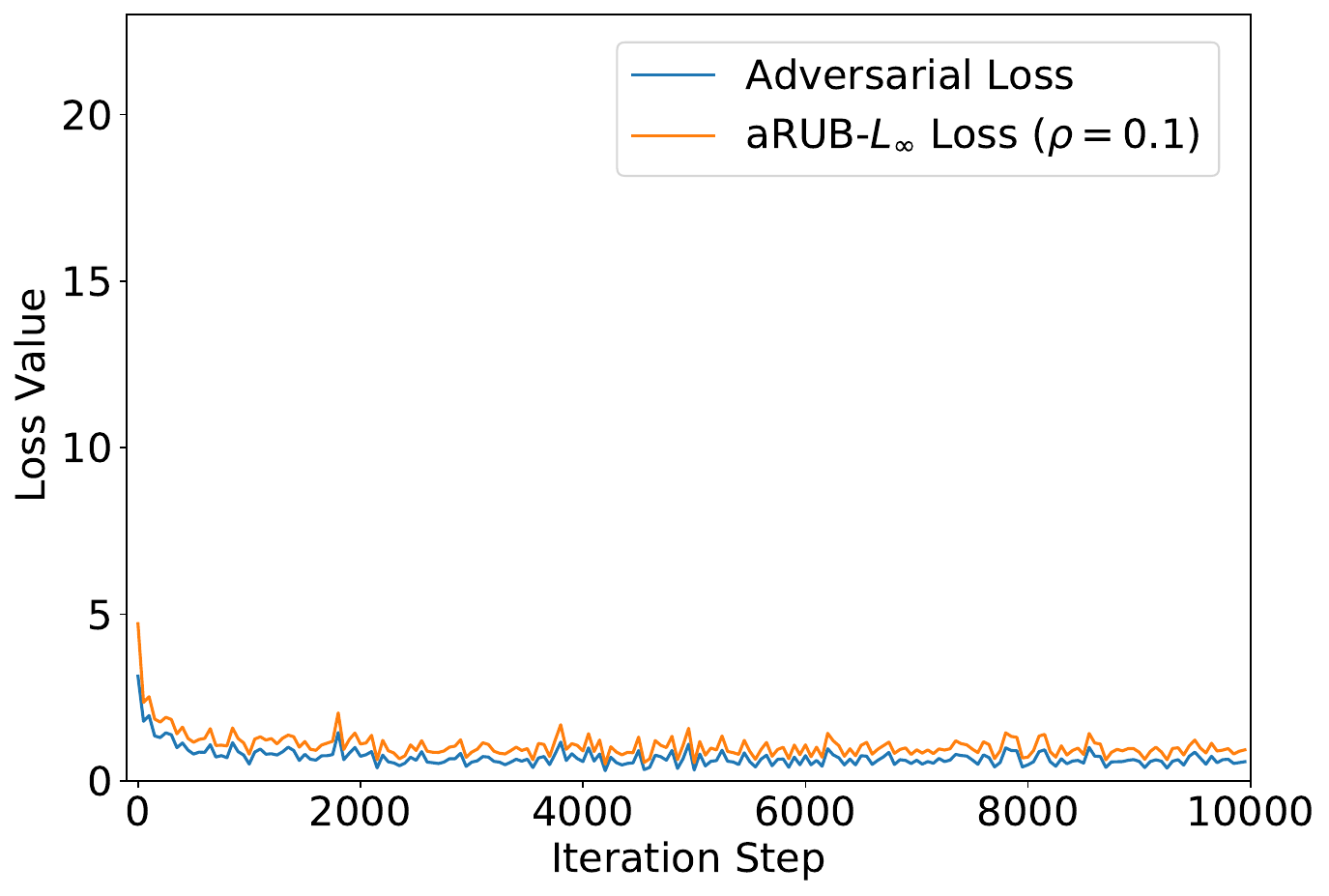}
         \caption{\textcolor{black}{aRUB-$L_{\infty}$, $\rho=0.1$}}
     \end{subfigure}
     \begin{subfigure}[b]{0.45\linewidth}
         \centering
         \includegraphics[width=\textwidth]{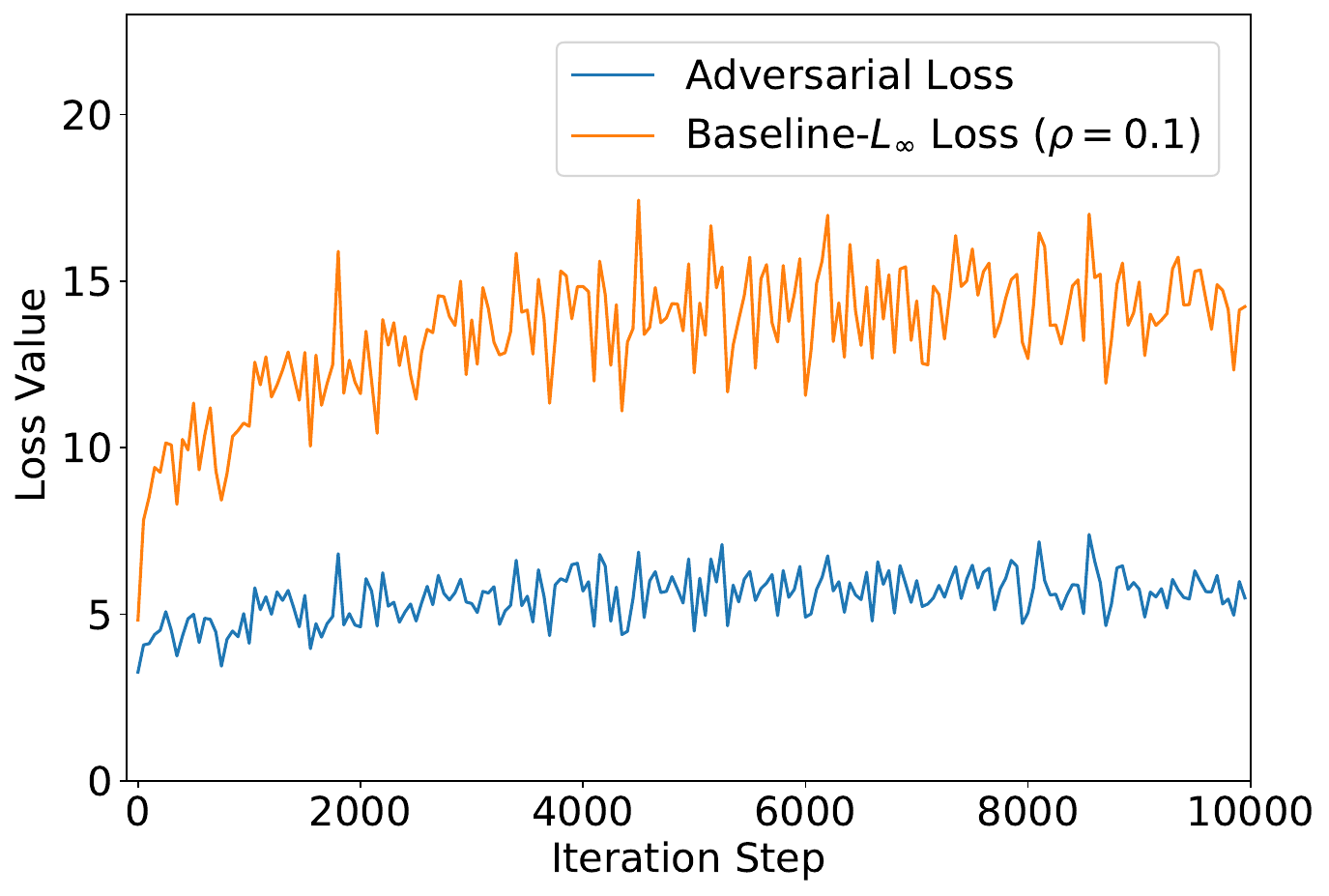}
         \caption{\textcolor{black}{Baseline-$L_{\infty}$, $\rho=0.1$}}
     \end{subfigure}\\
     \hfill \\
        \caption{Adversarial loss (cross entropy loss evaluated at adversarial images bounded in $L_{\infty}$ norm) vs the loss function being minimized across training iterations (aRUB-$L_\infty$ on the left and Baseline-$L_\infty$ on the right); the value of $\rho$ increases for lower rows. \textcolor{black}{Experiments are in the MNIST data set with the infinity norm $(p = \infty)$. We use a feed forward neural network with three hidden layers with the softmax-cross-entropy loss (see subsection~\ref{subsec:61} for details). We used a learning rate of $10^{-3}$ and a batch size of $32$, which we found to work best across the experiments in this figure.}}
        \label{fig:loss_plot}
\end{figure}
\begin{table}[ht]
\centering
\begin{tabular}{lrrrrrrrrrrr}
\toprule
  &    $\boldsymbol{\rho} $ =   $\boldsymbol{0.0}$ &       $\boldsymbol{0.0008}$ &     $\boldsymbol{0.001}$ &    $\boldsymbol{0.0015}$ &     $\boldsymbol{0.002}$ &     $\boldsymbol{0.003}$ &      $\boldsymbol{0.01}$ &       $\boldsymbol{0.1}$ &       $\boldsymbol{0.3}$ &       $\boldsymbol{0.5}$ &       $\boldsymbol{1.0}$ \\
\midrule
 $L_\infty$ &  $94\%$ &   $99\%$ &  $99\%$ &  $99\%$ &  $99\%$ &  $99\%$ &  $99\%$ &  $95\%$ &  $86\%$ &  $81\%$ &  $79\%$ \\
 $L_1$ &  $93.3\%$ &   $99\%$ &  $99\%$ &  $99\%$ &  $99\%$ &  $99\%$ &  $99\%$ &  $99\%$ &  $99\%$ &  $99\%$ &  $98\%$ \\
\bottomrule 
\end{tabular}\caption{Percentage of times when the aRUB approach yields an upper bound of the adversarial loss with respect to PGD attacks, \textcolor{black}{i.e., percentage of times when Eq. \eqref{method1} is larger than Eq. \eqref{minmax-problem} evaluated using PGD-$L_p$ attacks. For each row, aRUB-$L_p$ and the PGD-$L_p$ attacks use the $L_p$ norm indicated on the first column. The loss function utilized is the cross entropy with softmax. Percentages are computed across all networks trained (ie., for all tested hyperparameters) in the 46 UCI data sets, every $500$ training steps \textcolor{black}{(see subsections \ref{subsec:61} and \ref{subsec:UCI} for more details about the networks and data sets)}. In this way, the approximate bound is evaluated in a large number of different conditions, including data sets, training steps and hyperparameters. }}\label{table:upper_bound}
\end{table}

\section{Robust Upper bound for the \texorpdfstring{$L_1$}{Lg} norm and general \texorpdfstring{$\rho$}{Lg}.}\label{sec:RUB}
In this section we derive a \textcolor{black}{provable} upper bound for the robust counterpart of the inner maximization problem in Eq. \eqref{minmax-problem} for the specific case in which the uncertainty set is the ball of radius $\rho$ in the $\text{L}_1$ space; i.e. $\mathcal{U} = \{{\boldsymbol{\delta}} : \|{\boldsymbol{\delta}}\|_1 \leq \rho\}$. 

In the \textcolor{black}{RO} framework, finding a \textcolor{black}{provable upper bound} for a problem that is convex (or concave) in the uncertainty parameters relies on using convex (or concave) conjugate functions to make the problem linear in the uncertainty \citep{bertsimasbook2022}. The following lemmas are a generalization of this approach for the case in which the objective involves the composition of two functions, with the goal of making the problem linear not in the uncertainty but in the inner most function. Lemma \ref{convex-bound} makes this generalization when the outer function is convex while Lemma \ref{concave-bound} focuses on the concave case. \textcolor{black}{The proofs of these lemmas rely on the definition of conjugate functions as well as the Fenchel duality theorem \citep{rockafellar-1970a}, and can be found in Appendix \ref{sec:appendixA}. Together, these lemmas are the core of the methodology presented in this section.}
 
 \begin{lemma} \label{convex-bound}
If $f:A\rightarrow B$ is a convex and closed function then for any function $\bz:\mathcal{U} \rightarrow {A}$, and any function $g:A\rightarrow B$ we have
\begin{align*}
\sup_{\boldsymbol{\delta} \in \mathcal{U}}\: f(\boldsymbol{z}(\boldsymbol{\delta})) + g(\boldsymbol{z}(\boldsymbol{\delta})) = \sup_{\bu\in \text{dom}(f^\star)} \enspace \sup_{\boldsymbol{\delta} \in \mathcal{U}}\:  \boldsymbol{z}(\boldsymbol{\delta})^T\bu - f^\star(\bu) + g(\boldsymbol{z}(\boldsymbol{\delta})),
\end{align*}
where the convex conjugate function $f^\star$ is defined by $f^\star(\bz) = \sup_{\bx\in\text{dom}(f)} \: \bz^T\bx - f(\bx)$.
 \end{lemma}
 \vspace{-0.5cm}
 \begin{proof}
 See Appendix \ref{A2}.\hfill
 \end{proof}
 \begin{lemma} \label{concave-bound}
 Let $g:A \rightarrow B$ be a concave and closed function. If a function $\bz:\mathcal{U} \rightarrow {A}$ satisfies  $g(\bz(\boldsymbol{\delta}))<\infty$ for all $\boldsymbol{\delta}\in \mathcal{U}$, then
 \[
\sup_{\boldsymbol{\delta} \in \mathcal{U}}\: g(\bz(\boldsymbol{\delta})) = \inf_{\bv\in \text{dom}(g_\star)} \enspace  \sup_{\boldsymbol{\delta} \in \mathcal{U}}\:  \bz(\boldsymbol{\delta})^T\bv - g_\star(\bv),
\]
where the concave conjugate function is defined by $g_\star(\bz) = \inf_{\bx\in \text{dom}(g)} \: \bz^T\bx - g(\bx)$.
 \end{lemma}
  \vspace{-0.5cm}
 \begin{proof}
 See Appendix \ref{A3}.\hfill
 \end{proof}
From the lemmas above we can observe that to apply them we will need to compute convex and concave conjugate functions. The next lemma facilitates \textcolor{black}{these} computations for neural networks with ReLU activation functions.
 \begin{lemma}\label{conjugate-computation}
If $\bu, \bp, \bq\geq \boldsymbol{0}$, then the functions $f(\bx) = \bp^\top [\bx]^+$ and $g(\bx) = \bx^T\bu - \bq^\top[\bx]^+$ satisfy
\begin{align*}
    a) \: f^\star(\bz) = \begin{cases} 0& \quad \text{if    } \boldsymbol{0}\leq \bz \leq \bp, \\ \infty& \quad \text{otherwise.} \end{cases} \quad \text{and} \quad b) \:g_\star(\bz) = \begin{cases} 0& \quad \text{if    } \bu-\bq\leq \bz \leq \bu, \\ -\infty& \quad \text{otherwise.} \end{cases}
\end{align*}
 \end{lemma}
  \vspace{-0.5cm}
 \begin{proof}
 See Appendix \ref{A4}.\hfill
 \end{proof}
 \hfill \\

\textcolor{black}{As observed in \textcolor{black}{Lemma \ref{lemma:gen_upper_bound}}, we can obtain an upper bound of the min-max problem in Eq. \eqref{minmax-problem} by finding instead upper bounds for \textcolor{black}{$\max_{\boldsymbol{\delta}\in \mathcal{U}}\: \boldsymbol{z}_k^L(\theta, \bx + \boldsymbol{\delta}) - \boldsymbol{z}_y^L(\theta, \bx + \boldsymbol{\delta})$} for each class $k$. We will find these upper bounds in the following three steps:
\begin{itemize} 
\item \textit{Step \#1 - Linearize the uncertainty:} We first make the maximization problem over the uncertainty set $\mathcal{U}$ linear in the first layer of the network (and therefore also linear in the uncertainty $\boldsymbol{\delta}$). Starting from the last layer we recursively split the objective function as the sum of a convex function and a concave function, and we then apply \textcolor{black}{L}emmas \ref{convex-bound} and \ref{concave-bound} to make the maximization problem over $\mathcal{U}$ linear in the previous layer. 
\item \textit{Step \#2 - Optimize over the uncertainty set:} Then, we solve the maximization problem over $\mathcal{U}$. This problem can be exactly solved because the first layer of the network is linear in the uncertainty and the dual function of the $L_1$ norm is the $L_\infty$ norm (a maximum over a finite set). 
\item \textit{Step \#3 - Backtrack:} Finally, we backtrack to solve for the variables $\bu, \bv$ that \textcolor{black}{L}emmas \ref{convex-bound} and \ref{concave-bound} introduce.
\end{itemize}}
 For simplicity, we develop and prove the upper bound for the robust counterpart assuming that the neural network has only two layers; however, the results can be extended to the general case as shown in Appendix \ref{sec:appendixB}. In addition, all the theorems can be generalized for convolution\textcolor{black}{al} neural networks, since convolutions are a special case of matrix multiplication.
 
\subsection{\textbf{\textit{Step \#1 - Linearize the uncertainty}}}

\textcolor{black}{\noindent The following theorem shows how the maximization problem over $\mathcal{U}$ can be transformed from a linear problem in the second layer to a linear problem in the first layer of the network. The proof relies on Lemma \ref{convex-bound} and Lemma \ref{concave-bound}}.
 
\begin{theorem}
The maximum difference between the output of the correct class and the output of any other class $k$ can be written as
 \begin{align}
     \textcolor{black}{\sup_{{\boldsymbol{{\delta}}} \in \mathcal{U}} \: \boldsymbol{z}_k^2(\theta, \bx + \boldsymbol{\delta}) - \boldsymbol{z}_y^2(\theta, \bx + \boldsymbol{\delta}) =} &\quad \sup_{{\boldsymbol{{\delta}}} \in \mathcal{U}} \: \bc_k^\top \bz^2(\theta, \bx + \boldsymbol{{\delta}}) \\
     \begin{split}
     =  &  \sup_{0\leq {\bs}\leq 1} \inf_{0\leq {\bt}\leq 1}\sup_{\boldsymbol{{\delta}} \in \mathcal{U}}  \: (\bp - \bq)^\top\: \bz^{1}(\theta, \bx + {\boldsymbol{{\delta}}} ) + \bc_k^\top\bb^2\\
     & \hspace{1.5cm}\text{s.t.} \quad \bp = [(\bbw^2)^\top\bc_k]^+\odot \bs\\
     &  \hspace{2.4cm} \bq = [-(\bbw^2)^\top\bc_k]^+\odot \bt, 
     \label{induction-problem1}
     \end{split}
 \end{align}
 where $\odot$ corresponds to entry-wise multiplication. \label{general-decomposition1}
 \end{theorem}
 
\begin{proof}
 By definition of the two layer neural network $\bz^2$, we have 
 \begin{align*}
     \bc_k^\top\bz^2(\theta, \bx + {\boldsymbol{\delta}})  &= \bc_k^\top  \bbw^2 [\bz^{1}(\theta, \bx + {\boldsymbol{\delta}})]^+ + \bc_k^\top\bb^2 \\
     &= f_+(\bz^{1}(\theta, \bx + {\boldsymbol{\delta}}))  - f_{-}(\bz^{1}(\theta, \bx + {\boldsymbol{\delta}}))+ \bc_k^\top\bb^2,
 \end{align*}
where $f_+, f_{-}$ are the convex functions defined by 
\[f_+(\bx) = [ \bc_k^\top  \bbw^2 ]^+[\bx]^+, \quad \text{ and } \quad   f_{-}(\bx) = [ -\bc_k^\top \bbw^2 ]^+[\bx]^+.\] 
Applying Lemma \ref{convex-bound} to the function $f_+$ we then have
 \begin{align}
     &\sup_{{\boldsymbol{\delta}} \in \mathcal{U}} \: \bc_k^\top \bz^2(\theta, \bx + {\boldsymbol{\delta}} ), \nonumber\\
     =& \sup_{{\boldsymbol{\delta}} \in \mathcal{U}} \: f_+(\bz^{1}(\theta, \bx + {\boldsymbol{\delta}}))  - f_{-}(\bz^{1}(\theta, \bx + {\boldsymbol{\delta}})) + \bc_k^\top\bb^2,\nonumber \\
     =& \sup_{\bu\in \text{dom}(f_+^\star)} \: \sup_{{\boldsymbol{\delta}}  \in \mathcal{U}} \: \:  \bu^\top \bz^{1}(\theta, \bx + {\boldsymbol{\delta}}) - f_+^\star(\bu)- f_{-}(\bz^{1}(\theta, \bx + {\boldsymbol{\delta}}))+ \bc_k^\top\bb^2,\quad& &(\text{By Lemma \ref{convex-bound}}),\\
     =& \sup_{\bu\in \text{dom}(f_+^\star)} \: \sup_{{\boldsymbol{\delta}}  \in \mathcal{U}} \: \:  \bu^\top \bz^{1}(\theta, \bx + {\boldsymbol{\delta}}) - f_{-}(\bz^{1}(\theta, \bx + {\boldsymbol{\delta}}))+ \bc_k^\top\bb^2,\quad& &(\text{By Lemma \ref{conjugate-computation}a)}.\label{eq:concave-reduction}
 \end{align}
Defining the concave function $g(\bx) = \bu^\top \bx - f_{-}(\bx)$, and applying Lemma \ref{concave-bound} to the function $g$ we obtain
 \begin{align}
     &\sup_{{\boldsymbol{\delta}} \in \mathcal{U}} \: \bc_k^\top\bz^2(\theta, \bx + {\boldsymbol{\delta}}) \nonumber \\
     =& \sup_{\bu\in \text{dom}(f_+^\star)} \:  \sup_{{\boldsymbol{\delta}} \in \mathcal{U}}\:  g(\bz^1(\theta, \bx +\boldsymbol{\delta})) + \bc_k^\top\bb^2\\
     =& \sup_{\bu\in \text{dom}(f_+^\star)} \: \inf_{\bv \in \text{dom}(g_\star)}\: \sup_{{\boldsymbol{\delta}} \in \mathcal{U}}\:   \bv^\top\bz^{1}(\theta, \bx + {\boldsymbol{\delta}}) - g_\star(\bv) + \bc_k^\top\bb^2\quad&  &(\text{By Lemma \ref{concave-bound}}),\\
     =& \sup_{\bu\in \text{dom}(f_+^\star)} \: \inf_{\bv \in \text{dom}(g_\star)}\: \sup_{{\boldsymbol{\delta}} \in \mathcal{U}}\:   \bv^\top\bz^{1}(\theta, \bx + {\boldsymbol{\delta}}) + \bc_k^\top\bb^2\quad&  &(\text{By Lemma \ref{conjugate-computation}b}).\label{concave-part1}
 \end{align}
Lastly, by Lemma \ref{conjugate-computation}a and \ref{conjugate-computation}b, we know that the variables $\bu$ and $\bv$ can be parameterized as
\begin{align*}
    \bu &= [(\bbw^2)^\top \bc_k ]^+\odot \bs, \\
    \bv &= [ (\bbw^2)^\top \bc_k ]^+\odot \bs - [ -(\bbw^2)^\top \bc_k ]^+\odot \bt
\end{align*}
with $0\leq \bs, \bt\leq 1$. Substituting these values in Eq. \eqref{concave-part1} we obtain Eq. \eqref{induction-problem1}, as desired. \hfill 
\end{proof}
\subsection{\textbf{\textit{Step \#2 - Optimize over the uncertainty set}}}

\noindent Notice that the objective in Eq. \eqref{induction-problem1} is linear in $\bz^1$ and therefore it is also linear in ${\boldsymbol{\delta}}$, which facilitates the computation of the exact value of the supremum over $\mathcal{U}$, as shown in the next corollary.

\begin{corollary} \label{cor:p-norm1}
If $\mathcal{U} = \{\boldsymbol{\delta} : \|\boldsymbol{\delta}\|_p \leq \rho\}$, then:
 \begin{align}
     &\sup_{{\boldsymbol{\delta}} \in \mathcal{U}} \: \bc_k^\top\bz^2(\theta, \bx + {\boldsymbol{\delta}})  \\
     \begin{split}
     =& \sup_{0\leq {\bs}\leq 1} \inf_{0\leq {\bt}\leq 1} \: \rho \|(\bp - \bq)^\top\bbw^1\|_{q} + (\bp - \bq)^\top(\bbw^1\bx+ \bb^1) +  \bc_k^\top\bb^2  \\
     & \hspace{1.5cm}\text{s.t.} \quad \bp = [(\bbw^2)^\top\bc_k]^+\odot \bs\\
     &  \hspace{2.4cm} \bq = [-(\bbw^2)^\top \bc_k]^+\odot \bt,
     \label{conjugate_problem1}
     \end{split}
 \end{align}
 where $\|\cdot \|_q$ is the dual norm of $\|\cdot\|_p$, with $\frac{1}{p} + \frac{1}{q} = 1$.
 \end{corollary}
 Before proceeding to the proof of the corollary, notice that we can recover the approximation method developed in the previous section by setting 
 \begin{align}
     \bs = \bt = [sign(\bz^1(\theta, \bx))]^+\label{linear-approx-vals}
 \end{align} in the objective of problem \eqref{conjugate_problem1} to obtain \begin{align}  \rho\|((\bbw^2)^\top \bc_k\odot[sign(\bz^1(\theta, \bx))]^+)^\top \bbw^1 \|_q +\bc_k^\top\bbw^2[\bz^1(\theta, \bx)]^+ + \bc_k^\top\bb^2,\end{align} which is the same as the linear approximation of $\bc_k^\top\bz^2(\theta + \boldsymbol{\delta})$ obtained in Eq. \eqref{class-adversary}.
 
 \begin{proof}
  The proof follows directly after applying Theorem \ref{general-decomposition1}  and using the fact that for all vectors $\boldsymbol{\bc}$ we have
  \begin{align}
      \sup_{{\boldsymbol{\delta}} : \|\boldsymbol{\delta}\|_p \leq \rho} \bc^\top \boldsymbol{\delta}  = \rho\|\bc\|_q.\label{robust_trick}
  \end{align}
 \end{proof}

\subsection{\textbf{\textit{Step \#3 - Backtrack}}}
Since neural networks are trained by minimizing the empirical loss over the parameters $\theta$, we want to avoid the computation of supremums in the objective. While the previous corollary shows how to solve the supremum over the uncertainty set, a new supremum was introduced in Theorem \ref{general-decomposition1} over the variables $\bs$. The next theorem tells us how we can remove this new supremums for the specific case $p = 1$.

\begin{theorem}\label{last} The maximum difference between the output of the correct class and the output of any other class $k$ can be upper bounded by
\begin{align}
\sup_{{\boldsymbol{\delta}}: \|\boldsymbol{\delta}\|_1 \leq \rho} &\: \bc_k^\top\bz^2(\theta, \bx + {\boldsymbol{\delta}})   \nonumber \\
       \leq &  \inf_{0\leq {\bt}\leq 1} \: \max_{m\in [M]} \: \max\bigg\{\  g^2_{k, m}(\theta, \bt,\bx, \rho), g^2_{k, m}(\theta, \bt, \bx, -\rho)\bigg\},
\end{align}
where the new network $g$ is defined by the equations
\begin{align*}
    g^1_m(\theta, a) &= a \bbw^1\boldsymbol{e}_m + \bbw^1\bx + \bb^1,\\
    g^2_{k,m}(\theta, \bt, \bx, a)& = [\bc_k^\top \bbw^2]^+[g^1_m(\theta, a)]^+ - [-\bc_k^\top \bbw^2]^+[g^1_m(\theta, a)]\odot \bt + \bc_k^\top\bb^2, 
\end{align*}
for $a = \rho, -\rho$.
\end{theorem}
\begin{proof}
Applying Corollary \ref{cor:p-norm1} with $p =1$ and using the min-max inequality we obtain
\begin{align}
     &\sup_{{\boldsymbol{\delta}}: \|{\boldsymbol{\delta}}\|_1 \leq \rho} \: \bc_k^\top\bz^2(\theta, \bx + {\boldsymbol{\delta}})  \\
     \begin{split}
     \leq & \inf_{0\leq {\bt}\leq 1} \sup_{0\leq {\bs}\leq 1}  \rho \|(\bp - \bq)^\top\bbw^1\|_{\infty} + (\bp - \bq)^\top(\bbw^1\bx+ \bb^1) + \bc_k^\top\bb^2\\
     & \hspace{1.5cm}\text{s.t.} \quad \bp = [(\bbw^2)^\top\bc_k]^+\odot \bs\\
     &  \hspace{2.4cm} \bq = [-(\bbw^2)^\top \bc_k]^+\odot \bt.  
     \label{conjugate_problem_proof}
     \end{split}
 \end{align}
Defining $\bp(\bs) = [(\bbw^2)^\top\bc_k]^+\odot \bs$ and  $\bq(\bt) = [-(\bbw^2)^\top \bc_k]^+\odot \bt$, we have that for fixed ${\bt}$ it holds\\
\begin{align*}
    &\sup_{0\leq {\bs}\leq 1} \: \rho \|(\bp(\bs) - \bq(\bt))^\top\bbw^1\|_{\infty} + (\bp(\bs) - \bq(\bt))^\top(\bbw^1\bx+ \bb^1) + \bc_k^\top\bb^2\\
    =& \max_{m\in [M]}\: \max\bigg\{\sup_{0\leq {\bs}\leq 1}\: ( \bp(\bs) -\bq(\bt))^\top(\bbw^1(\bx + \rho \be_m) + \bb^1)+ \bc_k^\top\bb^2,\\
    & \hspace{2.1cm}\sup_{0\leq {\bs}\leq 1} \: (\bp(\bs) - \bq(\bt))^\top(\bbw^1(\bx - \rho \be_m) + \bb^1)+ \bc_k^\top\bb^2\bigg\}\\
    &\text{\hfill }\\
    =& \max_{m\in [M]} \: \max\bigg\{\  [\bc_k^\top\bbw^2]^+ [\bbw^1(\bx + \rho \be_m) + \bb^1]^+- \bq(\bt)^\top(\bbw^1(\bx + \rho \be_m) + \bb^1)+ \bc_k^\top\bb^2, \\
    & \hspace{2.1cm} [(\bc_k^\top\bbw^2]^+[\bbw^1(\bx - \rho \be_m) + \bb^1]^+ - \bq(\bt)^\top (\bbw^1(\bx - \rho \be_m) + \bb^1)+ \bc_k^\top\bb^2\bigg\}\\
    =& \max_{m\in [M]} \: \max\{g^2_{k,m}(\theta, \bt, \bx, \rho), g^2_{k,m}(\theta, \bt, \bx, -\rho)\}.
\end{align*}
The theorem then follows after applying the $\inf$ over $0\leq \bt\leq 1$.
\end{proof}
Notice that in the previous proof it was crucial to use $p=1$, since the dual of the $L_1$ norm is the $L_\infty$ norm, which can be written as a maximum over a finite set. With a different $p$ we would not have been able to apply the same technique to remove the variables $\bs$. However, for the chosen uncertainty set we obtain an upper bound of Eq. \eqref{minmax-problem} by applying the result from the previous Theorem to \textcolor{black}{Lemma \ref{lemma:gen_upper_bound}}.

While we could include the variables $\bt$ in the minimization problem over $\theta$, we instead use fixed values $\bt = [sign(\bz^1(\theta, \bx))]^+$ based on the linear approximation of $\bc_k^\top\bz^2(\theta, \bx + {\boldsymbol{\delta}})$, as described in Eq. \eqref{linear-approx-vals}. Notice that setting specific values for $\bt$ does not affect the inequalities: since the upper bound includes the infimum over $\bt$, any $0\leq \bt\leq 1$ yields an upper bound of the robust problem. For the specific case of the cross entropy loss function, the proposed upper bound for the min-max robust problem is 
\begin{align}
\min_\theta \frac{1}{N} \sum_{n = 1}^N \log\bigg(\sum_k e^{\big({\max_{m\in[M]} \max\{g^2_{k,m}(\theta, [sign(\bz^1(\theta, \bx))]^+, \bx, \rho), g^2_{k,m}(\theta, [sign(\bz^1(\theta, \bx))]^+, \bx, -\rho)\}}\big)}\bigg).\label{eq:upper-bound}
\end{align}

\section{Experiments}\label{sec:experiments}
In this section, we demonstrate the effectiveness of the proposed methods in practice. We first introduce the experimental setup and then we compare the robustness of several defenses.
\subsection{Experimental details}\label{subsec:61}
\paragraph{Data sets.} We use 46 data sets from the UCI collection~\citep{uci}, which correspond to classification tasks with a diverse number of features that are not categorical. For each data set we do a $80\% / 20\%$ split for training/testing sets, and we further reserve  $25\%$  of each training set for validation. In addition, we use three popular computer vision data sets, namely the MNIST~\citep{deng2012mnist}, Fashion MNIST~\citep{xiao2017fashionmnist} and CIFAR~\citep{Krizhevsky09learningmultiple} data sets. 
\paragraph{Pre-processing.} All input data has been previously scaled, which facilitates the comparison of the adversarial attacks across data sets. For the UCI data sets, each feature is standarized using the statistics of the training set, while for the vision data sets each image channel is normalized to be between $0$ and $1$, or standarized, depending on what leads to best robustness.

\paragraph{Attacks.} We use the implementation provided by the foolbox library~\citep{rauber2017foolbox, rauber2017foolboxnative} using the default parameters. We evaluate attacks using projected gradient descent and fast gradient methods. More specifically, we use the following adversarial attacks:
\begin{itemize}
    \item PGD-$L_p$: Attack bounded in $L_p$ norm and found using Projected Gradient Descent.
    \item FGM-$L_p$: Attack bounded in $L_p$ norm and found using Fast Gradient Method for $p=1,2$ and Fast Gradient Sign Method for $p =\infty$.
\end{itemize}

\paragraph{Defenses.} Our comparisons include different defenses denoted as follows:
\begin{itemize}
    \item aRUB-$L_p$: Approximate Robust Upper Bound method described in section \ref{sec:aRUB} using the $L_1$ or $L_\infty$ sphere as the uncertainty set.
    \item RUB:  Robust Upper Bound method described in section \ref{sec:RUB} using the $L_1$ sphere as the uncertainty set.
    \item PGD-$L_\infty$: Adversarial training method in which the \textcolor{black}{network is trained using} attacks that are bounded in the $L_\infty$ norm and found using Projected Gradient Descent.
    \item Baseline-$L_{\infty}$: Simple approximation method resulting from minimizing Eq. \eqref{simple-approx-problem} using the $L_p$ sphere as the uncertainty set.
    \item Nominal: Standard vanilla training with no robustness ($\rho = 0$).
\end{itemize}

\paragraph{Architecture.} We evaluate a neural network with three dense hidden layers with 200 neurons in each hidden layer. For the vision data sets, we also provide results with Convolutional Neural Networks (\emph{CNNs}) in Appendix \ref{sec:appendixD}. The architecture has two convolutional layers alternated with pooling operations, and two dense layers, as in~\cite{madry2019deep}. \textcolor{black}{The parameters of the networks were initialized with the Glorot initialization~\citep{glorot2010understanding}.}

\paragraph{Hyperparameter Tuning.} Each network and defense is trained for different learning rates ($\{1, 10^{-1}, 10^{-2}, 10^{-3}, 10^{-4}, 10^{-5}, 10^{-6}\}$). For the UCI data set we use a batch size of $256$ and for the vision data sets we try a batch size of $32$ and $256$.
For the $L_\infty$ based training methods we try all values of $\rho$ from the set ($\{10^{-4}, 10^{-5}, 10^{-3}, 10^{-2}, 0.1, 0.3$, $0.5, 1,$ $3, 5,10\}$). For the methods based on the $L_1$ norm, we scale those values of $\rho$ by a factor of $\sqrt{m}$, since $\|\bx\|_\infty \leq \|\bx\|_2$ and $\|\bx\|_1\leq \sqrt{m}\|\bx\|_2$ for any $\bx\in \mathbb{R}^m$.  In this way we ensure that the $L_1$ spheres and the $L_\infty$ spheres contain the same $L_2$ spheres, allowing for a fair comparison of all methods in terms of adversarial attacks that are bounded in the $L_2$ norm.  All networks trained using the UCI
data sets are trained for $5000$ iterations, and all vision data sets are trained for
$10000$ \textcolor{black}{iterations}. \textcolor{black}{For each network, data set, batch size, and defense radius $\rho$, we have verified that for at least one of the learning rates the validation accuracy converges with the aforementioned number of training iterations.}  Finally, for each attack type with radius $\rho$, we select on the validation set the best hyperparameters for each defense, \textcolor{black}{i.e., given a data set, an attack type and its radius $\rho$, the hyperparameters of a defense (network, learning rate, batch size, normalization and defense radius $\rho$) are the ones that lead to the highest adversarial robustness in the validation set. In total we trained more than $40,000$ networks across all tested data sets and defenses.}

\subsection{UCI data sets}\label{subsec:UCI}
We run experiments on the $46$ UCI data sets using different methods for robust training and compare the adversarial accuracies achieved with multiple types of adversarial attacks. For each data set we rank every training method, where the method with rank $1$ corresponds to the one with highest adversarial accuracy. The \textcolor{black}{average ranks and the corresponding $95\%$ confidence intervals} are shown in Figure \ref{fig:ranks}, where we observe a similar pattern across all types of attacks, namely, we see that the best ranks are achieved with aRUB-$L_\infty$ and PGD-$L_{\infty}$ when $\rho$ is smaller than $10^{-1}$; next there is a small range in which PGD-$L_\infty$ does best and finally for larger values of $\rho$ the best rank is that of RUB. \textcolor{black}{In addition, for large values of $\rho$ we observe better results with Baseline-$L_\infty$ than with aRUB-$L_\infty$, suggesting that the linear approximation of the network becomes inaccurate and leads to a large change in the loss function.} We also highlight that looking at $\rho = 0$, it is clear that robust training methods achieve better natural accuracy than the one obtained with Nominal training.

\begin{figure}[H]
    \centering
    \begin{subfigure}[b]{0.49\textwidth}
        \centering
        \includegraphics[width=\textwidth]{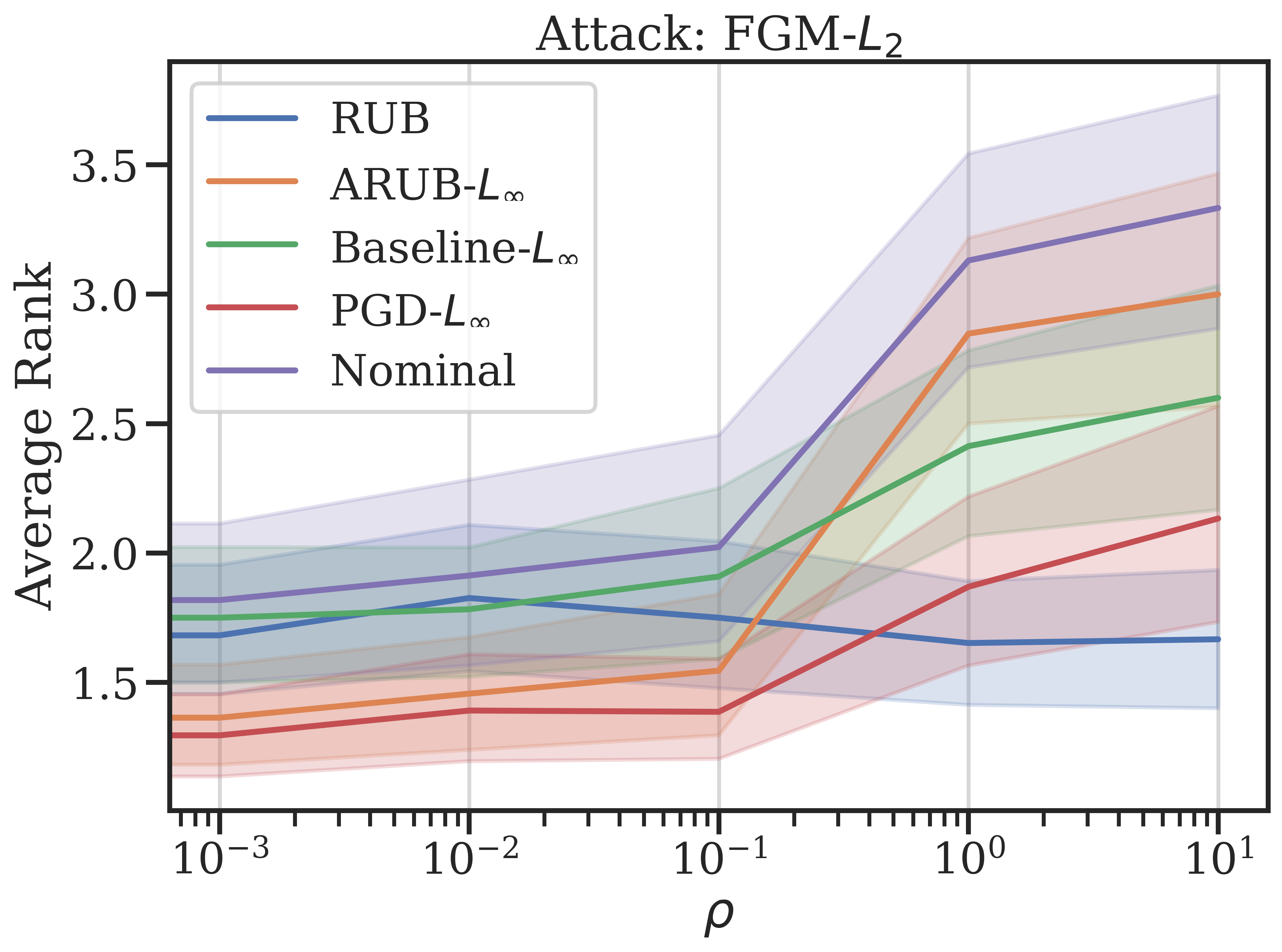}
        \vspace{-2\baselineskip}
        \caption{}
        \label{fig:FGM_l2_ranks}
    \end{subfigure}
    \hfill
    \begin{subfigure}[b]{0.49\textwidth}
        \centering
        \includegraphics[width=\textwidth]{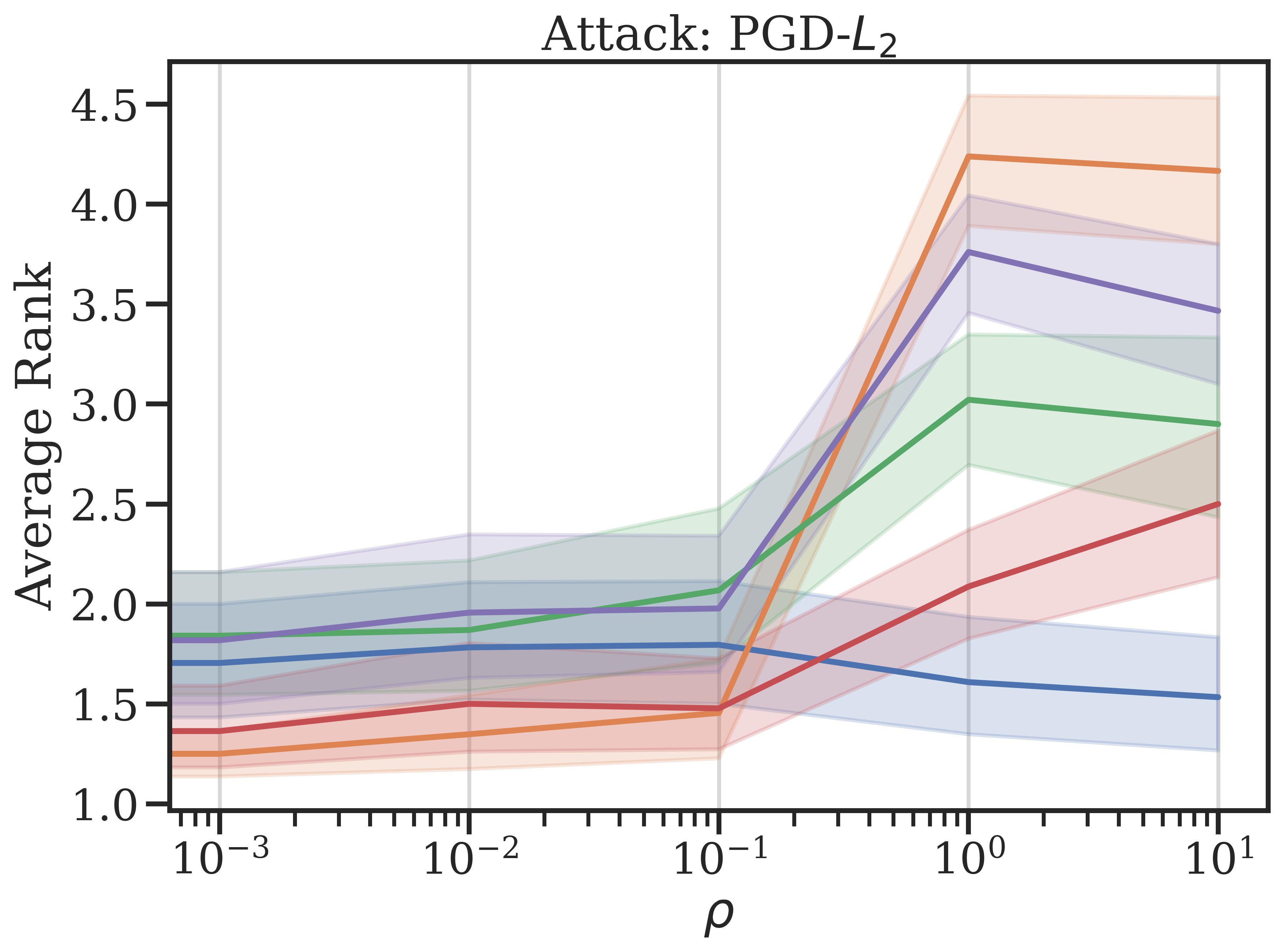}
        \vspace{-2\baselineskip}
        \caption{}
        \label{fig:PGD_l2_ranks}
    \end{subfigure}\\
    \hfill \\
    \begin{subfigure}[b]{0.49\textwidth}
        \centering
        \includegraphics[width=\textwidth]{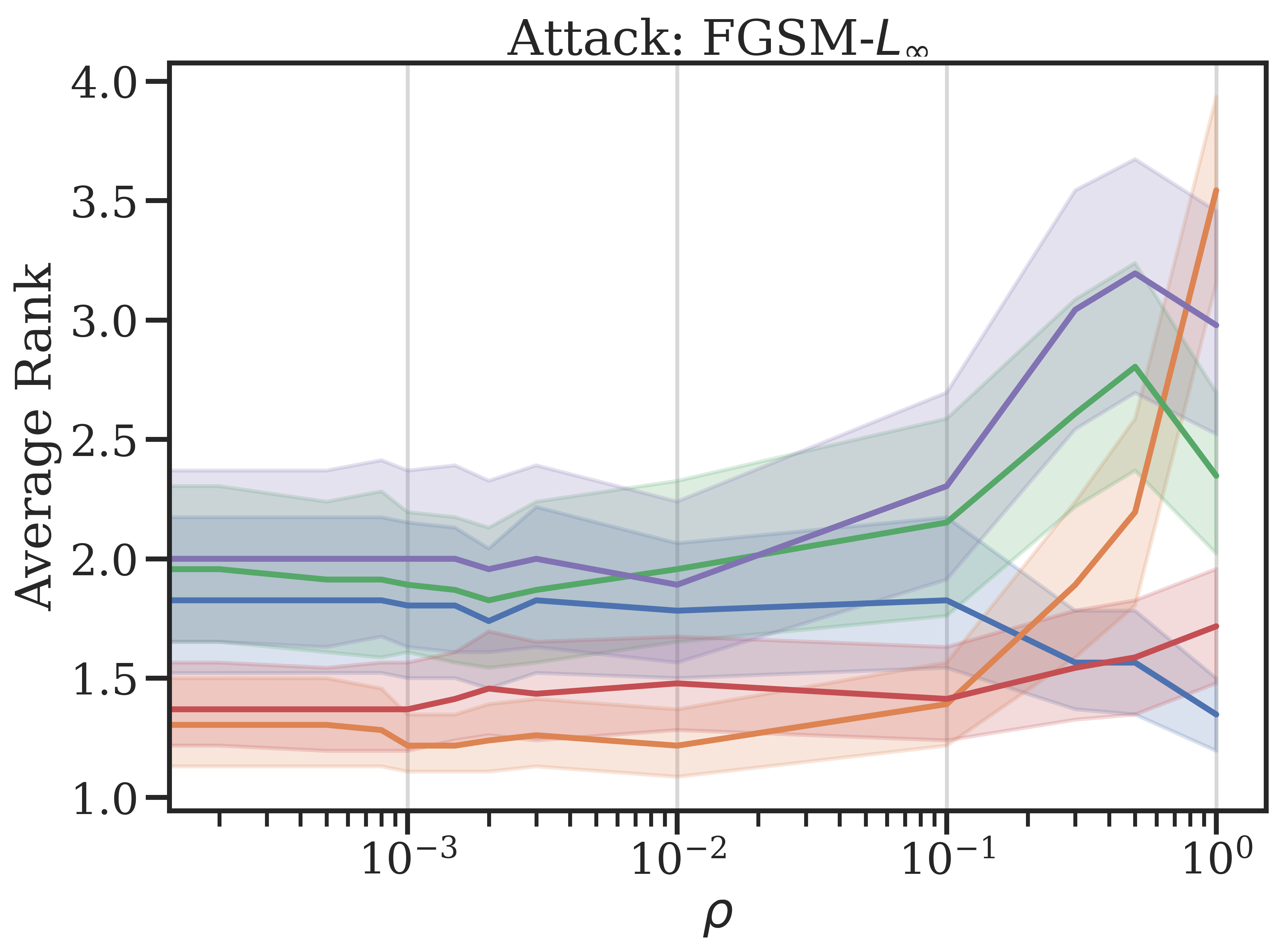}
        \vspace{-2\baselineskip}
        \caption{}
        \label{fig:FGSM_linf_ranks}
    \end{subfigure}
    \hfill
    \begin{subfigure}[b]{0.49\textwidth}
        \centering
        \includegraphics[width=\textwidth]{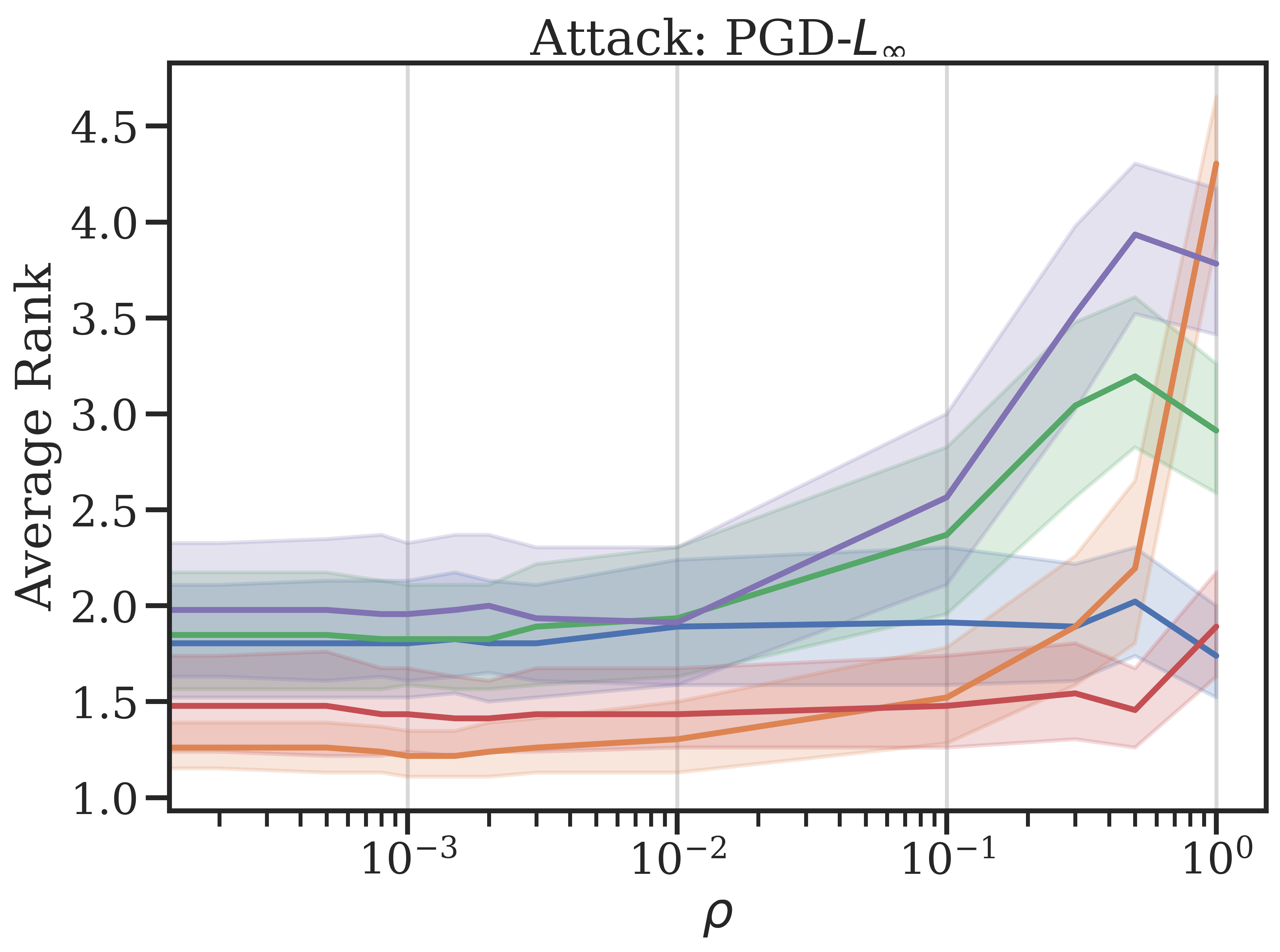}
        \vspace{-2\baselineskip}
        \caption{}
        \label{fig:PGD_linf_ranks}
    \end{subfigure}\\
    \hfill \\
    \begin{subfigure}[b]{0.49\textwidth}
        \centering
        \includegraphics[width=\textwidth]{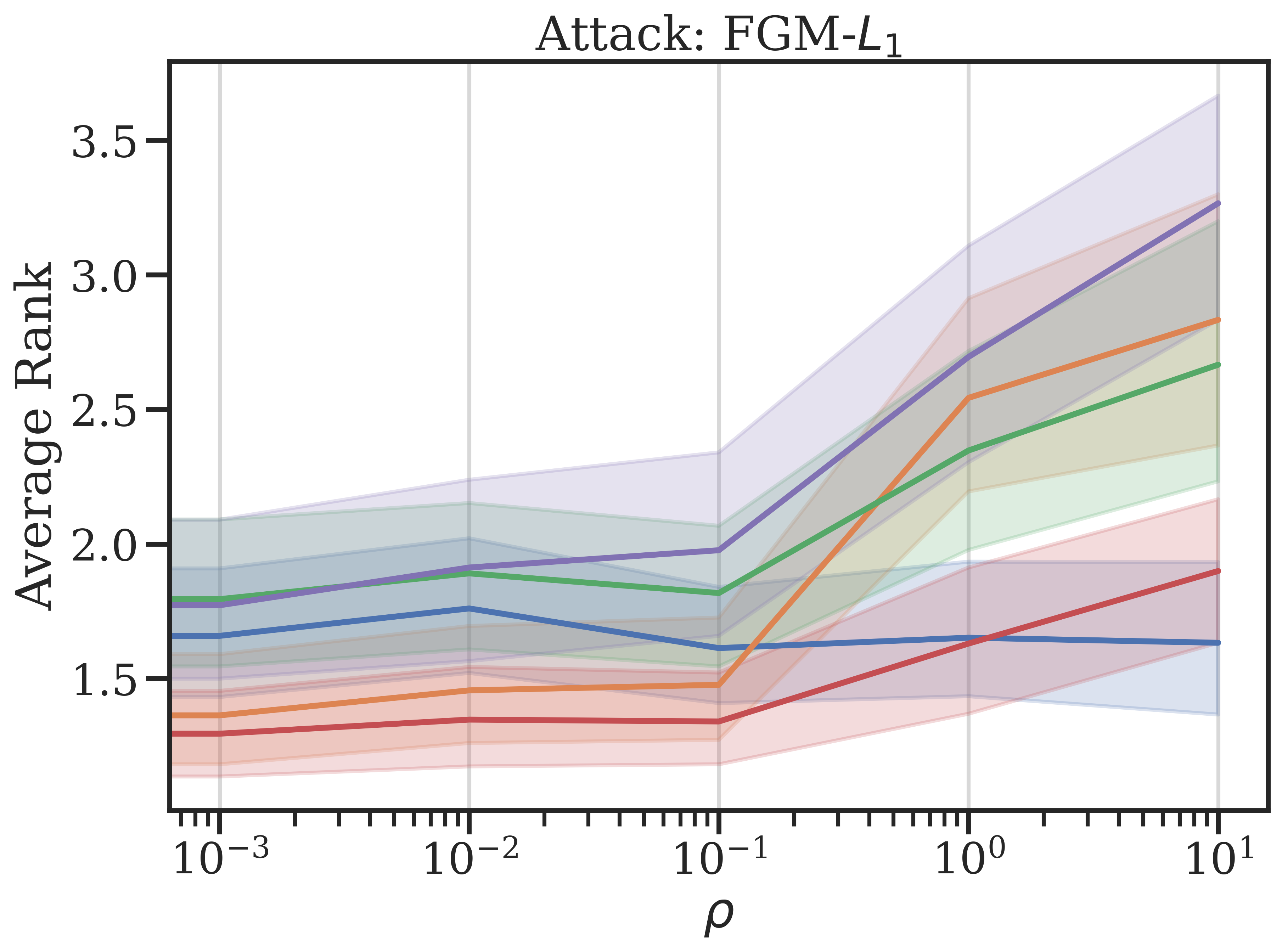}
        \vspace{-2\baselineskip}
        \caption{}
        \label{fig:FGSM_l1_ranks}
    \end{subfigure}
    \hfill
    \begin{subfigure}[b]{0.49\textwidth}
        \centering
        \includegraphics[width=\textwidth]{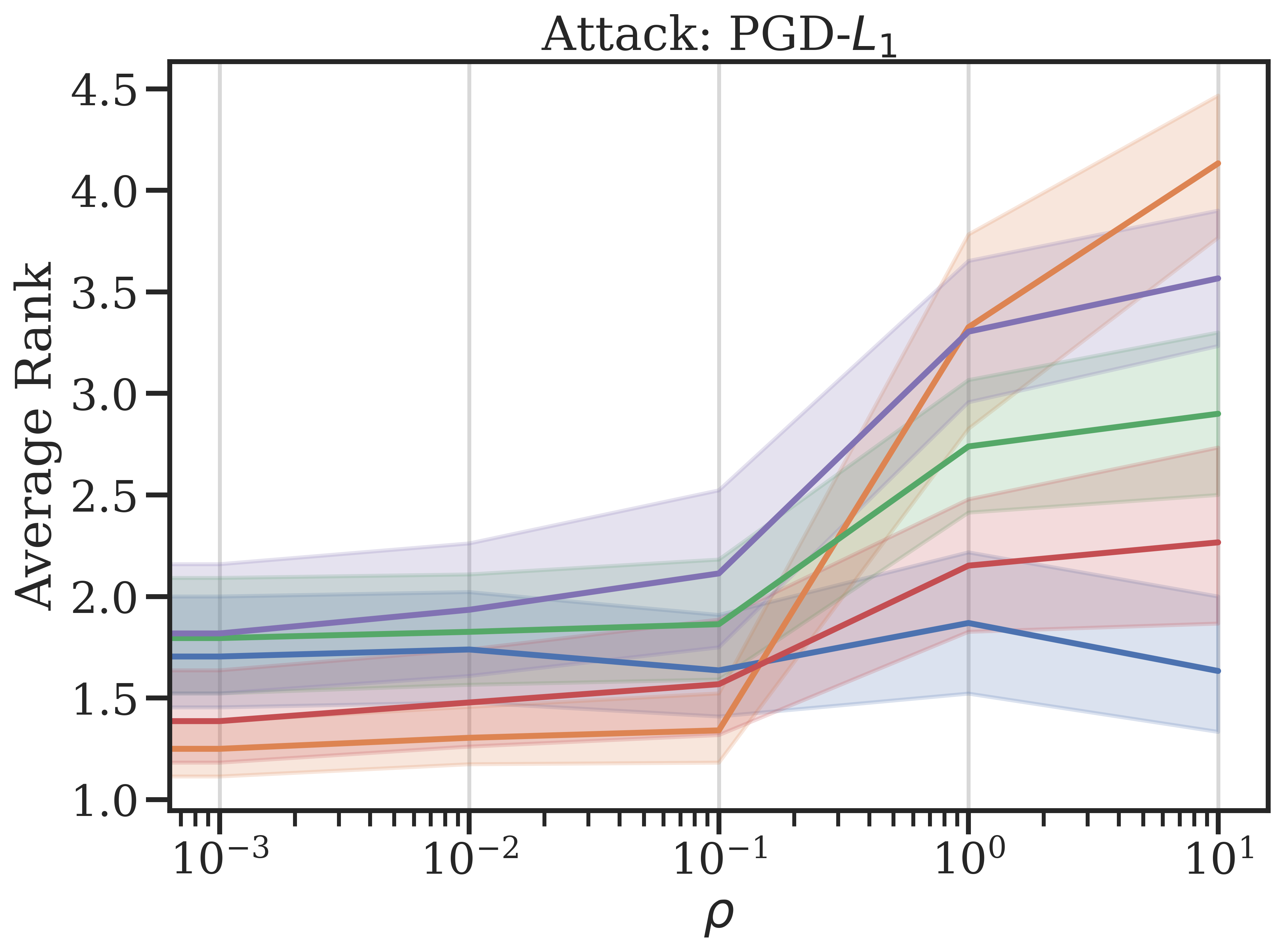}
        \vspace{-2\baselineskip}
        \caption{}
        \label{fig:PGD_l1_ranks}
    \end{subfigure}
       \caption{Average rank \textcolor{black}{ and the corresponding $95\%$ confidence interval for} each method across the 46 UCI data sets for adversarial attacks bounded in $L_2, L_{\infty}$ and $L_1$ norm, respectively from top to bottom. The figures on the left use attacks based on Fast Gradient methods, and the figures on the right use instead attacks based on Projected Gradient Descent.}
       \label{fig:ranks}
\end{figure}
\newpage

\begin{figure}[h!]
    \centering
    \begin{subfigure}[b]{0.32\textwidth}
        \centering
        \includegraphics[width=\textwidth]{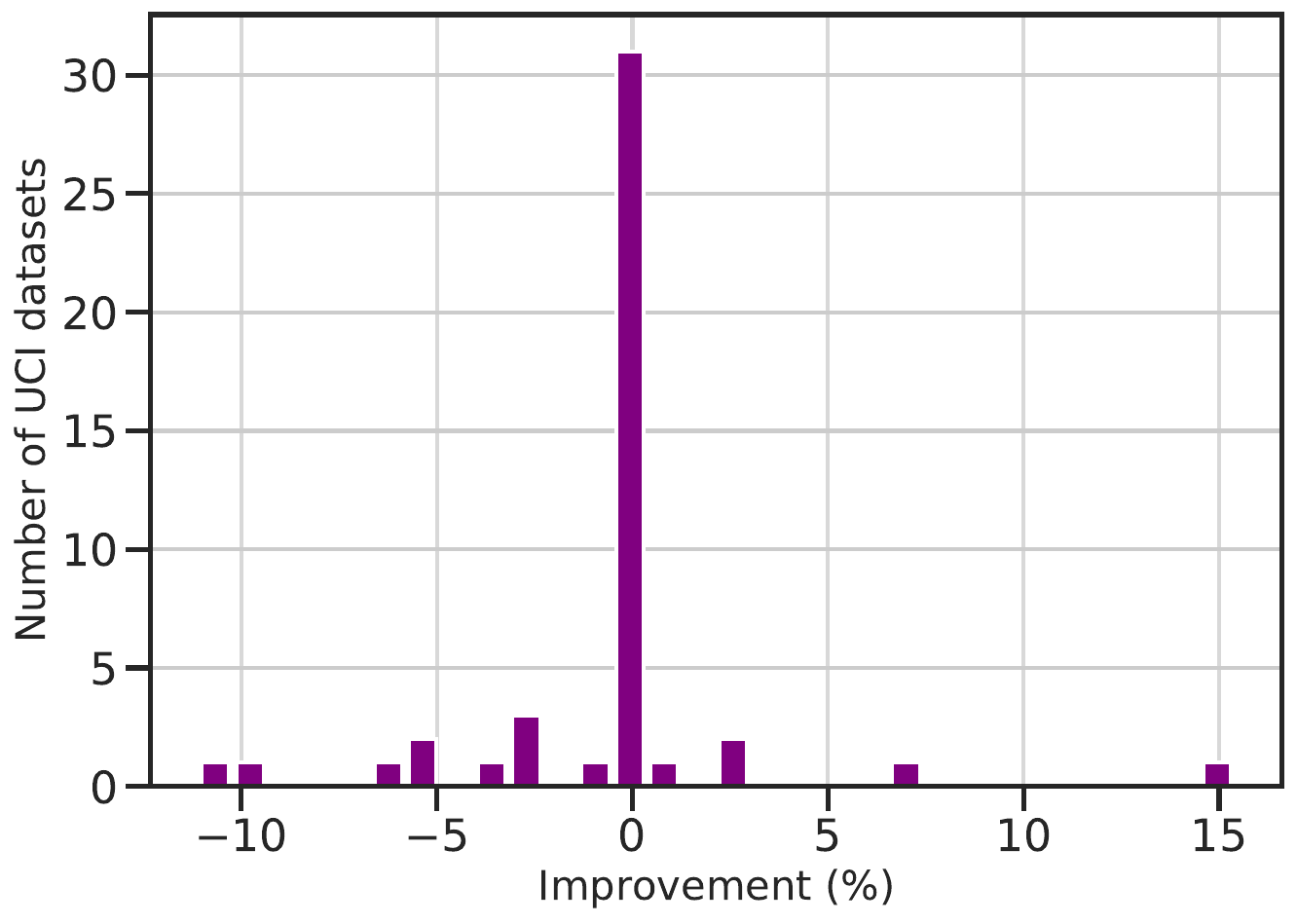}
        \caption{$\rho = 0.01$}
        \label{fig:UCI_RUB_pct_improved_01}
    \end{subfigure}
    \begin{subfigure}[b]{0.32\textwidth}
        \centering
        \includegraphics[width=\textwidth]{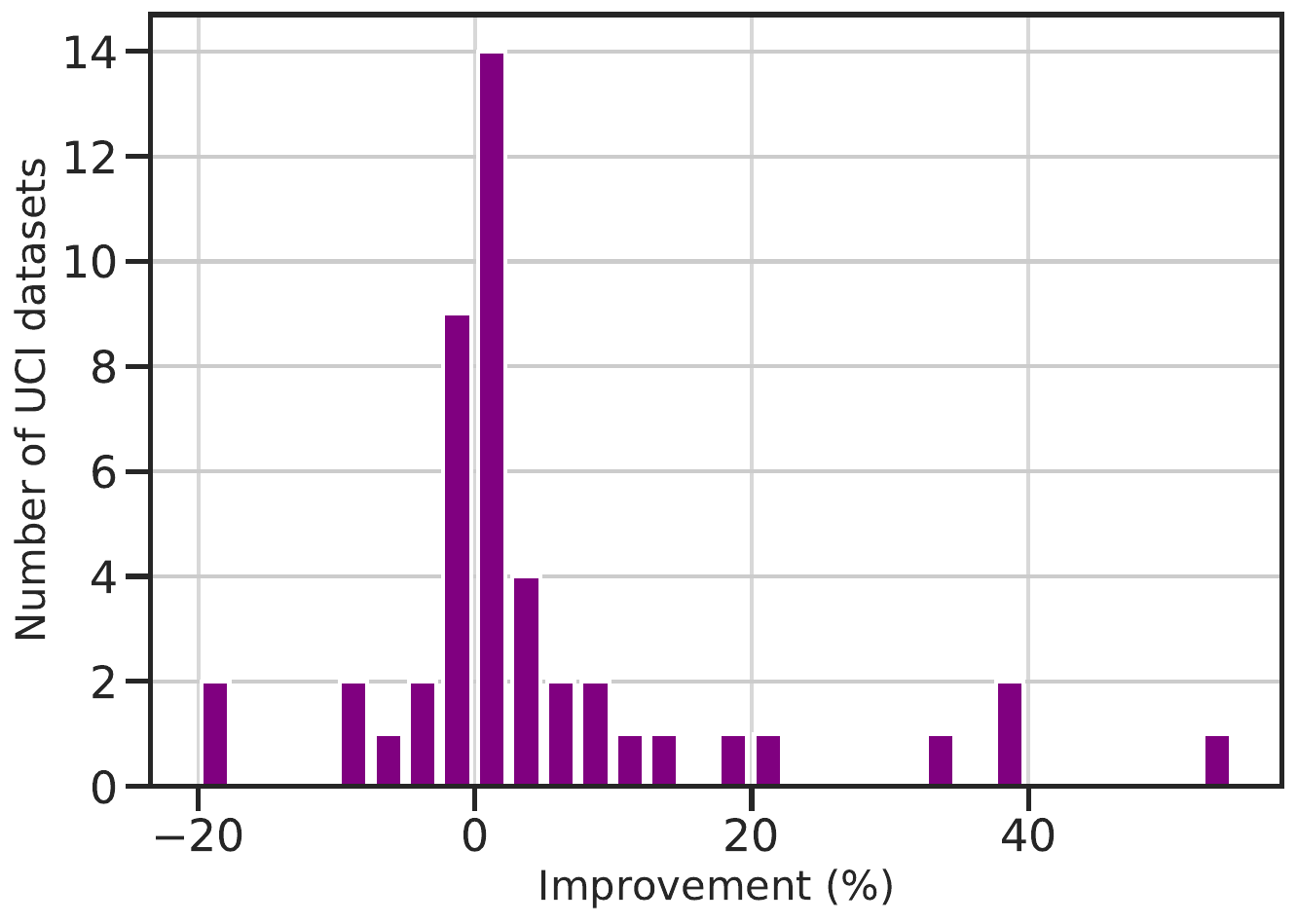}
        \caption{$\rho = 1$}
        \label{fig:UCI_RUB_pct_improved_1}
    \end{subfigure}
    \begin{subfigure}[b]{0.32\textwidth}
        \centering
        \includegraphics[width=\textwidth]{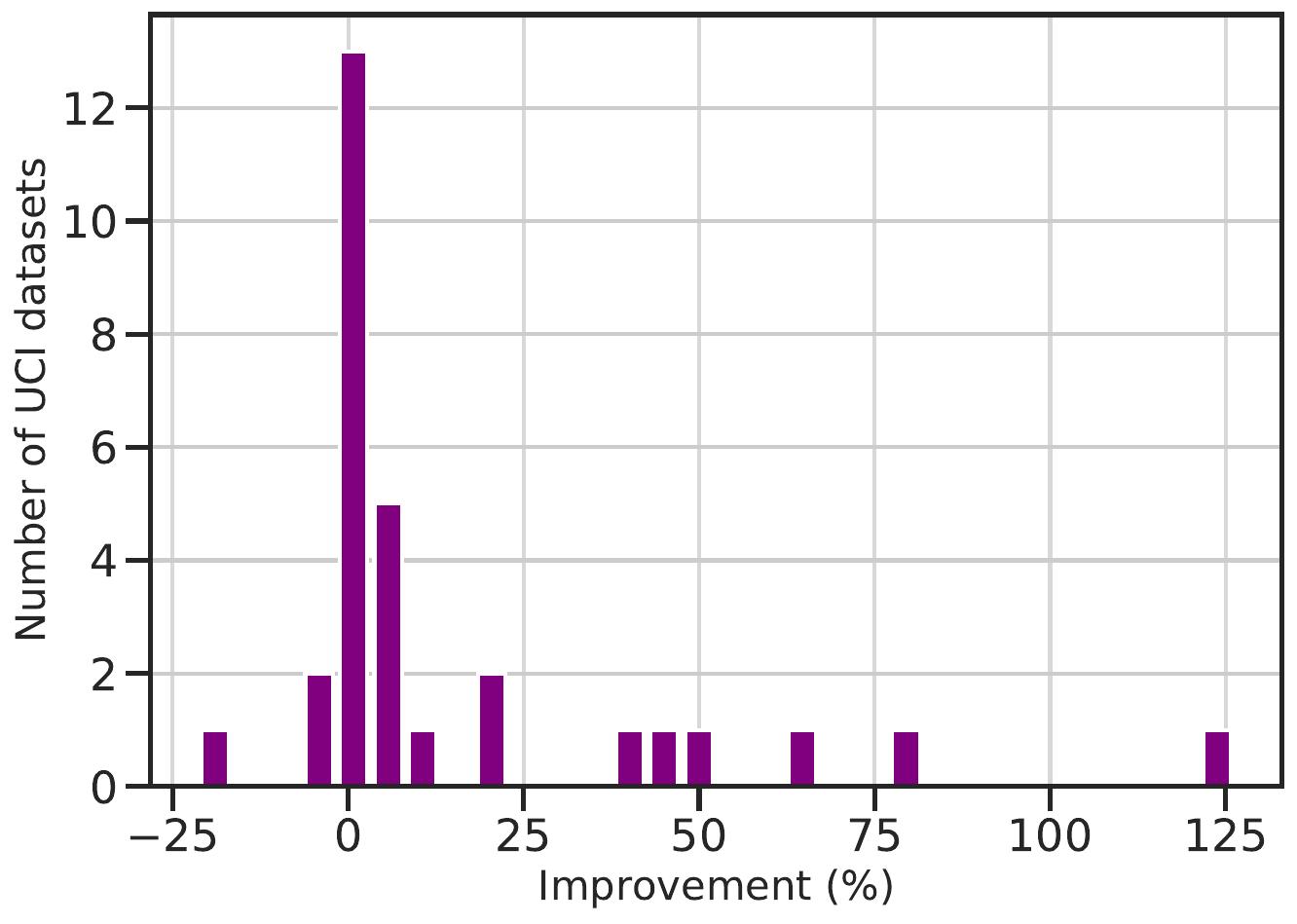}
        \caption{$\rho = 10$}
        \label{fig:UCI_RUB_pct_improved_10}
    \end{subfigure}
       \caption{Number of UCI data sets for which RUB-$L_1$ improves adversarial accuracy over PGD-$L_\infty$ by a specific percentage. Figures a), b) and c) show the corresponding plots for PGD-$L_2$ adversarial examples with different values of $\rho$.}
       \label{fig:UCI_RUB_pct_improved_}
\end{figure}
\begin{figure}[h!]
    \centering
    \begin{subfigure}[b]{0.32\textwidth}
        \centering
        \includegraphics[width=\textwidth]{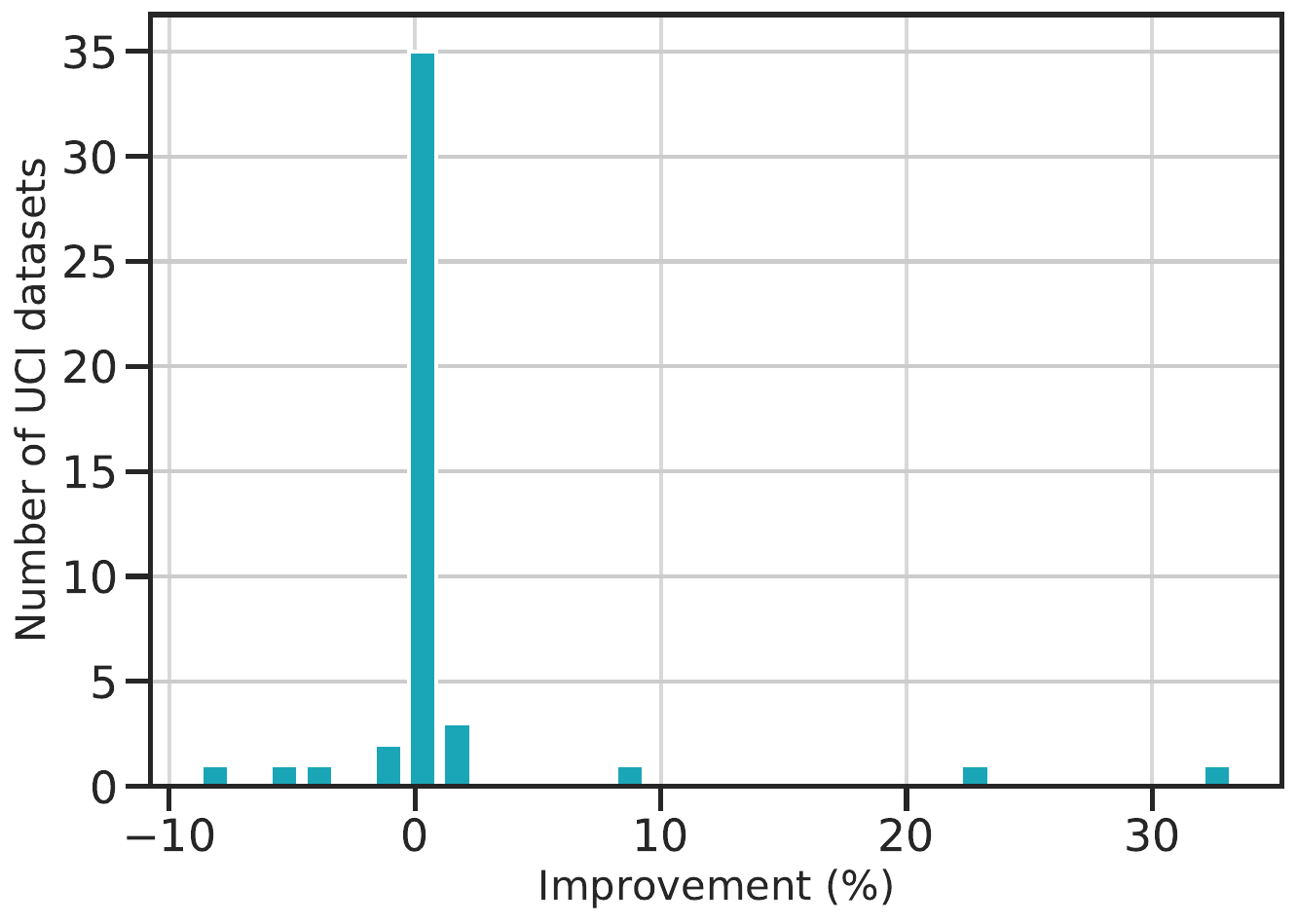}
        \caption{$\rho = 0.001$}
        \label{fig:UCI_aRUB_pct_improved_0001}
    \end{subfigure}
    \begin{subfigure}[b]{0.32\textwidth}
        \centering
        \includegraphics[width=\textwidth]{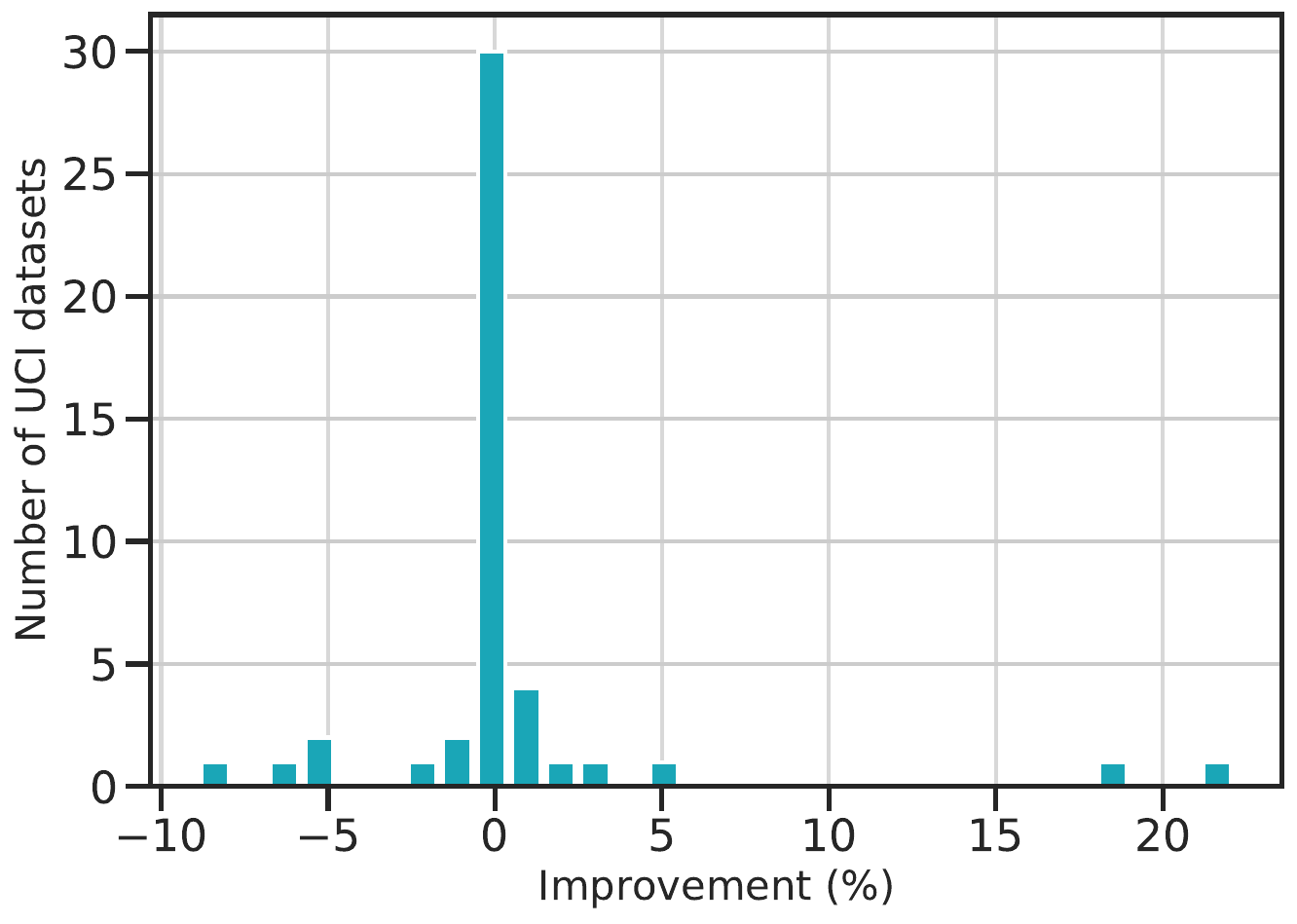}
        \caption{$\rho = 0.1$}
        \label{fig:UCI_aRUB_pct_improved_01}
    \end{subfigure}
    \begin{subfigure}[b]{0.32\textwidth}
        \centering
        \includegraphics[width=\textwidth]{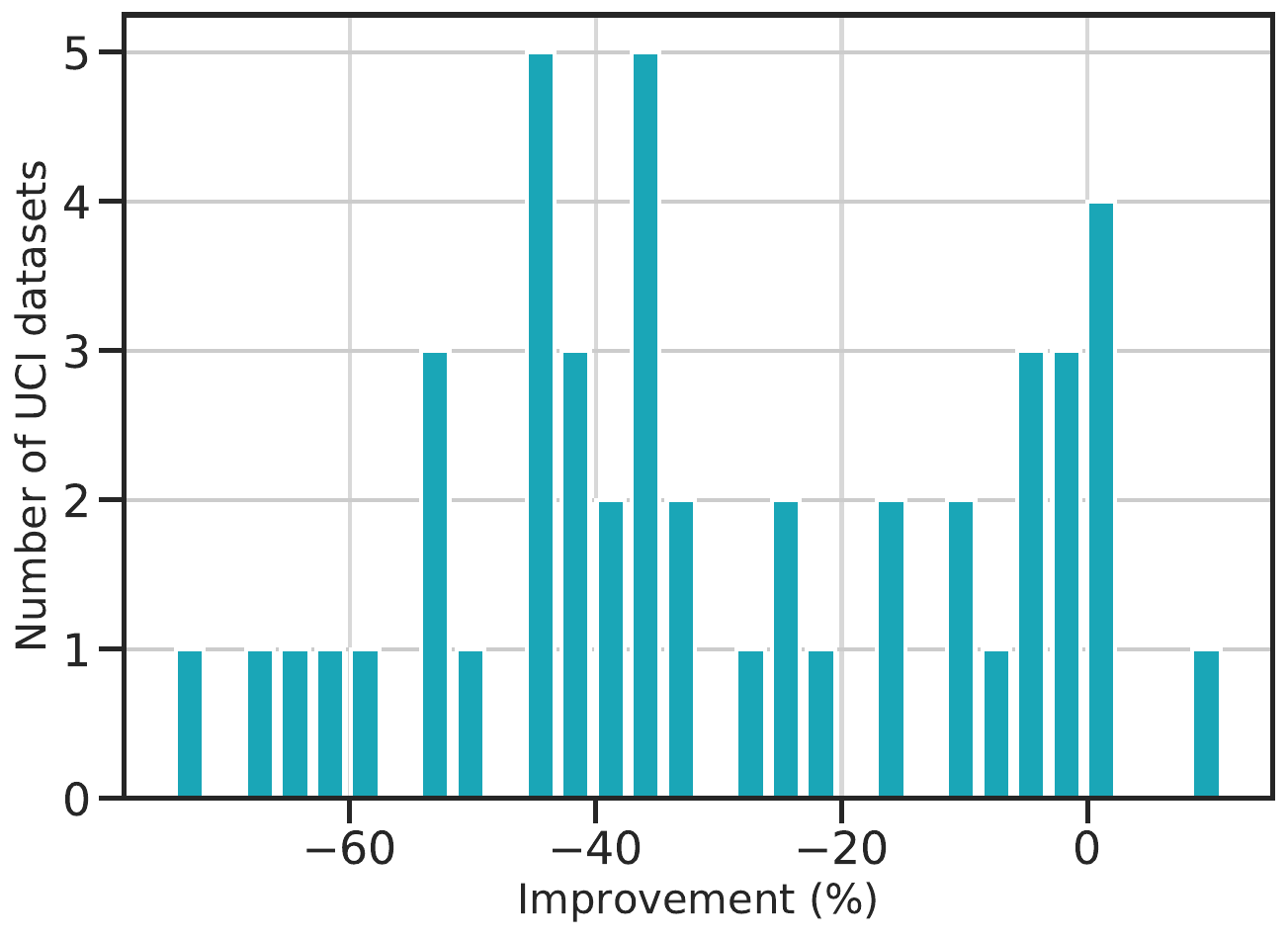}
        \caption{$\rho = 1$}
        \label{fig:UCI_aRUB_pct_improved_1}
    \end{subfigure}
       \caption{Number of UCI data sets for which aRUB-$L_\infty$ improves adversarial accuracy over PGD-$L_\infty$ by a specific percentage. Figures a), b) and c) show the corresponding plots for PGD-$L_\infty$ adversarial examples with different values of $\rho$.}
       \label{fig:UCI_aRUB_pct_improved_}
\end{figure}
\begin{table}[h!]
   \centering
   \begin{tabular}{|c|c|c|}
   \hline
         & \textbf{Avg no. \textcolor{black}{batches} per second} & \textbf{Standard Deviation}  \\
        \hline
        \hline
        RUB& \textcolor{black}{$65.1$}	& \textcolor{black}{$12.6$}\\
        \hline
        \textcolor{black}{aRUB -$L_1$} & \textcolor{black}{$17.5$}	 & \textcolor{black}{$0.5$}\\
        \hline
        aRUB -$L_\infty$ & \textcolor{black}{$18.8$	} & \textcolor{black}{$0.4$}\\
        \hline
        Baseline-$L_\infty$ &\textcolor{black}{ $465$}	 & \textcolor{black}{$56.4$}\\
        \hline
        PGD -$L_\infty$ & \textcolor{black}{$4.5$	} & \textcolor{black}{$0.1$}\\
        \hline
        Nominal & \textcolor{black}{$712.6$}	 & \textcolor{black}{$87.8$}\\
        \hline
   \end{tabular}\caption{Average number of \textcolor{black}{batches} processed per second across the 46 UCI data sets, as well as the corresponding standard deviations. }\label{Table:UCI_time}
\end{table}

To better compare the performances of RUB \textcolor{black}{ and aRUB against} PGD-$L_\infty$, \textcolor{black}{we also analyze the percentage by which each method improves adversarial accuracy across data sets}. In Figure \ref{fig:UCI_RUB_pct_improved_} we show the number of data sets for which RUB improves $L_2$ adversarial accuracy over PGD-$L_\infty$ by a specific percentage. We observe that the improvement becomes \textcolor{black}{larger} as $\rho$ increases, and in particular, for $\rho=10$ we observe that RUB only lowers adversarial accuracy for $3$ data sets, while it shows more than $15\%$ improvement over PGD-$L_\infty$ for $8$ of the data sets. \textcolor{black}{Similarly, Figure \ref{fig:UCI_aRUB_pct_improved_} displays the number of data sets for which aRUB-$L_\infty$ improves adversarial accuracy over PGD-$L_\infty$ by some percentage; we observe that for small perturbations ($\rho = 0.001, \rho =0.1$) aRUB seems to slightly improve over PGD-$L_\infty$, while this last defense has a clear advantage for larger perturbations ($\rho = 1$).}

Lastly, in Table \ref{Table:UCI_time} we display the {average number of \textcolor{black}{batches} processed per second as well as the corresponding standard deviation for each method across the 46 UCI data sets}. As expected, we see that Nominal is the method that processes the largest number of \textcolor{black}{batches} per second, and all defense methods except Baseline-$L_\infty$ are much slower than Nominal. \textcolor{black}{However, RUB, RUB-$L_1$ and aRUB-$L_{\infty}$ all process more batches per second compared to PGD-$L_\infty$.}

\begin{table}[t!]
    \centering
    \begin{tabular}{llllllllllllll}
    \toprule
    \toprule
    \multirow{9}*{\rotatebox{90}{$L_2$ Attacks}}
    &\hfill $\boldsymbol{\rho}=$ &         0.00  &           0.01  &           0.06  &           0.28  &           2.80  &           8.40  &           14.00 &           28.00 & 56.00 \\   
    \cline{2-11}
    &RUB&           \cellcolor{green!20}{89.88} &           89.22 &           \cellcolor{green!20}{88.95} &   \cellcolor{green!20}{\textbf{86.80}} &           60.51 &   \cellcolor{green!20}{\textbf{55.47}} &   \cellcolor{green!20}{\textbf{55.47}} &   \cellcolor{green!20}{\textbf{55.27}} &   \cellcolor{green!20}{\textbf{49.41}} \\
    &aRUB-$L_1$          &  \cellcolor{green!20}{\textbf{90.31}} &  \cellcolor{green!20}{\textbf{90.04}} &  \cellcolor{green!20}{\textbf{89.45}} &           85.98 &           63.98 &           29.10 &           18.24 &           15.04 &            9.96\\
    &aRUB-$L_\infty$ &           89.18 &           89.06 &           88.75 &           85.98 &           65.43 &           24.18 &           18.24 &           16.02 &            9.96\\
     \cdashline{2-11}
    &Baseline-$L_\infty$         &           89.38 &           89.02 &           88.32 &           85.04 &           48.20 &           29.88 &           28.71 &           26.52 &           18.55 \\
    &PGD-$L_\infty$      &           89.61 &           89.38 &           88.01 &           85.23 &   \cellcolor{green!20}{\textbf{68.20}} &           31.60 &           25.90 &           22.85 &           19.10 \\
    &\textcolor{black}{FGSM-$L_\infty$}      &          {89.41} &           {87.81} &           {87.19} &           {84.06} &   \cellcolor{green!20}{{67.81}} &           {35.23} &           {27.11} &          {25.39} &           {19.06} \\
    &\textcolor{black}{COAP-$L_\infty$ }      &           89.14 &           88.48 &           87.11 &           82.58 &           30.51 &           20.94 &           19.69 &           19.69 &           19.69\\
    &Nominal         &           87.70 &           88.40 &           88.01 &           85.35 &           46.60 &           19.92 &           16.84 &           15.39 &           15.39\\
    \midrule
    \midrule
    \multirow{9}*{\rotatebox{90}{$L_1$ Attacks}}
    & \textbf{\hfill $\boldsymbol{\rho}=$} &  {0.00}  &          { 0.01 } &         {  0.06}  &         {  0.28}  &          { 2.80 } &        {   8.40}  &          { 14.00} &        {   28.00 } & 140.00\\
    \cline{2-11}
    &RUB  &  \cellcolor{green!20}{\textbf{89.84}} &  \cellcolor{green!20}{\textbf{89.80}} &  \cellcolor{green!20}{\textbf{89.73}} &  \cellcolor{green!20}{\textbf{89.41}} &  \cellcolor{green!20}{\textbf{87.93}} &           \cellcolor{green!20}{84.18} &           81.72 &           \cellcolor{green!20}{75.70} & \cellcolor{green!20}{\textbf{55.51}} \\
    &aRUB-$L_1$   &           89.02 &           89.02 &           88.98 &           88.71 &           87.34 &  \cellcolor{green!20}{\textbf{84.57}} &  \cellcolor{green!20}{\textbf{82.34}} &  \cellcolor{green!20}{\textbf{75.86}} & 41.21		 \\
    &aRUB-$L_\infty$   &           89.02 &           89.02 &           89.02 &           88.79 &           87.03 &           82.85 &           79.77 &           71.25 & 30.63\\
     \cdashline{2-11}
     &Baseline-$L_\infty$ &           88.83 &           88.20 &           88.20 &           87.85 &           85.74 &           82.11 &           78.09 &           68.28 & 30.59	\\
    &PGD-$L_\infty$         &           89.02 &           89.02 &           88.95 & \cellcolor{green!20}{88.95} &           87.11 &  \cellcolor{green!20}{\textbf{84.57} }&           81.37 &           73.16 & 39.69\\
    &\textcolor{black}{FGSM-$L_\infty$}      &          \cellcolor{green!20}{{89.80}} &           \cellcolor{green!20}{{89.77}} &           \cellcolor{green!20}{\textbf{{89.73}}} &           {89.26} &  {85.00} &          {82.19} &          {78.44} &          {73.36} &          {35.62} \\
    &COAP-$L_\infty$ &           89.22 &           89.22 &           89.22 &           88.67 &           82.58 &           73.01 &           60.47 &           48.59 & 22.19	 \\
    &Nominal         &           86.99 &           86.95 &           86.91 &           86.91 &           85.00 &           82.54 &           78.83 &           68.98 & 21.88 \\
    \midrule
    \midrule
    \multirow{9}*{\rotatebox{90}{$L_\infty$ Attacks}}
    & \textbf{\hfill $\boldsymbol{\rho}=$} &  {0.000}  &          { 0.001 } &         {  0.003}  &         {  0.010}  &          { 0.100 } &        {   0.300}  &          { 0.500} &        {   1.00 } & 5.00\\
    \cline{2-11}
    &RUB &              89.38 & {89.18} &           88.52 &           86.45 &           65.08 &           56.17 &           56.13 &  \cellcolor{green!20}{\textbf{55.78}} & \cellcolor{green!20}{\textbf{37.77}} \\
    &aRUB-$L_1$  &  \cellcolor{green!20}{\textbf{90.00}} &            \cellcolor{green!20}{\textbf{89.73}} &  \cellcolor{green!20}{\textbf{89.06}} &           86.05 &           65.82 &           28.95 &           11.02 &           10.00 & 10.00 \\
    &aRUB-$L_\infty$  &           \cellcolor{green!20}{89.77} &           89.22 &           \cellcolor{green!20}{88.71} &           86.68 &           79.80 &           67.93 &           52.81 &           12.15 & 10.00 \\
    \cdashline{2-11}
    &Baseline-$L_\infty$ &           89.14 &           88.63 &           86.99 &           85.55 &          56.80 &           30.20 &           29.06 &           27.19 & 19.06	 \\
    &PGD-$L_\infty$           &           89.02 &           89.02 &           88.24 &  \cellcolor{green!20}{\textbf{87.77}} &  \cellcolor{green!20}{80.43} &  \cellcolor{green!20}{\textbf{73.09}} &  \cellcolor{green!20}{\textbf{65.74}} &           31.21 & 18.52	\\
    &\textcolor{black}{FGSM-$L_\infty$}      &           \cellcolor{green!20}{{89.61}} &           {89.10} &           {87.97} &           {86.91} &  \cellcolor{green!20}{\textbf{{80.55}}} &           {62.93} &         {53.16} &          {25.74} &           {18.12} \\
    &COAP-$L_\infty$ &           89.18 &           88.24 &           86.76 &           85.66 &           67.15 &           23.87 &           17.19 &           17.19 & 16.21 \\
    &Nominal         &           88.12 &           87.93 &           87.03 &           85.27 &           58.32 &           21.80 &           18.79 &           16.37 & 10.94	 \\
    \bottomrule
    \bottomrule
    \end{tabular}\caption{\textcolor{black}{Adversarial Accuracy ($\%$) for Fashion MNIST. For each choice of $L_p$ norm, defense method and noise radius $\rho$, we report the minimum accuracy achieved with PGD and FGM attacks. Colored cells correspond to accuracies that are within $0.5$ percentage units of the best result in each column, which has bold font.}}\label{Table:FASHION_3L}
    \end{table}

\subsection{Vision Data Sets}
\label{sec:vision-attacks}
We next show experiment results for the three vision data sets. Specifically, we compare for different training methods their performance against adversarial attacks as well as the security guarantees obtained from applying the upper bound from Eq. \eqref{eq:upper-bound}. 
Since the proposed RUB method significantly increases memory requirements for inputs with large dimensions, we compare this method against other defenses using the feed forward architecture with three hidden layers. We do include results using a CNN architecture for all other methods in Appendix \ref{sec:appendixD}, where we also explain how to extend the theory of the RUB method for convolutional layers with ReLU and MaxPool activation functions.

\begin{table}[t]
    \centering
    \begin{tabular}{llllllllllllll}
    \toprule
    \toprule
    \multirow{9}*{\rotatebox{90}{$L_2$ Attacks}}
    &\hfill $\boldsymbol{\rho}=$ &         0.00  &           0.01  &           0.06  &           0.28  &           2.80  &           8.40  &           14.00 &           28.00 & 56.00 \\                  
    \cline{2-11}
    &RUB&           97.93 &           97.85 &           97.07 &           97.07 &           79.30 &  \cellcolor{green!20}{\textbf{66.48}} &  \cellcolor{green!20}{62.89} &  \cellcolor{green!20}{\textbf{55.00}} &  \cellcolor{green!20}{\textbf{38.20}}  \\
    &aRUB-$L_1$          &  \cellcolor{green!20}{\textbf{98.63}} &  \cellcolor{green!20}{\textbf{98.52}} &           97.73 &           97.46 &           80.82 &           57.58 &           55.04 &           46.60 &           29.96 \\
    &aRUB-$L_\infty$     &           97.97 &           97.97 &           \cellcolor{green!20}{97.81} &  \cellcolor{green!20}{97.81} &           91.17 &           51.80 &           49.53 &           43.40 &           31.13 \\
     \cdashline{2-11}
    &Baseline-$L_\infty$      &           97.58 &           97.50 &           97.50 &           97.11 &           72.85 &           64.18 &           61.29 &           52.81 &           34.80 \\
    &PGD-$L_\infty$    &           \cellcolor{green!20}{98.24} &           \cellcolor{green!20}{98.24} &  \cellcolor{green!20}{\textbf{98.24}} &           \cellcolor{green!20}{\textbf{98.09}} &  \cellcolor{green!20}{{92.70}} &           \cellcolor{green!20}{66.05} &           {62.46} &           54.26 &           35.94 \\
    &\textcolor{black}{FGSM-$L_\infty$}      &     {97.81} & {97.81} & {97.50} & {97.50} & \cellcolor{green!20}{\textbf{{93.01}}} & {64.96} & \cellcolor{green!20}{\textbf{{63.20}}} &{53.40} & {36.52}\\
    &\textcolor{black}{COAP-$L_\infty$ }     &           98.01 &           97.97 &           97.54 &           94.88 &           54.73 &           31.29 &           31.29 &           28.98 &           28.48\\
    &Nominal         &           97.62 &           97.46 &           97.34 &           96.41 &           69.61 &           19.57 &           15.90 &           15.35 &           11.25\\
    \midrule
    \midrule
    \multirow{9}*{\rotatebox{90}{$L_1$ Attacks}}
    & \textbf{\hfill $\boldsymbol{\rho}=$} &  {0.00}  &          { 0.01 } &         {  0.06}  &         {  0.28}  &          { 2.80 } &        {   8.40}  &          { 14.00} &        {   28.00 } & 140.00\\
    \cline{2-11}
    &RUB& \cellcolor{green!20}{98.01} &          \cellcolor{green!20}{97.97} &          97.93 &          97.89 &          97.03 &          97.03 &          96.25 &          91.84 & \cellcolor{green!20}{\textbf{66.95}} \\
    &aRUB-$L_1$          &         \cellcolor{green!20}{98.28} &          \cellcolor{green!20}{98.28} &          \cellcolor{green!20}{98.28} &          \cellcolor{green!20}{98.28} &         \cellcolor{green!20}{ 97.89} &          \cellcolor{green!20}{97.46} &          \cellcolor{green!20}{96.52} & \cellcolor{green!20}{94.49} &          58.20 \\
    &aRUB-$L_\infty$     & \cellcolor{green!20}{98.40} &          \cellcolor{green!20}{98.40} &          \cellcolor{green!20}{98.40} &          \cellcolor{green!20}{98.24} & \cellcolor{green!20}{\textbf{98.24}} & \cellcolor{green!20}{\textbf{97.58}} & \cellcolor{green!20}{\textbf{96.91}} &          93.16 &          51.05 \\
     \cdashline{2-11}
    &Baseline-$L_\infty$      &        \cellcolor{green!20}{98.05} &          \cellcolor{green!20}{98.05} &          \cellcolor{green!20}{98.05} &          \cellcolor{green!20}{98.01} &          97.58 &          96.52 &          95.47 &          90.16 &          63.83 \\
    &PGD-$L_\infty$    &        \cellcolor{green!20}{98.20} & \cellcolor{green!20}{98.20} & \cellcolor{green!20}{98.20} & \cellcolor{green!20}{98.20} &         \cellcolor{green!20}{98.20} &          \cellcolor{green!20}{97.50} &          \cellcolor{green!20}{96.68} &          93.87 &          \cellcolor{green!20}{66.84} \\
    &\textcolor{black}{FGSM-$L_\infty$}      &      \cellcolor{green!20}{\textbf{{98.44}}} & \cellcolor{green!20}{\textbf{{98.44}}} &\cellcolor{green!20}{\textbf{{98.44}}} & \cellcolor{green!20}{\textbf{{98.44}}} & {97.19} & \cellcolor{green!20}{{97.19}} & {95.00} & \cellcolor{green!20}{\textbf{{94.53}}} & {64.02}\\
    &\textcolor{black}{COAP-$L_\infty$ }     &          97.81 &          97.81 &          97.81 &          97.81 &          96.99 &          92.50 &          83.12 &          69.30 &          30.47 \\
    &Nominal         &            97.58 &          97.58 &          97.58 &          97.54 &          97.30 &          95.94 &          94.49 &          89.96 &          19.38 \\
    \midrule
    \midrule
    \multirow{9}*{\rotatebox{90}{$L_\infty$ Attacks}}
    & \textbf{\hfill $\boldsymbol{\rho}=$} &  {0.000}  &          { 0.001 } &         {  0.003}  &         {  0.010}  &          { 0.100 } &        {   0.300}  &          { 0.500} &        {   1.00 } & 5.00\\
    \cline{2-11}
    &RUB& \cellcolor{green!20}{98.44} &          \cellcolor{green!20}{98.32} &          97.93 &          97.70 &          86.09 &          68.52 &          67.15 &          61.56 & \cellcolor{green!20}{\textbf{27.15}} \\
    &aRUB-$L_1$       &   \cellcolor{green!20}{98.52} &          \cellcolor{green!20}{98.36} &          \cellcolor{green!20}{98.28} &          \cellcolor{green!20}{98.09} &          86.84 &          61.88 &          59.61 &          54.77 &          21.64 \\
    &aRUB-$L_\infty$     & \cellcolor{green!20}{\textbf{98.71}} & \cellcolor{green!20}{\textbf{98.59}} & \cellcolor{green!20}{98.36} & \cellcolor{green!20}{98.36} &          96.95 &          89.96 &          78.16 &          49.69 &          23.24 \\
     \cdashline{2-11}
    &Baseline-$L_\infty$      &             \cellcolor{green!20}{98.24} &          97.81 &          97.81 &          97.66 &          84.69 &          65.08 &          63.24 &          59.49 &          24.73 \\
    &PGD-$L_\infty$    &         98.20 &         \cellcolor{green!20}{98.20}  &          \cellcolor{green!20}{98.20}  &          \cellcolor{green!20}{98.20}  & {{96.99}} & \cellcolor{green!20}{\textbf{93.16}} & \cellcolor{green!20}{\textbf{86.80}} & \cellcolor{green!20}{\textbf{64.57}} &          25.35 \\
    &\textcolor{black}{FGSM-$L_\infty$}      &   \cellcolor{green!20}{{98.59}} & \cellcolor{green!20}{\textbf{{98.59}}} & \cellcolor{green!20}{\textbf{{98.59}}} & \cellcolor{green!20}{\textbf{{98.59}}} & \cellcolor{green!20}{\textbf{{97.50}}} & {91.17} & {74.10} & {63.12} & {26.36}\\
    &\textcolor{black}{COAP-$L_\infty$ }     &           \cellcolor{green!20}{98.12} &          \cellcolor{green!20}{98.12} &         \cellcolor{green!20}{98.12} &          96.56 &          90.78 &          37.81 &          31.72 &          32.62 &          24.88 \\
    &Nominal         &           97.62 &          97.62 &          97.46 &          96.80 &          81.91 &          23.28 &          16.48 &          14.61 &           9.69 \\
    \bottomrule
    \bottomrule
    \end{tabular}
    \caption{\textcolor{black}{Adversarial Accuracy ($\%$) for MNIST. For each choice of $L_p$ norm, defense method and noise radius $\rho$, we report the minimum accuracy achieved with PGD and FGM attacks. Colored cells correspond to accuracies that are within $0.5$ percentage units of the best result in each column, which has bold font.}}\label{Table:MNIST_3L}
    \end{table}

\paragraph{Performance Against Adversarial Attacks.} 
\textcolor{black}{We evaluate adversarial accuracy for all the aforementioned methods (Nominal, RUB, aRUB-$L_1$, aRUB-$L_\infty$, Baseline-$L_\infty$, PGD-$L_\infty$), and we add two other state-of-the-art defenses to make our evaluation even more comprehensive. These two methods were proposed in \cite{wong2018provable} and \cite{wong2020fast}; which we call COAP-$L_{\infty}$ (Convex Outer Adversarial Polytope
with $L_\infty$ norm) and FGSM-$L_\infty$ (Fast Gradient Sign Method) respectively, and are representative of state-of-the-art defenses in terms of robustness and training computational cost, respectively.}

\begin{table}[t]
    \centering
    \begin{tabular}{llllllllllllll}
    \toprule
    \toprule
    \multirow{9}*{\rotatebox{90}{$L_2$ Attacks}}
    &\hfill $\boldsymbol{\rho}=$ &        0.00 &           0.01 &           0.03 &           0.06 &           0.08 &           0.11 &           0.17 &           0.55 &           5.54 \\                  
    \cline{2-11}
    &RUB&49.69 &          48.67 &          46.33 &          48.59 &          48.36 &          48.12 &          47.58 &          44.14 &          17.73 \\
    &aRUB-$L_1$  &        52.42 &          51.41 &          51.37 & \cellcolor{green!20}{\textbf{51.13}} & \cellcolor{green!20}{\textbf{51.02}} & \cellcolor{green!20}{\textbf{50.74}} & \cellcolor{green!20}{\textbf{50.20}} & \cellcolor{green!20}{\textbf{46.56}} & {{21.88}} \\
    &aRUB-$L_\infty$     &         \cellcolor{green!20}{53.83} &          \cellcolor{green!20}{52.89} &          \cellcolor{green!20}{51.84} &          47.77 &          47.11 &          46.48 &          44.45 &          41.37 &          20.66 \\
     \cdashline{2-11}
    &Baseline-$L_\infty$      &          53.12 &          52.07 &          50.66 &          45.82 &          42.50 &          42.97 &          42.42 &          35.86 &           9.30 \\
    &PGD-$L_\infty$    &\cellcolor{green!20}{\textbf{53.91}} & \cellcolor{green!20}{\textbf{53.32}} & \cellcolor{green!20}{\textbf{52.27}} &          48.83 &          48.63 &          48.48 &          48.09 &          45.12 &          20.16 \\
    &\textcolor{black}{FGSM-$L_\infty$}      &  \cellcolor{green!20}{{53.52}} & {52.07} & {50.82} & {49.34} & {47.97} & {47.30} & {47.30} & {44.53} & \cellcolor{green!20}{\textbf{{24.69}}}\\
    &\textcolor{black}{COAP-$L_\infty$ }     &         51.45 &          50.86 &          49.45 &          48.59 &          47.07 &          45.35 &          41.76 &          35.86 &          11.76 \\
    &Nominal         &          46.02 &          45.78 &          44.84 &          44.10 &          42.89 &          42.15 &          40.43 &          37.11 &          10.59 \\
    \midrule
    \midrule
    \multirow{9}*{\rotatebox{90}{$L_1$ Attacks}}
    & \textbf{\hfill $\boldsymbol{\rho}=$} &   0.00 &0.01 & 0.03 & 0.06&  0.11&  0.55 &  5.54 & 16.63& 27.71\\
    \cline{2-11}
    &RUB&50.86 &          50.86 &          50.86 &          50.70 &          50.55 &          48.98 &          47.03 &          44.84 &          42.58 \\
    &aRUB-$L_1$  &          51.56 &          51.56 &          51.56 &          51.56 &          51.56 &          51.37 & \cellcolor{green!20}{\textbf{50.35}} & \cellcolor{green!20}{\textbf{47.85}} & \cellcolor{green!20}{\textbf{45.98}} \\
    &aRUB-$L_\infty$     & \cellcolor{green!20}{53.55} & \cellcolor{green!20}{53.48} & \cellcolor{green!20}{53.44} & \cellcolor{green!20}{53.40} & \cellcolor{green!20}{53.24} & {52.30} &          44.02 &          40.59 &          40.62 \\
     \cdashline{2-11}
    &Baseline-$L_\infty$      &         \cellcolor{green!20}{53.32} &          \cellcolor{green!20}{53.28} &          \cellcolor{green!20}{53.28} &          \cellcolor{green!20}{53.24} &          52.03 &          50.94 &          41.72 &          38.79 &          34.53 \\
    &PGD-$L_\infty$    &         52.97 &          52.97 &          52.97 &          52.97 &          52.93 & {52.30} &          47.89 &          45.62 &          43.52 \\
    &\textcolor{black}{FGSM-$L_\infty$}      & \cellcolor{green!20}{\textbf{{53.71}}} &  \cellcolor{green!20}{\textbf{{53.71}}} & \cellcolor{green!20}{\textbf{{53.67}}} & \cellcolor{green!20}{\textbf{{53.59}}} & \cellcolor{green!20}{\textbf{{53.55}}} & \cellcolor{green!20}{\textbf{{53.20}}} & {47.62} & {45.39} & {43.36}\\
    &\textcolor{black}{COAP-$L_\infty$ }     &          50.86 &          50.86 &          50.78 &          50.74 &          50.66 &          49.84 &          43.59 &          29.96 &          29.96 \\
    &Nominal         &         46.02 &          46.02 &          46.02 &          45.98 &          45.98 &          45.66 &          43.87 &          39.69 &          36.25 \\
    \midrule
    \midrule
    \multirow{9}*{\rotatebox{90}{$L_\infty$ Attacks}}
    & \textbf{\hfill $\boldsymbol{\rho}=$} &  {0.000}  &          { 0.001 } &         {  0.003}  &         {  0.010}  &          { 0.100 } &        {   0.300}  &          { 0.500} &        {   1.00 } & 5.00\\
    \cline{2-11}
    &RUB&50.86 &          46.02 &          47.58 &          45.08 &          21.64 &           9.38 &           9.30 &           9.30 &           8.12 \\
    &aRUB-$L_1$     &    53.28 & \cellcolor{green!20}{\textbf{50.94}} & \cellcolor{green!20}{\textbf{50.23}} & \cellcolor{green!20}{\textbf{47.07}} &          26.91 &          10.66 &          10.12 &          10.12 &          10.62 \\
    &aRUB-$L_\infty$      &          53.52 &          46.76 &          45.23 &          42.54 & {{29.06}} &          10.43 &          10.78 &          10.78 &           9.69 \\
     \cdashline{2-11}
    &Baseline-$L_\infty$      &             51.48 &          47.54 &          40.90 &          39.22 &           9.34 &           9.34 &           9.34 &           9.34 &          10.00 \\
    &PGD-$L_\infty$    & \cellcolor{green!20}{\textbf{54.10}} &          49.34 &          48.05 &          46.33 &          25.31 & \cellcolor{green!20}{\textbf{15.90}} & \cellcolor{green!20}{\textbf{12.93}} & \cellcolor{green!20}{\textbf{12.89}} & \cellcolor{green!20}{\textbf{13.16}} \\
    &\textcolor{black}{FGSM-$L_\infty$}      & \cellcolor{green!20}{{54.06}} & {50.39} & {43.12} & {40.31} & \cellcolor{green!20}{\textbf{{30.63}}} & {13.55} & {10.94} & {10.94} & {11.88} \\
    &\textcolor{black}{COAP-$L_\infty$ }     &         51.76 &          48.52 &          44.92 &          41.76 &          11.56 &          11.56 &          11.56 &          11.56 &          10.00 \\
    &Nominal         &           46.72 &          45.31 &          43.44 &          39.18 &           9.92 &           9.92 &           9.92 &           9.92 &          10.00 \\
    \bottomrule
    \bottomrule
    \end{tabular}
    
    \caption{\textcolor{black}{Adversarial Accuracy ($\%$) for CIFAR. For each choice of $L_p$ norm, defense method and noise radius $\rho$, we report the minimum accuracy achieved with PGD and FGM attacks. Colored cells correspond to accuracies that are within $0.5$ percentage units of the best result in each column, which has bold font.}}\label{Table:CIFAR_3L}
    \end{table}

\textcolor{black}{For each $L_p$ norm we report the minimum adversarial accuracy achieved using both Projected Gradient Descent attacks and Fast Gradient Method attacks. We observe that for the Fashion MNIST data set (Table \ref{Table:FASHION_3L}), aRUB-$L_1$ and RUB achieve the best accuracies for small values of $\rho$; we then observe a small range in which PGD-$L_\infty$ does best and lastly for larger values of $\rho$ we see that RUB takes the lead, which is similar to the average results observed for the UCI data sets. For the MNIST data set, all defenses achieve similar results when the input perturbations are small, with RUB showing again better performance with larger radius $\rho$. In the CIFAR data set we observe a different behavior; PGD-$L_\infty$, FGSM-$L_\infty$ and aRUB methods achieve better accuracies at various radius regimes}, although we observe that all methods perform very poorly overall as this is a notoriously more difficult data set. Lastly, we again observe that in all three data sets \textcolor{black}{Baseline-$L_\infty$ outperforms aRUB-$L_\infty$ when $\rho$ is large, and} the robust training methods achieve a higher natural accuracy than the accuracy resulting from standard nominal training.

\textcolor{black}{In Table \ref{Table:vision_time} we present the average number of batches processed per second as well as the corresponding standard deviation for each method across the 3 vision data sets. We observe that FGSM-$L_\infty$ is the fastest defense after Baseline-$L_\infty$, followed by the aRUB methods. We highlight that contrary to the results obtained with the UCI data sets, the RUB method is slower than the aRUB defenses. This is attributed to the increased memory requirements for RUB with high dimensional inputs, which prevented full parallelization during training.}

\begin{table}[t!]
    \centering
    \begin{tabu}{|c|c|c|}
    \hline
         \rowfont{\color{black}} 
         & \textbf{Avg no. \textcolor{black}{batches} per second} & \textbf{Standard Deviation}  \\
         \hline
         \hline
         \rowfont{\color{black}}
         RUB& \textcolor{black}{$4.8$	}& \textcolor{black}{$0.2$}\\
         \hline
         \textcolor{black}{aRUB -$L_1$} & \textcolor{black}{$53.1$}	 & \textcolor{black}{$2.4$}\\
         \hline
         \rowfont{\color{black}}
         aRUB -$L_\infty$ & \textcolor{black}{$56.4$}	 & \textcolor{black}{$0.2$}\\
         \hline
         \rowfont{\color{black}}
         Baseline -$L_\infty$ & \textcolor{black}{$343.5$}	 & \textcolor{black}{$33.4$}\\
         \hline
         \rowfont{\color{black}}
         PGD -$L_\infty$ & \textcolor{black}{$3.7$}	 & \textcolor{black}{$0.2$}\\
         \hline
         \rowfont{\color{black}}
         FGSM -$L_\infty$ & \textcolor{black}{$86$}	 & \textcolor{black}{$4.1$}\\
         \hline
         \rowfont{\color{black}}
         COAP -$L_\infty$ & \textcolor{black}{$12.2$}	 & \textcolor{black}{$0.3$}\\
         \hline
         \rowfont{\color{black}}
         Nominal & \textcolor{black}{$473$}	 & \textcolor{black}{$45.2$}\\
         \hline
    \end{tabu}\caption{\textcolor{black}{Average number of \textcolor{black}{batches} processed per second across the 3 vision data sets, as well as the corresponding standard deviations.}}\label{Table:vision_time}
\end{table}

\paragraph{Security Guarantees against $L_1$ Norm Bounded Attacks.}
Finally, we use the upper bound of the \textcolor{black}{adversarial loss} derived in section \textcolor{black}{\ref{sec:RUB}} to find lower bounds for the adversarial accuracy with respect to attacks bounded in the $L_1$ norm by $\rho$. \textcolor{black}{Specifically, the RUB-$L_1$ defense finds an upper bound for $\sup_{{\boldsymbol{\delta}}: \|{\boldsymbol{\delta}}\|_1 \leq \rho} \: \bz_k^L(\theta, \bx + {\boldsymbol{\delta}}) - \bz_y^L(\theta, \bx + {\boldsymbol{\delta}})$, and therefore when this upper bound is nonpositive for all $k\in[K]$ we know that the network $\bz^L(\theta, \cdot)$ correctly classifies all adversarial attacks $\bx + {\boldsymbol{\delta}}$ for which $\|{\boldsymbol{\delta}}\|_1 \leq \rho$. In other words, the nonpositivity of this upper bound gives the network a security guarantee against the attacks considered. The percentage of images in the testing set for which this guarantee exists is therefore a lower bound of the adversarial accuracy achieved by network.} In \textcolor{black}{Tables \ref{Table:FASHION_BOUND}, \ref{Table:MNIST_BOUND} and \ref{Table:CIFAR_BOUND}} we report the lower bounds for each method by selecting the hyperparameters that lead to the best lower bound in the validation set. \textcolor{black}{In particular, notice that for a given choice of radius $\rho$ and defense method, the selected network might not be the same as the one selected in the previous results for adversarial accuracy.}

We observe that, as expected, for all three data sets, CIFAR, Fashion MNIST and MNIST, the best security guarantees are the ones for the RUB method. While these results are only lower bounds for the adversarial accuracy and we cannot claim a better accuracy for RUB than for the rest of the methods, the lower bound for the RUB method shows that this method indeed performs very well against $L_1$ attacks bounded by large values of $\rho$. For instance, in Table \ref{Table:FASHION_BOUND} we can see that for the Fashion MNIST data set the RUB method guarantees $86.02\%$  adversarial accuracy (less than $5\%$ decrease from the best natural accuracy) against attacks with $L_1$ norm smaller or equal to $\rho = 2.8$. Similarly, in Table \ref{Table:MNIST_BOUND} we observe that for this same attacks RUB has at least $97.11\%$ adversarial accuracy (less than $1\%$ decrease over natural accuracy) for the MNIST data set. And finally, for the CIFAR data set, we can see in Table \ref{Table:CIFAR_BOUND} that RUB achieves $45.66\%$ adversarial accuracy (less than $5\%$ decrease over natural accuracy) against attacks whose $L_1$ norm is upper bounded by $\rho = 5.54$.

\section{Conclusions}

\begin{table}[t]
    \centering
    \begin{tabular}{llllllllll}
    \toprule
    \toprule
    \hfill $\boldsymbol{\rho}=$ &         0.00  &           0.01  &           0.06  &           0.28  &           2.80  &           8.40  &           14.00 &           28.00 \\
    \midrule
    RUB             &       \cellcolor{green!20}{90.27} &                   \cellcolor{green!20}{90.16} &                   \cellcolor{green!20}{90.16} &                   \cellcolor{green!20}{89.49} &          \cellcolor{green!20}{\textbf{86.02}} &          \cellcolor{green!20}{\textbf{80.55}} &          \cellcolor{green!20}{\textbf{76.21}} &          \cellcolor{green!20}{\textbf{66.95}} \\
    aRUB-$L_1$         &       \cellcolor{green!20}{\textbf{90.51}} &                   \cellcolor{green!20}{\textbf{90.51}} &                   \cellcolor{green!20}{\textbf{90.39}} &          \cellcolor{green!20}{\textbf{89.73}} &                   85.59 &                   73.98 &                   69.61 &                   47.54 \\
    aRUB-$L_{\infty}$       &   89.96 &                   89.92 &                   89.84 &                   88.75 &                   76.60 &                   21.37 &                   10.16 &                    9.84 \\
    \cdashline{1-9}
    Baseline-$L_\infty$            &           89.49 &           89.38 &           89.38 &           87.50 &           77.42 &           46.41 &           20.04 &           15.23 \\
    PGD-$L_{\infty}$             &        89.92 &          {89.92} &          \cellcolor{green!20}{89.92} &                   87.85 &                   78.75 &                   40.35 &                   19.22 &                   15.55\\
    Nominal         &        88.59 &                   88.55 &                   88.48 &                   87.93 &                   77.50 &                   40.98 &                   15.04 &                    9.88\\
    \bottomrule
    \bottomrule
    \end{tabular}
    \caption{Fashion MNIST: Lower bound of adversarial accuracy with uncertainty bounded in $L_1$ norm by $\rho$.}\label{Table:FASHION_BOUND}
    \end{table}
    \begin{table}[t]
    \centering
    \begin{tabular}{llllllllll}
    \toprule
    \toprule
    \hfill $\boldsymbol{\rho}=$ &            0.00  &           0.01  &           0.06  &           0.28  &           2.80  &           8.40  &           14.00 &           28.00\\
    \midrule
    RUB             &         98.01 &                   97.93 &                   97.93 &                   97.46 &          \cellcolor{green!20}{\textbf{97.11}} &          \cellcolor{green!20}{\textbf{93.67}} &          \cellcolor{green!20}{\textbf{89.77}} &          \cellcolor{green!20}{\textbf{74.96}} \\
    aRUB-$L_1$         &  \cellcolor{green!20}{\textbf{98.52}} &          \cellcolor{green!20}{\textbf{98.48}} &          \cellcolor{green!20}{\textbf{98.48}} &                   \cellcolor{green!20}{98.20} &                   96.48 &                   89.96 &                   80.86 &                   26.05 \\
    aRUB-$L_{\infty}$       &         \cellcolor{green!20}{98.40} &                   \cellcolor{green!20}{98.40} &                   \cellcolor{green!20}{98.28} &                   97.54 &                   94.69 &                   68.52 &                   16.56 &                   12.19 \\
    \cdashline{1-9}
    Baseline-$L_\infty$            &           \cellcolor{green!20}{98.05} &           \cellcolor{green!20}{98.05} &           \cellcolor{green!20}{98.05} &           97.81 &           95.23 &           57.23 &           20.74 &           10.35 \\
    PGD-$L_{\infty}$             &        \cellcolor{green!20}{98.48} &          \cellcolor{green!20}{\textbf{98.48}} &                   \cellcolor{green!20}{98.44} &         \cellcolor{green!20}{\textbf{98.36}} &                   95.55 &                   63.59 &                   15.43 &                   11.60\\
    Nominal         &      97.73 &                   97.73 &                   97.58 &                   97.54 &                   93.87 &                   44.80 &                   10.31 &                   10.23 \\
    \bottomrule
    \bottomrule
    \end{tabular}
    \caption{MNIST: Lower bound of adversarial accuracy with uncertainty bounded in $L_1$ norm by $\rho$.}\label{Table:MNIST_BOUND}
    \end{table}
    \begin{table}[h!]
    \centering
    \begin{tabular}{llllllllll}
    \toprule
    \toprule
    \hfill $\boldsymbol{\rho}=$ &        0.00  &           0.01  &           0.06  &           0.55  &           5.54  &           16.63 &           27.71 &           55.43  \\
    \midrule
    RUB             &         50.62 &           50.62 &           49.96 &           48.91 &  \cellcolor{green!20}{\textbf{45.66}} &  \cellcolor{green!20}{\textbf{37.81}} &  \cellcolor{green!20}{\textbf{32.85}} &  \cellcolor{green!20}{\textbf{23.28}}\\
    aRUB-$L_1$         &      53.40 &           53.16 &           51.99 &  \cellcolor{green!20}{\textbf{51.33}} &           43.36 &           33.83 &           27.19 &           15.16   \\
    aRUB-$L_{\infty}$       &    53.67 &           53.52 &           53.20 &           47.07 &           37.30 &           14.26 &            9.65 &            9.65 \\
    \cdashline{1-9}
    Baseline-$L_\infty$            &           53.32 &           53.24 &           52.66 &           46.09 &           36.99 &           13.71 &            9.61 &            9.61 \\
    PGD-$L_{\infty}$             &  \cellcolor{green!20}{\textbf{54.88}} &  \cellcolor{green!20}{\textbf{54.80}} &  \cellcolor{green!20}{\textbf{53.98}} &           49.26 &           37.42 &           27.30 &           14.02 &           12.46  \\
    Nominal        &       46.88 &           46.88 &           46.76 &           44.69 &           37.19 &           14.88 &            9.02 &            8.95\\
    \bottomrule
    \bottomrule
    \end{tabular}
    \caption{CIFAR: Lower bound of adversarial accuracy with uncertainty bounded in $L_1$ norm by $\rho$.}\label{Table:CIFAR_BOUND}
    \end{table}

We developed two new methods for \textcolor{black}{adversarial} training of neural networks, both of which provide \textcolor{black}{ an upper bound of the adversarial loss by considering the whole network at once instead of \textcolor{black}{applying convex relaxations} and propagating bounds for each layer separately as in previous works.} First, we found an \textcolor{black}{empirical} upper bound by incorporating the first order approximation of the network's output layer. This method does not provide security guarantees against adversarial attacks but it performs very well across a variety of data sets when the uncertainty set is small \textcolor{black}{and it stands out for its simplicity}. Second, by extending state-of-the-art tools from \textcolor{black}{RO} to non-convex and non-concave functions, we were able to construct a \textcolor{black}{provable} upper bound \textcolor{black}{of the adversarial loss. Experimental results show that this method has a performance edge for larger uncertainty sets, and importantly, this method can certify the non-existance of adversarial attacks bounded in $L_1$ norm. The two proposed upper bounds are in closed-form and can be effectively minimized with backpropagation.} 
Lastly, we provide evidence that adding robustness can improve the natural accuracy of neural networks for classification problems with tabular or vision data.

\textcolor{black}{For future work we are interested in extending the RUB approach for other types of norms as well as  understanding how the tightness of the proposed upper bounds change across layers in order to facilitate further improvements. Adversarial robustness is crucial in the development of more secure machine learning systems, and we hope that our work will inspire further research in this important area. }

\paragraph{Code Availability Statement.}
All the code to reproduce the results can be found here: \url{https://github.com/kimvc7/Robustness}. An illustrative example on how to use the code to train a network is provided here: \url{https://colab.research.google.com/github/kimvc7/Robustness/blob/main/demo.ipynb}

\paragraph{Acknowledgements.}
\textcolor{black}{We would like to thank the editor and the reviewers of the paper for their comments, which helped us to improve the paper significantly.}

\textcolor{black}{Xavier Boix has been supported by the Center for Brains, Minds and Machines (funded by NSF STC award CCF-
1231216), the R01EY020517 grant from the National Eye Institute (NIH). Kimberly Villalobos Carballo and Xavier Boix have been supported by Fujitsu Laboratories
Ltd. (Contract No. 40008819).}

\bibliography{robust_paper}

\newpage

\appendix


\section{Proofs of Lemmas}\label{sec:appendixA}

In this Appendix, we prove \textcolor{black}{L}emmas \ref{lemma:gen_upper_bound}, \ref{convex-bound}, \ref{concave-bound} and \ref{conjugate-computation}.
\textcolor{black}{\subsection{Proof of Lemma \ref{lemma:gen_upper_bound}}\label{A1}}
\textcolor{black}{\noindent\textbf{Proof:} By Assumption \ref{assumption:trans-inv} with $c = \bz_y^L(\theta, \bx + \boldsymbol{\delta})$ we know
 \begin{align*}
    &\min_\theta \:\max_{\boldsymbol{\delta}\in \mathcal{U}} \: \mathcal{L}(y, \bz^L(\theta,\bx + \boldsymbol{\delta})) = \min_\theta \:\max_{\boldsymbol{\delta}\in \mathcal{U}} \: \mathcal{L}\left(y, \bz^L(\theta,\bx + \boldsymbol{\delta}) - \bz_y^L(\theta, \bx + \boldsymbol{\delta})\boldsymbol{e}\right).
     \end{align*}
\textcolor{black}{Define $\bar{\bz}(\boldsymbol{\delta})\coloneqq \bz^L(\theta,\bx + \boldsymbol{\delta}) - \bz_y^L(\theta, \bx + \boldsymbol{\delta})\boldsymbol{e}$ and $\bar{\boldsymbol{z}}^\prime = (\max_{\boldsymbol{\delta}\in \mathcal{U}}\bar{\bz}_1(\boldsymbol{\delta})$, 
$\hdots, \max_{\boldsymbol{\delta}\in \mathcal{U}}\bar{\bz}_K(\boldsymbol{\delta}))$. Notice that the $y^{th}$ coordinates of $\bar{\bz}(\boldsymbol{\delta})$ and $\bar{\bz}^\prime$ are both zero, and therefore for all $k\in [K]$ we have
\[ \bar{\bz}_k({\boldsymbol{\delta}}) - \bar{\bz}_y({\boldsymbol{\delta}}) =\bar{\bz}_k({\boldsymbol{\delta}})  \leq  \max_{\boldsymbol{\delta}\in \mathcal{U}} \bar{\bz}_k(\boldsymbol{\delta}) = = \bar{\bz}^\prime_k = \bar{\bz}^\prime_k - \bar{\bz}^\prime_y.\]
Therefore, we can apply Assumption \ref{assumption:monotone} with $\boldsymbol{z} = \bar{\bz}(\boldsymbol{\delta})$ and $\boldsymbol{z}^\prime = \bar{\bz}^\prime$ to obtain}
    \begin{align*}
     &\min_\theta \:\max_{\boldsymbol{\delta}\in \mathcal{U}} \: \mathcal{L}\left(y, \bz^L(\theta,\bx + \boldsymbol{\delta}) - \bz_y^L(\theta, \bx + \boldsymbol{\delta})\boldsymbol{e}\right)\\
    \leq &\min_\theta \: \mathcal{L}\left(y , \bigg(\max_{\boldsymbol{\delta}\in \mathcal{U}} \bz_1^L(\theta, \bx + \boldsymbol{\delta})- \bz_y^L(\theta, \bx + \boldsymbol{\delta}),\dots, \max_{\boldsymbol{\delta}\in \mathcal{U}} \bz_K^L(\theta, \bx + \boldsymbol{\delta}) - \bz_y^L(\theta, \bx + \boldsymbol{\delta})\bigg) \right).
\end{align*}
We then conclude
 \begin{align*}
    &\min_\theta \:\max_{\boldsymbol{\delta}\in \mathcal{U}} \: \mathcal{L}(y, \bz^L(\theta,\bx + \boldsymbol{\delta})) \\
    \leq &\min_\theta \: \mathcal{L}\left(y , \bigg(\max_{\boldsymbol{\delta}\in \mathcal{U}} \bz_1^L(\theta, \bx + \boldsymbol{\delta})- \bz_y^L(\theta, \bx + \boldsymbol{\delta}),\dots, \max_{\boldsymbol{\delta}\in \mathcal{U}} \bz_K^L(\theta, \bx + \boldsymbol{\delta}) - \bz_y^L(\theta, \bx + \boldsymbol{\delta})\bigg) \right).
     \end{align*}}

\subsection{ Proof of Lemma \ref{convex-bound}}\label{A2}
\begin{proof}
Since $f$ is convex and closed, we have $f = (f^*)^*$ \citep{rockafellar-1970a}, and applying the definition of the convex conjugate function we obtain
\[
f(\bz(\boldsymbol{\delta})) = (f^\star)^\star(\bz(\boldsymbol{\delta})) = \sup_{\bu\in \text{dom}(f^*)} \bz(\boldsymbol{\delta})^\top \bu - f^\star(\bu),
\]
which implies 
\begin{align*}
\sup_{\boldsymbol{\delta} \in \mathcal{U}}\: f(\boldsymbol{z}(\boldsymbol{\delta})) + g(\boldsymbol{z}(\boldsymbol{\delta})) &=\sup_{\boldsymbol{\delta} \in \mathcal{U}} \enspace \sup_{\bu\in \text{dom}(f^\star)}\:  \boldsymbol{z}(\boldsymbol{\delta})^T\bu - f^\star(\bu) + g(\boldsymbol{z}(\boldsymbol{\delta})) \\ 
&= \sup_{\bu\in \text{dom}(f^\star)} \enspace \sup_{\boldsymbol{\delta} \in \mathcal{U}}\:  \boldsymbol{z}(\boldsymbol{\delta})^T\bu - f^\star(\bu) + g(\boldsymbol{z}(\boldsymbol{\delta})),
\end{align*}
as desired.\hfill 
\end{proof}

\subsection{ Proof of Lemma \ref{concave-bound}}\label{A3}
Let $\mathcal{Z} = \{\bz(\boldsymbol{\delta}) : \boldsymbol{\delta}\in \mathcal{U}\}$. Defining the indicator function \[\gamma(\bz | \mathcal{Z}) = \begin{cases} &0 \quad \text{if  }\bz \in \mathcal{Z},\\&\infty \quad \text{otherwise,} \end{cases}\] and applying the Fenchel duality theorem \citep{rockafellar-1970a}, we obtain:
\begin{align}
    &\sup_{\boldsymbol{\delta} \in \mathcal{U}} g(\bz(\boldsymbol{\delta})) =\sup_{\bz \in \mathcal{Z}} g(\bz) = \sup_{\bz\in \text{dom}(g)\cap \text{dom}(\gamma) } g(\bz) - \gamma(\bz | \mathcal{Z})  = \inf_{\bv\in \text{dom}(g_\star)} \gamma^\star(\bv| \mathcal{Z}) - g_\star(\bv).
\end{align}
Finally, since $\gamma^\star(\bv |\mathcal{Z}) = \sup_{\bz\in \mathcal{Z}} \bz^\top\bv$, we conclude
\[
\sup_{\boldsymbol{\delta} \in \mathcal{U}} g(\bz(\boldsymbol{\delta})) =\inf_{\bv\in \text{dom}(g_\star)} \sup_{\bz \in \mathcal{Z}} \bz^\top\bv - g_\star(\bv) = \inf_{\bv\in \text{dom}(g_\star)} \sup_{\boldsymbol{\delta} \in \mathcal{U}} \bz(\boldsymbol{\delta})^\top\bv - g_\star(\bv).
\]

\subsection{Proof of Lemma \ref{conjugate-computation}}\label{A4}
We first prove part a). By definition, we have
\[
f^\star(\bz) = \sup_{\bx} \bz^\top \bx  - \bp^\top[\bx]^+.
\]
Notice that if the $i^{th}$ component of $\bz$ is negative for any $i$, then $f^\star(\bz) = \infty$ because $\bx$ can be the vector with an arbitrarily large negative value in the $i^{th}$ coordinate and $0$ everywhere else. Similarly, if the $i^{th}$ component of $\bz$ is larger than the $i^{th}$ coordinate of $\bp$ for any $i$, then again $f^\star(\bz) = \infty$ because $\bx$ can be the vector with an arbitrarily large positive value in the $i^{th}$ coordinate and $0$ everywhere else. Moreover, if $\boldsymbol{0}\leq \bz\leq \bp$, then  
\[\sup_{\bx} \bz^\top \bx  - \bp^\top[\bx]^+ \leq \sup_{\bx} \bz^\top \bx  - \bz^\top[\bx]^+ = \sup_{\bx} \bz^\top (\bx  - [\bx]^+)\leq 0.
\]
Since $\bx = \boldsymbol{0}$ achieves an objective value of $0$, we conclude that $\boldsymbol{0}\leq \bz\leq \bp$ implies $f^\star(\bz) = 0$ as desired.\\

\noindent Next, we proceed to prove part b). By definition of the concave conjugate we have
\[
g_\star(\bz) = \inf_{\bx} \bz^\top \bx - (\bx^\top \bu - \bq^\top[\bx]^+) = \inf_{\bx} (\bz-\bu)^\top \bx + \bq^\top[\bx]^+. 
\]
If the $i^{th}$ component of $\bz$ is larger than the $i^{th}$ component of $\bu$ for any $i$, then $g_\star(\bz) = \infty$ because $\bx$ can be the vector with an arbitrarily large negative value in the $i^{th}$ coordinate and $0$ everywhere else. Similarly, if the $i^{th}$ component of $\bz$ is smaller than the $i^{th}$ coordinate of $(\bu-\bq)$ for any $i$, then again $g_\star(\bz) = \infty$ because $\bx$ can be the vector with an arbitrarily large positive value in the $i^{th}$ coordinate and $0$ everywhere else. In addition, if $\boldsymbol{\bu-\bq}\leq \bz\leq \bu$, then  
\[\inf_{\bx} (\bz-\bu)^\top \bx + \bq^\top[\bx]^+ \geq \inf_{\bx} (\bz-\bu)^\top \bx  + (\bu - \bz)^\top[\bx]^+ = \inf_{\bx} (\bu - \bz)^\top ( [\bx]^+ - \bx)\geq 0.
\]
Since $\bx = \boldsymbol{0}$ achieves an objective value of $0$, we conclude that $\boldsymbol{\bu-\bq}\leq \bz\leq \bu$ implies $g_\star(\bz) = 0$ as desired.\\

\newpage
\section{Generalized Results}\label{sec:appendixB}
We now state and proof the generalization of Theorem \ref{general-decomposition1}, Corollary \ref{cor:p-norm1} and Theorem \ref{last} for the case in which the neural network has more than 2 layers. 

 \begin{theorem}[Generalization of Theorem \ref{general-decomposition1}]
 \hfill\\

For all $2\leq l \leq L$, it holds
 \begin{align}
     \sup_{{\boldsymbol{\delta}} \in \mathcal{U}} \:& \bc_k^\top\bz^\ell(\theta, \bx + {\boldsymbol{\delta}}) =  \\
     \begin{split}
     & \sup_{{\bs}_L} \inf_{{\bt}_L}\dots \sup_{{\bs}_l} \inf_{{\bt}_l}\sup_{\boldsymbol{\delta} \in \mathcal{U}}  \: (\bp_l - \bq_l)^\top\bz^{l-1}(\theta, \bx + {\boldsymbol{\delta}} )  + \sum_{\ell = l}^{L-1}(\bp_{\ell+1} - \bq_{\ell+1})^\top\bb^\ell +  \bc_k^\top\bb^L\\
     & \hspace{3.5cm}\text{s.t.} \quad \bp_L = [(\bbw^L)^\top\bc_k]^+\odot \bs_L\\
     &  \hspace{4.4cm} \bq_L = [-(\bbw^L)^\top\bc_k]^+\odot \bt_L\\
     &  \hspace{4.4cm} \bp_{\ell}= (([\bbw^\ell]^+)^\top \bp_{\ell+1} + ([-\bbw^\ell]^+)^\top \bq_{\ell+1})\odot \bs_\ell \quad \forall \: \ell = l,\dots, L-1\\
     &  \hspace{4.4cm} \bq_{\ell} = (([-\bbw^\ell]^+)^\top\bp_{\ell+1} + ( [\bbw^\ell]^+)^\top \bq_{\ell+1})\odot \bt_\ell \quad \forall \: \ell = l,\dots, L-1\\
    & \hspace{4.4cm} 0\leq \bs_\ell, \bt_\ell \leq 1 \quad \forall \: \ell = l,\dots, L.
     \label{induction-problem}
     \end{split}
 \end{align} \label{general-decomposition}
 \end{theorem} 
\begin{proof}
 We will proceed by backward induction on the layer number $l$.\\
 \noindent \textbf{Case $l = L$:}\hfill \\
The proof is equivalent to the case $L=2$ already proved in Section \ref{sec:RUB}.\\
\hfill \\
\noindent \textbf{Case $l-1$:}\hfill \\
Suppose the theorem holds for some fixed $l$ with $l> 2$. We have 
\begin{align*}
     &(\bp_l - \bq_l)^\top \bz^{l - 1}(\theta, \bx + {\boldsymbol{\delta}}) \\
     =& (\bp_l - \bq_l)^\top (\bbw^{l - 1}[\bz^{l - 2}(\theta , \bx + {\boldsymbol{\delta}})]^+ + \bb^{l - 1}) \\
     =& f_+(\bz^{l-2}(\theta, \bx + {\boldsymbol{\delta}}))  - f_{-}(\bz^{l-2}(\theta, \bx + {\boldsymbol{\delta}})) + (\bp_l - \bq_l)^\top\bb^{l - 1},
 \end{align*}
where \begin{align*}f_+(\bx) &= (\bp_{l}^\top [\bbw^{l-1}]^+ + \bq_{l}^\top[-\bbw^{l-1}]^+)[\bx]^+, \quad   \text{and} \\
f_{-}(\bx) &= (\bp_{l}^\top [-\bbw^{l-1}]^+ + \bq_{l}^\top[\bbw^{l-1}]^+)[\bx]^+.
\end{align*} 
By Lemma \ref{convex-bound}  we then obtain 
\begin{align*}
    &\sup_{{\boldsymbol{\delta}} \in \mathcal{U}} \: (\bp_l - \bq_l)^\top\bz^{l - 1}(\theta, \bx + {\boldsymbol{\delta}})   \\
   =& \sup_{{\boldsymbol{\delta}} \in \mathcal{U}} \: f_+(\bz^{l-2}(\theta, \bx + {\boldsymbol{\delta}}))  - f_{-}(\bz^{l-2}(\theta, \bx + {\boldsymbol{\delta}}))+ (\bp_l - \bq_l)^\top\bb^{l - 1}\\
     =& \sup_{\bu_{l-1}\in \text{dom}(f_+^\star)} \sup_{{\boldsymbol{\delta}} \in \mathcal{U}}  \bu_{l-1}^\top \bz^{l-2}(\theta, \bx + {\boldsymbol{\delta}})- f_{-}(\bz^{l-2}(\theta, \bx + {\boldsymbol{\delta}}))+ (\bp_l - \bq_l)^\top\bb^{l - 1}\label{convex-part-l}.
\end{align*}
Defining the concave function $g(\bx) = \bu_{l-1}^\top \bx - f_{-}(\bx)$, and applying Lemma \ref{concave-bound} we obtain
 \begin{align}
     &\sup_{{\boldsymbol{\delta}} \in \mathcal{U}} \: (\bp_l - \bq_l)^\top\bz^{l - 1}(\theta, \bx + {\boldsymbol{\delta}})   \\
     =& \sup_{\bu_{l-1}\in \text{dom}(f_+^\star)} \inf_{\bv_{l-1} \in \text{dom}(g_\star)}\sup_{{\boldsymbol{\delta}} \in \mathcal{U}}   \bv_{l-1}^\top\bz^{l-2}(\theta, \bx + {\boldsymbol{\delta}}) + (\bp_l - \bq_l)^\top\bb^{l - 1}\label{concave-part}.
 \end{align}
Lastly, by Lemma \ref{conjugate-computation} we can substitute 
\begin{align*}
    \bu_{l-1} &= ( [(\bbw^{l-1}]^+)^\top \bp_{l} + ([-\bbw^{l-1}]^+)^\top \bq_{l})\odot \bs_{l-1}\\
    &= \bp_{l-1}, \quad \text{and}\\
    \bv_{l-1} &= (( [\bbw^{l-1}]^+)^\top \bp_{l} + ([-\bbw^{l-1}]^+)^\top \bq_{l})\odot \bs_{l-1} - (\bp_{l} [-\bbw^{l-1}]^+ + \bq_{l}[\bbw^{l-1}]^+)\odot \bt_{l-1}\\
    & = \bp_{l-1} - \bq_{l-1},
\end{align*}
which together with the induction hypothesis imply that Eq. \eqref{induction-problem} is equivalent to

\begin{align*}
    \sup_{{\boldsymbol{\delta}} \in \mathcal{U}} \:& \bc_k^\top\bz^{\ell-1}(\theta, \bx + {\boldsymbol{\delta}})  =  \\
     & \sup_{{\bs}_L} \inf_{{\bt}_L}\dots \sup_{{\bs}_{l-1}} \inf_{{\bt}_{l-1}}\sup_{\boldsymbol{\delta} \in \mathcal{U}}  \: (\bp_{l-1} - \bq_{l-1})^\top\bz^{l-2}(\theta, \bx + {\boldsymbol{\delta}} )  + \sum_{\ell = l-1}^{L-1}(\bp_{\ell+1} - \bq_{\ell+1})^\top \bb^\ell + \bc_k^\top\bb^L\\
     & \hspace{3.5cm}\text{s.t.} \quad \bp_L = [(\bbw^L)^\top\bc_k]^+\odot \bs_L\\
     &  \hspace{4.4cm} \bq_L = [-(\bbw^L)^\top\bc_k]^+\odot \bt_L\\
     &  \hspace{4.4cm} \bp_{\ell}= (([\bbw^\ell]^+)^\top \bp_{\ell+1} + ([-\bbw^\ell]^+)^\top \bq_{\ell+1})\odot \bs_\ell \quad \forall \:l-1 \leq  \ell\leq L-1\\
     &  \hspace{4.4cm} \bq_{\ell} = (([-\bbw^\ell]^+)^\top\bp_{\ell+1} + ( [\bbw^\ell]^+)^\top \bq_{\ell+1})\odot \bt_\ell \quad \forall \: l-1 \leq \ell \leq L-1\\
    & \hspace{4.4cm} 0\leq \bs_\ell, \bt_\ell \leq 1 \quad \forall \: \ell = l-1,\dots, L,
\end{align*}
and therefore the theorem holds for $l-1$ as desired.
\end{proof}
\begin{corollary}[Generalization of Corollary \ref{cor:p-norm1}] \label{cor:p-norm}
\hfill \\
If $\mathcal{U} = \{\boldsymbol{\delta} : \|\boldsymbol{\delta}\|_p \leq \rho\}$, then:
 \begin{align}
     &\sup_{{\boldsymbol{\delta}} \in \mathcal{U}} \: \bc_k^\top\bz^L(\theta, \bx + {\boldsymbol{\delta}})   \\
     \begin{split}
     =& \sup_{{\bs}_L} \inf_{{\bt}_L}\dots \sup_{{\bs}_2} \inf_{{\bt}_2}  \rho \|(\bp_2 - \bq_2)^\top\bbw^1\|_{q} + (\bp_2 - \bq_2)^\top\bbw^1\bx  + \sum_{\ell = 1}^{L-1}(\bp_{\ell+1} - \bq_{\ell+1})^\top\bb^\ell +  \bc_k^\top\bb^L \\
     & \hspace{3.5cm}\text{s.t.} \quad \bp_L = [(\bbw^L)^\top\bc_k]^+\odot \bs_L\\
     &  \hspace{4.4cm} \bq_L = [-(\bbw^L)^\top\bc_k]^+\odot \bt_L\\
     &  \hspace{4.4cm} \bp_{\ell}= (([\bbw^\ell]^+)^\top \bp_{\ell+1} + ([-\bbw^\ell]^+)^\top \bq_{\ell+1})\odot \bs_\ell \quad \forall \: \ell = 2,\dots, L-1\\
     &  \hspace{4.4cm} \bq_{\ell} = (([-\bbw^\ell]^+)^\top\bp_{\ell+1} + ( [\bbw^\ell]^+)^\top \bq_{\ell+1})\odot \bt_\ell \quad \forall \: \ell = 2,\dots, L-1\\
    & \hspace{4.4cm} 0\leq \bs_\ell, \bt_\ell \leq 1 \quad \forall \: \ell = 2,\dots, L,
     \label{conjugate_problem}
     \end{split}
 \end{align}
 where $\|\cdot \|_q$ is the conjugate norm of $\|\cdot\|_p$.
 \end{corollary}

 \begin{proof}
  The proof follows directly after applying Theorem \ref{general-decomposition} with $l = 2$ and using again Eq. \eqref{robust_trick}. \hfill \end{proof}
\begin{definition}
We introduce the following definitions to simplify notation:
\begin{align*}
\bs &\coloneqq (\bs_2, \hdots, \bs_L)\\
\bt & \coloneqq (\bt_2, \hdots, \bt_L)\\
\bp_L(\bs, \bt) &\coloneqq [(\bbw^2)^\top \bc_k]^+\odot \bs_L\\
\bq_L(\bs, \bt)& \coloneqq [-(\bbw^2)^\top \bc_k]^+\odot \bt_L\\
\hfill \\
\bp_\ell(\bs,\bt) &\coloneqq \left(([\bbw^\ell]^+)^\top \bp_{\ell + 1}(\bs, \bt) + ([-\bbw^\ell]^+)^\top \bq_{\ell + 1}(\bs, \bt)\right)\odot \bs_\ell\quad \forall \: 1\leq\ell<L\\
\bq_\ell(\bs,\bt) &\coloneqq \left(([-\bbw^\ell]^+)^\top \bp_{\ell + 1}(\bs, \bt) + ([\bbw^\ell]^+)^\top \bq_{\ell + 1}(\bs, \bt)\right)\odot \bt_\ell \quad \forall \: 1\leq\ell<L\\
\hfill \\
R_\ell(\bs, \bt) &\coloneqq \sum_{\ell^{\prime} = \ell}^{L-1} \left(\bp_{\ell^{\prime} + 1}(\bs, \bt) - \bq_{\ell^\prime + 1}(\bs, \bt)\right)^\top\bb^{\ell^{\prime}}.
\end{align*}
\end{definition}
 
 \begin{theorem}[Generalization of Theorem \ref{last}]\label{last_general}
 \hfill \\
\begin{align}
\sup_{{\boldsymbol{\delta}}: \|\boldsymbol{\delta}\|_1 \leq \rho} &\: \bc_k^\top\bz^L(\theta, \bx + {\boldsymbol{\delta}})   \nonumber \\
       \leq &\inf_{0\leq {\bt}\leq 1} \: \max_{m\in [M]} \: \max\bigg\{\  g^L_{k,m}(\theta, \bx, \bt, \rho), g^L_{k,m}(\theta, \bx, \bt, -\rho)\bigg\} ,
\end{align}
where the new network $g$ is defined by the equations
\begin{align*}
    g^1_m(\bbw, \bx, \bt, a, r) &= r(a\bbw^1_m  + \bbw^1\bx + \bb^1)\\
    g^\ell_m(\theta, \bx, \bt, a, r) &= [r\bbw^\ell]^+[g^{\ell - 1}_m(\bbw, \bx, \bt, a, 1)]^+ \hspace{-0.15cm}+ [-r\bbw^{\ell}]^+[g^{\ell - 1}_m(\bbw, \bx, \bt,a, -1)]\odot \bt_\ell + r\bb^{\ell}\\
    g^L_{k,m}(\theta, \bx, \bt, a) &= [\bc_k^\top \bbw^L]^+[g^{L - 1}_m(\theta, \bx, \bt, a, 1)]^+ + [-\bc_k^\top \bbw^L]^+[g^{L-1}_m(\theta, \bx, \bt, a, -1)]\odot \bt_L + \bc_k^\top\bb^L,
\end{align*}
for all $1<\ell<L$, $1\leq k\leq K$, $a\in \{\rho, -\rho\}$, and $r \in \{-1, 1\}$.
\end{theorem}
The proof of this theorem relies on the following lemma.

\begin{lemma}\label{lemma:solving-supremum}
For all $2\leq \ell\leq L-1$ it holds
\begin{align}
&\sup_{0\leq {\bs}_\ell\leq 1} \:   \bp_\ell(\bs, \bt)^\top g^{\ell - 1}_m(\theta, \bx, \bt, a, 1) + \bq_\ell(\bs, \bt)^\top g^{\ell - 1}_m(\theta, \bx, \bt, a, -1)\\
    =& \bp_{\ell + 1}(\bs, \bt)^\top g^\ell_m(\theta, \bx, \bt, a, 1) + \bq_{\ell + 1}(\bs, \bt)^\top g^\ell_m(\theta, \bx, \bt, a, -1) - (\bp_{\ell+1}(\bs, \bt) - \bq_{\ell + 1}(\bs, \bt))^\top \bb^{\ell} \label{eq:identity}.
\end{align}
\end{lemma}
\begin{proof}
Let $2\leq \ell \leq L-1$, we have
\begin{align*}
    &\sup_{0\leq {\bs}_\ell\leq 1} \:   \bp_\ell(\bs, \bt)^\top g^{\ell - 1}_m(\theta, \bx, \bt, a, 1) + \bq_\ell(\bs, \bt)^\top g^{\ell - 1}_m(\theta, \bx, \bt, a, -1)\\
    \hfill \\
    =& \sup_{0\leq {\bs}_\ell\leq 1} \: 
    \left( (([\bbw^\ell]^+)^\top \bp_{\ell + 1}(\bs, \bt) + ( [-\bbw^\ell]^+)^\top \bq_{\ell + 1}(\bs, \bt))\odot \bs_\ell \right)^\top g^{\ell - 1}_m(\theta, \bx, \bt, a, 1) +\\
    & \quad \quad \quad \left( (([-\bbw^\ell]^+)^\top \bp_{\ell + 1}(\bs, \bt) + ([\bbw^\ell]^+)^\top \bq_{\ell + 1}(\bs, \bt))\odot \bt_\ell\right)^\top g^{\ell - 1}_m(\theta, \bx, \bt, a, -1)\\
    &\text{\hfill }\\
    =& \sup_{0\leq {\bs}_\ell\leq 1} \: 
     \left(([\bbw^\ell]^+)^\top \bp_{\ell + 1}(\bs, \bt) + ([-\bbw^\ell]^+)^\top \bq_{\ell + 1}(\bs, \bt)\right)^\top \left(g^{\ell - 1}_m(\theta, \bx, \bt, a, 1)\odot \bs_\ell\right) +\\
    & \quad \quad \quad \left( ([-\bbw^\ell]^+)^\top \bp_{\ell + 1}(\bs, \bt) + ([\bbw^\ell]^+)^\top \bq_{\ell + 1}(\bs, \bt)\right)^\top \left(g^{\ell - 1}_m(\theta, \bx, \bt, a, -1)\odot \bt_\ell\right)\\
    &\text{\hfill }\\
    =&\left(([\bbw^\ell]^+)^\top \bp_{\ell + 1}(\bs, \bt) + ([-\bbw^\ell]^+)^\top \bq_{\ell + 1}(\bs, \bt)\right)^\top [g^{\ell - 1}_m(\theta, \bx, \bt, a, 1)]^+ +\\
    &  \left( ([-\bbw^\ell]^+)^\top \bp_{\ell + 1}(\bs, \bt) + ([\bbw^\ell]^+)^\top \bq_{\ell + 1}(\bs, \bt)\right)^\top \left(g^{\ell - 1}_m(\theta, \bx, \bt, a, -1)\odot \bt_\ell\right)\\
    \hfill \\
    =& \bp_{\ell + 1}(\bs, \bt)^\top g^\ell_m(\theta, \bx, \bt, a, 1) + \bq_{\ell + 1}(\bs, \bt)^\top g^\ell_m(\theta, \bx, \bt, a, -1) - (\bp_{\ell+1}(\bs, \bt) - \bq_{\ell + 1}(\bs, \bt))^\top \bb^{\ell},
\end{align*}
as desired.
\end{proof}

\begin{proof}[Proof of theorem \ref{last_general}]
\hfill \\
By Corollary \ref{cor:p-norm} with $p =1$ we know
\begin{align}
     &\sup_{{\boldsymbol{\delta}}: \|{\boldsymbol{\delta}}\|_1 \leq \rho} \: \bc_k^\top(\bz^L(\theta, \bx + {\boldsymbol{\delta}}) - \bb^L)  \\
     \leq & \sup_{{\bs}_L} \inf_{\bt_L}, \dots \sup_{\bs_2}\inf_{ {\bt}_2} \:  \rho \|(\bp_2(\bs, \bt) - \bq_2(\bs, \bt))^\top \bbw^1\|_{\infty} + (\bp_2(\bs, \bt) - \bq_2(\bs, \bt))^\top\bbw^1\bx  + R_1(\bs, \bt)\\
     \begin{split}
     \leq & \inf_{{\bt}_L\dots {\bt}_2} \sup_{{\bs}_L\dots {\bs}_2}  \rho \|(\bp_2(\bs, \bt) - \bq_2(\bs, \bt))^\top\bbw^1\|_{\infty} + (\bp_2(\bs, \bt) - \bq_2(\bs, \bt))^\top \bbw^1\bx  + R_1(\bs, \bt),
     \end{split}
 \end{align}
where the last inequality follows from the min-max inequality. Observe that  $\bp_{\ell^\prime}(\bs, \bt)$ and $\bq_{\ell^\prime}(\bs, \bt)$ are independent on $\bs_{\ell}$ for all ${\ell}^\prime > \ell$, which in turn implies that $R_{\ell^\prime}(\bs, \bt)$ does not depend on $\bs_{\ell}$ for all ${\ell}^\prime > \ell$. We can then solve the optimization problem in Eq. \eqref{conjugate_problem} for fixed ${\bt}$ as follows:\\
\begin{align*}
    &\sup_{0\leq \bs_L,\dots, {\bs}_2\leq 1} \: \rho \|(\bp_2(\bs, \bt) - \bq_2(\bs, \bt))^\top \bbw^1\|_{\infty} + (\bp_2(\bs, \bt) - \bq_2(\bs, \bt))^\top \bbw^1\bx + R_1(\bs, \bt) \\
    \hfill \\
    =& \max_{m\in [M]}\: \max\bigg\{\sup_{0\leq \bs_L,\dots, {\bs}_2\leq 1} \: ( \bp_2(\bs, \bt) -\bq_2(\bs, \bt))^\top (\bbw^1(\bx + \rho \be_m) + \bb^1) + R_2(\bs, \bt),\\
    & \hspace{2.1cm}\sup_{0\leq \bs_L,\dots, {\bs}_2\leq 1} \: (\bp_2(\bs, \bt) - \bq_2(\bs, \bt))^\top (\bbw^1(\bx - \rho \be_m) + \bb^1) + R_2(\bs, \bt)\bigg\}\\
    &\hfill \\
    =&  \max_{m\in [M]}\: \max\bigg\{\sup_{0\leq \bs_L,\dots, {\bs}_2\leq 1}\: \bp_2(\bs, \bt)^\top g^1_m(\theta, \bx, \bt, \rho, 1) + \bq_2(\bs, \bt)^\top g^1_m(\theta, \bx, \bt, \rho, -1)+ R_2(\bs, \bt),\\
    & \hspace{2.1cm}\sup_{0\leq \bs_L,\dots, {\bs}_2\leq 1} \: \bp_2(\bs, \bt)^\top g^1_m(\theta, \bx, \bt, -\rho, 1)  + \bq_2(\bs, \bt)^\top g^1_m(\theta, \bx, \bt, -\rho, -1)+ R_2(\bs, \bt)\bigg\}.
\end{align*}
By repeatedly applying Lemma \ref{lemma:solving-supremum} for each $\ell = 2, \dots, L-1$ we obtain 

\begin{align*}
    &\sup_{0\leq \bs_L,\dots, {\bs}_2\leq 1} \: \rho \|(\bp_2(\bs, \bt) - \bq_2(\bs, \bt))^\top\bbw^1\|_{\infty} + (\bp_2(\bs, \bt) - \bq_2(\bs, \bt))^\top\bbw^1\bx + R_1(\bs, \bt) \\
    \hfill \\
    =&  \max_{m\in [M]}\: \max\bigg\{\sup_{0\leq \bs_L\leq 1}\: \bp_L(\bs, \bt)^\top g^{L-1}_m(\theta, \bx, \bt, \rho, 1) + \bq_L(\bs, \bt)^\top g^{L-1}_m(\theta, \bx, \bt, \rho, -1),\\
    & \hspace{2.1cm}\sup_{0\leq \bs_L\leq 1} \: \bp_L(\bs, \bt)^\top g^{L-1}_m(\theta, \bx, \bt, -\rho, 1)  + \bq_L(\bs, \bt)^\top  g^{L-1}_m(\theta, \bx, \bt, -\rho, -1)\bigg\}\\
    &\hfill \\
    =&  \max_{m\in [M]}\: \max\bigg\{\sup_{0\leq \bs_L\leq 1} [\bc_k^\top \bbw^L]^+ \! \! \left(g^{L-1}_m(\theta, \bx, \bt, \rho, 1)\odot \bs_L\right) \!+\! [-\bc_k^\top \bbw^L]^+ \! (g^{L-1}_m(\theta, \bx, \bt, \rho, -1)\odot \bt_L),\\
    & \hspace{2.1cm}\sup_{0\leq \bs_L\leq 1}  [\bc_k^\top\bbw^L]^+ \! \! \left(g^{L-1}_m(\theta, \bx, \bt, -\rho, 1)\odot \bs_L\right)  \!+\! [-\bc_k^\top\bbw^L]^+ (g^{L-1}_m(\theta, \bx, \bt, -\rho, -1)\odot \bt_L)\! \bigg\}\\
    \hfill \\
    =&  \max_{m\in [M]}\: \max\bigg\{ g^L_{k,m}(\theta, \bx, \bt, \rho),g^L_{k,m}(\theta, \bx, \bt, -\rho)\bigg\} - \bc_k^\top\bb^L,
\end{align*}
as desired.
\end{proof}

\newpage

\newpage
\clearpage
\section{Convolutional Neural Networks}\label{sec:appendixD}
While in this paper we only consider feed forward neural networks, it is possible to extend \textcolor{black}{the RUB} method to convolutional neural networks that use ReLU and MaxPool activation functions. In fact, Lemma \ref{conjugate-computation} can be modified as follows:

\begin{lemma}\label{conjugate-computation-cnn}
Let $x\in \mathbb{R}^{A\times B}$. Define $MP(x)\in \mathbb{R}^{C\times D}$ as the MaxPool function whose $(c, d)$ coordinate corresponds to $\max_{i\in I_{cd}}\{x_i\}$ for fixed sets of indices $I_{cd}$, and denote by $\circledast$ the convolution operation. If $\bu\in \mathbb{R}^{A\times B}$, $\bp, \bq\in \mathbb{R}^{C\times D}$ have all nonnegative coordinates, then the functions $f(\bx) = \bp\circledast MP[\bx]^+$ and $g(\bx) = \bx\circledast\bu - \bq\circledast MP[\bx]^+$ satisfy
\begin{align*}
    &a) \: f^\star(\bz) = \begin{cases} 0& \quad \text{if    } 0\leq \sum_{i\in I_{cd}}z_{i} \leq p_{cd}, \forall \: c \in [C], d\in [D], \\ \infty& \quad \text{otherwise,} \end{cases} \quad \text{and}\\
    & b) \:g_\star(\bz) = \begin{cases} 0& \quad \text{if    } u_{cd}-q_{cd}\leq \sum_{i\in I_{cd}}z_i \leq u_{cd}, \forall \: c \in [C], d\in [D],\\ -\infty& \quad \text{otherwise.} \end{cases}
\end{align*}
 \end{lemma}
The lemma above allows to obtain an upper bound for the \textcolor{black}{adversarial loss} of convolutional networks with a very similar proof to that of Theorem \ref{last_general}. However, convolutional networks are notoriously more memory consuming and therefore computation of the robust upper bound requires more resources. We have then left this computation for future work, but we here report results for the other robust training methods using these more complex neural networks.

We evaluate a convolutional neural network (denoted as \emph{CNN}) that has been commonly used in previous works of adversarial robustness~\cite{madry2019deep}. It has two convolutional layers alternated with pooling operations, and two dense layers. We compare adversarial accuracy across four different methods: aRUB-$L_\infty$, aRUB-$L_1$, PGD-$L_\infty$ and Nominal training. Results for the CIFAR data set are shown in Table \ref{Table:CIFAR}. We observe that the proposed methods aRUB-$L_1$ and aRUB-$L_\infty$ yield the highest adversarial accuracies with respect to PGD-$L_2$ attacks. For the MNIST data set and the FASHION MNIST data set (Table \ref{Table:FASION} and Table \ref{Table:MNIST}, respectively), we see that aRUB-$L_\infty$ has the highest adversarial accuracies with respect to PGD-$L_2$ attacks when $\rho\leq 0.1$, whereas PGD-$L_\infty$ does best for larger values of $\rho$.

\begin{table}[H]
\centering
\begin{tabular}{llllllllllllll}
\toprule
\toprule
{} & \multicolumn{9}{l}{ } \\
\hfill $\boldsymbol{\rho}=$ &          0.000    &          0.010  &                0.020  &          0.030  &          0.100  &          1.000  &          3.000  &          5.000  &          10.000 \\
\midrule
aRUB-$L_1$    &  71.17        &                           70.98 &                   70.70 &           70.51 &  \textbf{69.88} &  \textbf{60.31} &  \textbf{44.69} &  \textbf{35.35} &  \textbf{17.89} \\
aRUB-$L_\infty$            &  \textbf{72.03} &    \textbf{71.76} &    \textbf{71.45} &  \textbf{71.25} &           69.84 &           56.13 &           43.98 &           34.26 &           16.25 \\
\cdashline{1-10}
PGD-$L_\infty$             &           70.78 &                                     70.55 &                   70.43 &           70.39 &           69.65 &           59.96 &           43.59 &           29.10 &           16.21 \\
Nominal         &           71.29 &                                   71.13 &                   71.05 &           70.66 &           69.38 &           48.16 &           16.05 &           10.04 &           10.04 \\
\bottomrule
\bottomrule
\end{tabular}
\caption{Adversarial Accuracy for CIFAR with CNN architecture and PGD-$L_2$ attacks.}\label{Table:CIFAR}
\end{table}

\begin{table}[H]
\centering
\begin{tabular}{llllllllllllll}
\toprule
\toprule
{} & \multicolumn{9}{l}{} \\
\hfill $\boldsymbol{\rho}=$ &          0.000  &        0.010  &               0.020  &          0.030  &          0.100  &          1.000  &          3.000  &          5.000  &          10.000 \\
\midrule
aRUB-$L_1$             &           91.25  &             91.09 &                 90.78 &           90.55 &           89.02 &           82.46 &           68.95 &           54.80 &           33.20 \\
aRUB-$L_\infty$             &  \textbf{91.37} &    \textbf{91.13} &   \textbf{91.05} &  \textbf{90.86} &  \textbf{90.59} &           82.81 &           71.45 &           60.23 &           30.86 \\
\cdashline{1-10}
PGD-$L_\infty$            &           90.59 &                    90.55 &                  90.43 &           90.23 &           89.30 &  \textbf{84.45} &  \textbf{73.20} &  \textbf{67.54} &  \textbf{57.77} \\
Nominal       &           91.02 &                 90.90 &                90.62 &           90.43 &           89.02 &           77.85 &           56.72 &           40.51 &           10.16 \\
\bottomrule
\bottomrule
\end{tabular}\caption{Adversarial Accuracy for Fashion MNIST with  CNN architecture and PGD-$L_2$ attacks.}\label{Table:FASION}
\end{table}

\begin{table}[H]
\centering
\begin{tabular}{llllllllllllll}
\toprule
\toprule
{} & \multicolumn{9}{l}{} \\
\hfill $\boldsymbol{\rho}=$ &          0.000  &          0.010  &                   0.020  &          0.030  &          0.100  &          1.000  &          3.000  &          5.000  &          10.000 \\   
\midrule
aRUB-$L_1$               &  \textbf{99.38} &    \textbf{99.38} &    \textbf{99.38} &           99.34 &  \textbf{99.30} &           98.32 &           91.68 &           70.51 &           33.83 \\
aRUB-$L_\infty$             &  \textbf{99.38} &    \textbf{99.38} &    \textbf{99.38} &  \textbf{99.38} &  \textbf{99.30} &  \textbf{98.79} &           95.70 &           87.46 &           47.93 \\
\cdashline{1-10}
PGD-$L_\infty$             &           99.22 &                   99.22 &                    99.14 &           99.14 &           99.02 &           98.55 &  \textbf{96.37} &  \textbf{91.29} &  \textbf{55.39} \\
Nominal       &           99.30 &                    99.26 &                   99.26 &           99.26 &           99.18 &           97.93 &           89.69 &           60.00 &            9.34 \\
\bottomrule
\bottomrule
\end{tabular}
\caption{Adversarial Accuracy for MNIST with CNN architecture and PGD-$L_2$ attacks.}\label{Table:MNIST}
\end{table}

\end{document}